\documentclass[11pt,oneside]{article}
\usepackage{amsmath}
\usepackage{mathtools}
\usepackage[
  left=1in,
  right=1in,
  top=1in,
  bottom=1in
]{geometry}

\usepackage{amsthm}
\usepackage{thmtools}
\usepackage{amssymb}
\usepackage{hyperref}
\usepackage[capitalize,nameinlink]{cleveref}
\usepackage[numbers]{natbib}
\usepackage{url}
\usepackage{tocloft}     
\usepackage{algorithm}
\usepackage[english]{babel}
\usepackage{xcolor}
\usepackage{enumitem}
\setlist{nosep}
\usepackage{algpseudocode}
\Crefname{algorithm}{Algorithm}{Algorithms}   % for \Cref
\Crefname{algorithm}{Algorithm}{Algorithms}   % for \Cref (capitalised)
\DeclareMathOperator{\Tr}{Tr}

\DeclareMathOperator*{\argmin}{arg\,min}
\newtheorem{theorem}{Theorem}[section]

\newtheorem{definition}[theorem]{Definition}

\newtheorem{lemma}[theorem]{Lemma}

\newtheorem{claim}[theorem]{Claim}

\newtheorem{model}[theorem]{Model}

\newtheorem{proposition}[theorem]{Proposition}

\theoremstyle{definition}
\newtheorem{remark}[theorem]{Remark}

\newtheorem{question}[theorem]{Question}

\newcommand\Mycomb[2][^n]{\prescript{#1\mkern-0.5mu}{}C_{#2}}
\title{Tractable Gaussian Phase Retrieval with Heavy Tails and Adversarial Corruption with Near-Linear Sample Complexity}

\author{
Santanu Das \\
Tata Institute of Fundamental Research, Mumbai \\
\texttt{dassantanu315@gmail.com}
\and
Jatin Batra \\
Tata Institute of Fundamental Research, Mumbai \\
\texttt{jatinbatra50@gmail.com}
}

\date{}
\begin{document}

% If your paper is accepted and the title of your paper is very long,
% the style will print as headings an error message. Use the following
% command to supply a shorter title of your paper so that it can be
% used as headings.
%
%\runningtitle{I use this title instead because the last one was very long}

% If your paper is accepted and the number of authors is large, the
% style will print as headings an error message. Use the following
% command to supply a shorter version of the author names so that
% they can be used as headings (for example, use only the surnames)
%
%\runningauthor{Surname 1, Surname 2, Surname 3, ...., Surname n}
\maketitle

%\twocolumn[

%\aistatstitle{Tractable Gaussian Phase Retrieval with Heavy Tails and Adversarial Corruption with Near-Linear Sample Complexity}
%\aistatsauthor{ Author 1 \And Author 2 \And  Author 3 }
%\aistatsaddress{ Institution 1 \And  Institution 2 \And Institution 3 } ]

\begin{abstract}
  Phase retrieval is the classical problem of recovering a signal $x^* \in \mathbb{R}^n$ from its noisy phaseless measurements $y_i = \langle a_i, x^* \rangle^2 + \zeta_i$ (where $\zeta_i$ denotes noise, and $a_i$ is the sensing vector) for $i \in [m]$. The problem of phase retrieval has a rich history, with a variety of applications such as optics, crystallography, heteroscedastic regression, astrophysics, etc. A major consideration in algorithms for phase retrieval is \emph{robustness} against measurement errors. In recent breakthroughs in algorithmic robust statistics, efficient algorithms have been developed for several parameter estimation tasks such as mean estimation, covariance estimation, robust principal component analysis (PCA), etc. in the presence of heavy-tailed noise and adversarial corruptions. In this paper, we study efficient algorithms for robust phase retrieval with heavy-tailed noise when a constant fraction of both the measurements $y_i$ and the sensing vectors $a_i$ may be arbitrarily adversarially corrupted. For this problem, Buna and Rebeschini (AISTATS 2025) very recently gave an \emph{exponential} time algorithm with sample complexity $O(n \log n)$. Their algorithm needs a \emph{robust spectral initialization}, specifically, a robust estimate of the top eigenvector of a covariance matrix, which they deemed to be beyond known efficient algorithmic techniques (similar spectral initializations are a key ingredient of a large family of phase retrieval algorithms). In this work, we make a connection between robust spectral initialization and recent algorithmic advances in robust PCA, yielding the first polynomial-time algorithms for robust phase retrieval with both heavy-tailed noise and adversarial corruptions, in fact with near-linear (in $n$) sample complexity.
\end{abstract}

\section{Introduction}
\label{sec:intro} 
\textbf{Phase retrieval.} Phase retrieval is the problem of recovering a signal from \emph{phase-less} measurements. In its simplest setting, the goal is to recover $x^* \in \mathbb{R}^n$ (up to a global sign i.e. output $x$ minimizing $dist(x,x^*) := \min(\|x-x^*\|,\|x+x^*\|$) from (squared) \emph{magnitudes} of \emph{linear} measurements
\begin{equation}
\label{model:noisy}
   y_{i}=|\langle a_{i},x^{*}\rangle|^2 + \zeta_i,\quad\forall i=1,2\ldots,m
\end{equation}
where $y_i$ and $a_i$ are known, and $\zeta_i$ denotes noise. %Phase retrieval has enjoyed a rich history of research \cite{gerchberg1972practical,fienup1978reconstruction,fienup1982phase,hirsch1971method,gallagher1973method}.
Broadly, phase retrieval arises in inverse problems where \emph{phase} (or sign) cannot be measured, which appears in many areas like optics \cite{walther1963question}, X-ray crystallography \cite{harrison1993phase,millane1990phase}, microscopy \cite{miao2008extending}, astronomy \cite{fienup1987phase}, diffraction and array imaging \cite{bunk2007diffractive}, acoustics \cite{balan2010signal}, blind channel estimation in wireless communications \cite{ranieri2013phase}, interferometry \cite{demanet2017convex}, quantum mechanics \cite{corbett2006pauli} and quantum information \cite{heinosaari2013quantum}, heteroscedastic regression \cite{DNNB23} etc.

\textbf{History of phase retrieval.} While algorithms for phase retrieval in practice had been devised since the alternating minimization algorithm of Gerchberg-Saxton \cite{gerchberg1972practical,fienup1978reconstruction,fienup1982phase,hirsch1971method,gallagher1973method}, the problem of phase retrieval resisted provable algorithms until recently when a variety of provable techniques emerged such as spectral initialization followed by alternating minimization \cite{NJS15,Wa18} or Wirtinger flow \cite{chen_candes_2015_solving,GFAZ24,CLS15,zhang2017nonconvex}; and semi-definite programming \cite{candesphaselifttight2014} (for more details,  see \Cref{sec:rel_work} and also the book \cite{barnett2022graveyard}).%particular, Wirtinger flow has been studied in \cite{chen_candes_2015_solving, zhang2017nonconvex, GFAZ24, CLS15}, while SDP-based methods have been explored in \cite{candesphaselifttight2014}. 

\textbf{Measurement error model.} 
In this work, we study the measurement error model where the noise $\zeta_i$ is \textbf{heavy-tailed} and homoscedastic (for the formal definition, see \Cref{def:error_model}) and further an $\epsilon$-fraction of the measurements may be \textbf{adversarially corrupted} (in particular, they can depend arbitrarily on \emph{all} of sensing vectors $a_i$ and responses $y_i$ for $i \in [m]$) where $\epsilon > 0$ is known; for the formal definition, see \Cref{def:strong_adv_corruption} and \Cref{def:error_model}.
Recently, the model of strong adversarial corruption has received a lot of attention in algorithmic robust statistics, see e.g. recent breakthroughs described in the book \cite{diakonikolas_Ilias_Kane_2019_recent_advances_in_algorithmic_high_dimensional_robust_statistics}. The following is the notion of strong adversarial corruption in robust statistics \cite{diakonikolas2019robust,lugosi2023online, DiaKP20outlier}:
\begin{model}
\label{def:strong_adv_corruption}
  \textbf{(Strong adversarial corruption)}. 
  A clean dataset $(a_i, y_i)^m_{i=1}$ is first generated according to \eqref{model:noisy}  and then the dataset is revealed to an adversary, who can inspect all $m$ samples and arbitrarily modify an $\epsilon$-fraction of them. That is, the adversary is allowed to choose any subset of at most $\epsilon m$ samples and replace both the covariates $a_i$ and responses $y_i$ in any manner. The learner only has access to the corrupted dataset provided by the adversary.
 \end{model}
This error model for phase retrieval was very recently studied by Buna and Rebeschini (for heavy-tailed $\zeta$) \cite{bunaR24_robust_phase} and Dong et al. \cite{Dong2024outlierrobust} (for $\zeta = 0$). Formally, for the setting of phase retrieval, in this paper, we study the following  measurement models:

\begin{model}
\label{def:error_model}
\textbf{(Gaussian phase retrieval with zero mean heavy tailed noise and strong adversarial corruption:)} Measurements $y_i$ are generated from an unknown $x^* \in \mathbb{R}^n$ according to \eqref{model:noisy} with sensing vectors $a_i \sim \mathcal{N}(0,I_n)$\footnote{The Gaussian model has been the prototypical setting for algorithms for phase retrieval with provable guarantees \cite{candesphaselifttight2014, CLS15, GFAZ24, bunaR24_robust_phase, NJS15, DNNB23, sun2018geometric, wang2017solvingsystemofquadraticequationsviaatruncatedmplitdeflow}.}, where $\zeta_i$ is zero-mean conditioned on $a_i$ i.e. $\mathbb{E}[\zeta_i|a_i] = 0$ and homoscedastic i.e. its conditional fourth moment $\mathbb{E}[\zeta^4|a_i]$ and its conditional variance $\mathbb{E}[\zeta^2|a_i]$ are constants (resp. $K^4_4$ and $\sigma^2$) independent of $a_i$ \footnote{This assumption on the second and the fourth moment are standard in robust statistics, see e.g. \cite{bunaR24_robust_phase, oliveira2024improved, pensia2024robust}.}. An $\epsilon$-fraction of $\{(a_{i},y_{i})\}^m_{i=1}$ (where both $a_{i}$ and $y_{i}$ can be corrupted) are then corrupted by a strong adversary as in \Cref{def:strong_adv_corruption}.\footnote{As stated, the $m$ measurements are given at once to the algorithm. However, for simplicity, we will allow our algorithms to collect a constant number of times a fresh batch of $m$ samples as in \Cref{def:error_model}, where the constant can depend on the error guarantees. This is also the assumption in prior robust iterative algorithms, e.g. \cite{bunaR24_robust_phase, prasad2020robust, merad_2023_robust_supervised_learning_with_coordiante_gradient_descent}.}
\end{model}
%\vspace{-0.5em} 
\begin{model}
\label{def:error_model_nonzero_mean}
\textbf{(Gaussian phase retrieval with non-zero mean heavy-tailed noise and strong adversarial corruption).} 
We consider the same setting as in \Cref{def:error_model}, with the modification that the noise has a nonzero conditional mean, i.e., 
$
\mathbb{E}[\zeta_i \mid a_i] = \mu \neq 0.
$ Also note that $\mu$ is unknown.
\end{model}

\textbf{Gaussian linear measurements for phase retrieval.} Gaussian measurements form the canonical and most extensively studied setting in the phase retrieval literature. Many influential methods - such as Truncated Wirtinger Flow—are analyzed in this model, and several applications (including certain heteroscedastic regression \cite{DNNB23} setups) naturally lead to Gaussian designs. Historically, theoretical advances in phase retrieval first emerged under Gaussian measurements and only later were extended to more structured models (e.g., coded diffraction patterns (CDP) \cite{chen_candes_2015_solving}), by building on insights developed in the Gaussian case. In this article, we focus exclusively on the Gaussian measurement model and leave extensions to other sensing designs, such as coded diffraction patterns, for future work.

\textbf{Motivation for heavy-tailed noise and adversarial corruption:}
Real-world data are rarely clean and typically contain multiple sources of noise. Broadly, these noises can be classified into two categories: uncorrelated noise across samples and correlated noise across the dataset. In the worst case, uncorrelated noise can be modeled as stochastic heavy-tailed noise, while correlated noise is often modeled through the lens of a strong adversary (\cref{def:strong_adv_corruption}) that can inspect the entire dataset and arbitrarily corrupt a constant fraction of samples. Such partial corruption models are well motivated by practical applications. For instance, wireless sensor networks \cite{wagner2004resilient, othman2013secure, tian2023distributed, chen2015amplitude} and IoT systems consist of many spatially distributed, low-power devices that continuously sense, collect, and transmit environmental data. These multi-sensor platforms are widely deployed in settings such as smart grids, industrial IoT, and seismic monitoring arrays, where their measurements are aggregated to form global statistics for monitoring, inference, and decision-making. In such large-scale deployments, typically only a small subset of sensors is compromised \cite{wagner2004resilient, wang2006survey}, due to factors such as limited physical access, partial network penetration, or firmware vulnerabilities that affect only certain devices. The majority of sensors remain honest. Consequently, the data arriving at the central aggregator naturally fit a model in which a fraction of the measurements are adversarial, while the rest are clean,  resulting in a constant fraction adversarial contamination model (\cref{def:strong_adv_corruption}). This setting directly motivates the framework of robust statistics, which has been extensively studied; see, for example, an entire book \cite{DiaK23_book} is devoted to this contamination model, and numerous fundamental problems have been studied under it, including robust mean estimation \cite{hopkins_2020_robust_and_heavy_tailed_mean_estimation}, covariance estimation \cite{kothari2018robust,nearly_linear_time_covarience_estiamtion_for_gaussain,duchi2025fast}, PCA \cite{kong2020robust,robust_PCA_Algorithm_in_nearly_linear_time_by_Diakonikolas}, linear regression \cite{cherapanamjeri2020optimal}, and non-convex optimization \cite{bunaR24_robust_phase,li2023robust}. Importantly, phase retrieval itself is a highly non-convex inverse problem in which measurements are frequently acquired across multiple sensing units \cite{yang2025robust}. In practical deployments, a fraction of these sensors may be faulty or compromised \cite{wagner2004resilient,wang2006survey}, leading to a corresponding fraction of corrupted measurements. While these applications domain provide concrete instances of partial adversarial contamination, our goal is to emphasize that the proposed model (\cref{def:error_model}) captures both gross, correlated corruptions and uncorrelated stochastic noise in a unified framework. Our corruption model is designed to abstract out more general forms of grossly correlated corruption present in real-world data.

\textbf{Prior work on phase retrieval with adversarial corruptions and heavy-tailed noise:} There is a long line of research on phase retrieval with adversarial corruption or heavy-tailed noise, mostly devoted to the simpler case when only the responses $y_i$ are corrupted (and all the sensing vectors $a_i$ are known exactly), see e.g. \cite{hand2016corruption, duchiR19_solving,zhang_median-truncated_2018,qian2017inexact,barik2024sample}. Our focus, in contrast, is the case where both $a_i$ and $y_i$ may be corrupted, which, to the best of our knowledge, is studied only very recently by \cite{bunaR24_robust_phase} and \cite{Dong2024outlierrobust}. \cite{bunaR24_robust_phase} introduced the  \Cref{def:error_model} and \Cref{def:error_model_nonzero_mean} for phase retrieval and showed exponential time algorithms that given $O(n\log n)$ samples, for a sufficiently small (but independent of $n$) $\epsilon$, output $x$ satisfying $dist(x,x^*) = O(\sigma\sqrt{\epsilon}/\|x^*\|)$\footnote{
Note that the $1/\|x^{\ast}\|$ dependence of the error is unavoidable and for a proof see \Cref{sec: correct scaling}.} with high probability (for context, note that $O(\sigma \sqrt{\epsilon})$ is the information theoretically minimum error achievable in the setting of robust \emph{mean estimation}\cite{diakonikolas_Ilias_Kane_2019_recent_advances_in_algorithmic_high_dimensional_robust_statistics}). \cite{bunaR24_robust_phase} also describes a polynomial time algorithm for the low contamination regime where $\epsilon = O(1/n)$; however, we do not discuss that here as the low contamination regime is not the focus of our paper. 

\cite{Dong2024outlierrobust} claimed a near-linear time algorithm with near-linear (in $n$) sample complexity for the noiseless ($\zeta=0$) case with strong adversarial corruptions (\Cref{def:strong_adv_corruption}) \footnote{While this claimed result appears to have a bug which we discuss in more detail in \Cref{sec: dong}, it still leads to an interesting link between efficient robust PCA and robust phase retrieval; see also the discussion at the end of the second para in \Cref{sec:tech}.}. We discuss the techniques in both of these works in more detail in \Cref{sec:tech}. While a complete survey of the rich algorithmic history of phase retrieval and robust statistics is beyond the scope of this work, we give a bird's eye view of other related work in \Cref{sec:rel_work}.

\textbf{Our results:} In this paper, we present the first polynomial-time algorithms (with running time $\tilde{O}(m^2 n)$) for the error models \Cref{def:error_model} and \Cref{def:error_model_nonzero_mean} that output $x$ satisfying the error guarantee $dist(x,x^*) = O(\sigma \sqrt{\epsilon}/\|x^*\|)$ with high probability, with sample complexity $\tilde{O}(n)$ (where $\tilde{O}$ hides log factors), for any constant $\epsilon$ (\emph{independent of} $n$) fraction of corruptions smaller than a constant depending only on noise to signal ratio $K_4/\|x^*\|^2$, improving the results of \cite{bunaR24_robust_phase}. 
%for the \Cref{def:error_model}. We also present the first polynomial-time algorithm for \Cref{def:error_model_nonzero_mean}, which achieves similar guarantees as in \Cref{def:error_model}.
 We also pose a question that elucidates a precise connection between robust phase retrieval in \Cref{def:error_model} and near-linear time algorithms for robust PCA, which further gives a different general perspective on the work of \cite{Dong2024outlierrobust} (who focused on the $\zeta = 0$ case).

We now describe our technical contribution in more detail and put it in context of prior work.
\section{Techniques}
\label{sec:tech}

\textbf{Spectral initialization and gradient descent:} We first briefly explain the main idea of a family of phase retrieval algorithms \cite{chen_candes_2015_solving,CLS15,bunaR24_robust_phase,wuR23_noisy,GFAZ24,ma2018implicit,NJS15}. The population risk $r(x) = \mathbb{E}[(y - \langle a,x \rangle^2)^2]/4$ for phase retrieval is non-convex, however, it is well-known \cite{bunaR24_robust_phase,ma2018implicit, chen_candes_2015_solving, CLS15, wuR23_noisy,GFAZ24, NJS15} that within a ball of radius $R = \Theta(\|x^*\|)$ around $\pm x^*$, $r(x)$ is smooth and strongly convex. Hence, if one could obtain an initial iterate $x_0$ inside this ball, \emph{gradient descent} would converge efficiently to $\pm x^*$. It turns out that the top eigenvector of $\mathbb{E}[y a a^T]$ is parallel to $x^*$, and hence to obtain an initial iterate $x_0$, it suffices to obtain an approximate top eigenvector of $\mathbb{E}[y a a^T]$, which is typically obtained by finding the top eigenvector of an empirical estimator, such as $\frac{1}{m}\sum_i y_i a_i a^T_i$ in the simplest setup (\emph{spectral initialization}). For several variants of spectral initialization, see \cite{NJS15,bunaR24_robust_phase, zhang2017nonconvex,wang2017solvingsystemofquadraticequationsviaatruncatedmplitdeflow,GFAZ24,chen_candes_2015_solving}.

\textbf{Prior techniques for phase retrieval with strong adversarial corruption and heavy-tailed noise (\Cref{def:error_model}):} Naturally, the main algorithmic principle of \cite{bunaR24_robust_phase} and \cite{Dong2024outlierrobust} is to use a \emph{robust} spectral initialization followed by a \emph{robust} gradient descent procedure. We first describe their robust gradient descent procedure: they express the gradient as a mean of a random variable whose samples can be obtained from the \Cref{def:error_model}, and leverage recent algorithmic breakthroughs for robust mean estimation; this idea has appeared in several recent works \cite{prasad2020robust,li2023robust,holland_2019_efficient_learning_with_robust_gredient_descent}. For robust spectral initialization, \cite{bunaR24_robust_phase} have the insight that the top eigenvector of $\operatorname{Cov}(y a)$ also suffices for spectral initialization and use the exponential time robust covariance estimator  \cite{oliveira2024improved} to estimate $\operatorname{Cov}(y a)$, and find the top eigenvector of the estimate. On the other hand \cite{Dong2024outlierrobust} (who focuses on the $\sigma=0$ case) showed that given sufficient samples, robustly finding the top eigenvector of $\mathbb{E}[y a a^T]$ can be reduced to a convex Ky-fan-2 norm optimization problem, which can be converted into a packing semidefinite program for which nearly-linear time solvers exist. Their reduction for near-linear sample complexity has a bug, which is later claimed to be resolved in \cite{dong2025outlier}. For details, see page 6, and for a full discussion, see \cref{sec: dong}.\footnote{However, their reduction works if the sample complexity is allowed to be more than $\tilde{O}(n)$ (but still a polynomial in $n$), we elaborate on this in \Cref{sec: dong}.}

\subsection{Our technical contribution:} 
\textbf{Robust phase retrieval for zero mean noise.} While we still use the same robust gradient descent procedure as prior works \cite{bunaR24_robust_phase,Dong2024outlierrobust}, we show that the result of \cite{kong2020robust} for robust PCA can be leveraged to obtain, in polynomial time, a robust spectral initialization in the \Cref{def:error_model} that suffices for initializing the robust gradient descent procedure. \cite{kong2020robust} shows that for \emph{bounded} random variables $X$ satisfying a certain fourth-moment condition, the top eigenvector of the covariance matrix $\operatorname{Cov}(X)$ can be approximated in polynomial time with sample complexity $\tilde{O}(n)$. We show that a well-chosen \emph{truncation} on $ya$ leads to an appropriately bounded random variable. While truncation combined with filtering is classical, it is not a priori clear that a valid truncation exists for robust phase retrieval: over-truncation distorts the distribution, and under-truncation weakens filtering. It turns out that with the bounded fourth-moment assumption on the noise, our choice of truncation gives a sufficiently good approximation to the top eigenvector of $\operatorname{Cov}(y a)$ in polynomial time with sample complexity $\tilde{O}(n)$ by employing the result of \cite{kong2020robust}. Hence, together with a robust gradient descent procedure (\cite{bunaR24_robust_phase} or \cite{Dong2024outlierrobust}), this leads to a polynomial time algorithm for the error \Cref{def:error_model}. Specifically, we show the following theorem (informally stated, whose formal version appears in \Cref{theorem: formal version of informal theorem using kong robust}):

\begin{theorem}
\label{theorem: informal using robust pca from kong paper}
 \textbf{(Informal.)} Consider the \Cref{def:error_model}. Assume an upper bound $r_{up}$ on $K_{4}/\|x^*\|^2$ is known. There exists a polynomial time algorithm  (\Cref{algorithm: Spectral Initialisation with Robust PCA for additive zero mean noise 1} (this work) +Algorithm 2 from \cite{bunaR24_robust_phase}) such that if the corruption fraction $\epsilon \leq C_1/ (r_{up}^2+1)^2$ and the number of samples  $m \geq C_2  n (r_{up}^2+1)\log(2n/(\delta \epsilon))/\epsilon$ (where $C_1$ and $C_2$ are absolute constants), then with probability at least $1-\delta$, the output $x_{out}$ of the algorithm satisfies 
\begin{equation}\notag
dist(x_{out},x^*)\leq O\left( \frac{\sigma}{\|x^*\|} \sqrt{\epsilon}\right).
\end{equation}
\end{theorem}
\textbf{ Comparison of our work with \cite{bunaR24_robust_phase}.} Since \cite{bunaR24_robust_phase} studies the same problem as ours, we compare the results here. Convergence guarantees for robust iterative optimization algorithms typically rely on strong convexity and smoothness \cite{prasad2020robust}, which hold only locally in non-convex problems. Hence, robust non-convex methods often require a good initialization to ensure that the iterates enter a well-behaved region (methods with random initializations often don't have as good sample complexity guarantees in the robust setting \cite{li2023robust}). Phase retrieval (PR) is a canonical example of such a problem. \cite{bunaR24_robust_phase} (AISTATS 2025) showed that the traditional spectral-initialization method for PR to land in a strongly convex and smooth region around the minimum, can be interpreted as computing the top eigenvector of $\operatorname{Cov}(y a)$, enabling the use of robust PCA methods for the initialization. Starting from such an initialization, \cite{bunaR24_robust_phase} analyzed the robust gradient-descent procedure for phase retrieval and showed that the hyperparameters (such as the learning rate, sample complexity, and allowable corruption level) can be chosen so that the iterates of the non-convex PR objective remain within the strongly convex and smooth region established by the initializer. Robust gradient descent is well understood given strong convexity and smoothness at the iterate, and methods using second order information for non-convex landscapes are also known e.g. the definition of robust gradient in \cite{bunaR24_robust_phase} appears for the first time in \cite{prasad2020robust}, see also \cite{cherapanamjeri2020optimal} utilizing this idea, see \cite{li2023robust} for robust non-convex optimization using second-order information at the cost of higher sample complexity. The key and poorly understood component in sample-efficient robust non-convex procedures is therefore robust initialization, \cite{bunaR24_robust_phase} gave the first such (spectral) initialization for a non-convex problem, which, however, needed exponential time. Our contribution is to show that one can find a choice of truncation on the covariates that, together with the robust PCA algorithm (filtering) of \cite{kong2020robust} gives a polynomial-time spectral initializer with near-linear sample complexity that tolerates a constant fraction of corruptions, thereby improving the result of \cite{bunaR24_robust_phase} from exponential time to polynomial time. The technical difficulty of finding such truncation + filtering combinations is explained below.

\textbf{Comparison of our technique with prior use of truncation+filtering.} While truncation and filtering are classical tools in robust statistics, their successful interplay is highly problem-dependent. Several examples illustrate this point: for instance, \cite{diakonikolas2022outlier1}'s robust sparse mean estimation proceeds by projecting (i.e. truncating) means to an $\ell_{\infty}$ ball around the median-of-means estimator, followed by the filtering procedure of \cite{balakrishnan2017computationally}. Catoni and Olivier \cite{catoni2012challenging} use a customized truncation-based estimator for heavy-tailed mean estimation. The work \cite{sasai2022robust} addresses the robust sparse regression problem by first using coordinate-wise truncation together with a filtering step to remove adversarially corrupted design vectors, and then they perform regression by minimizing the Huber loss, which is essentially a truncated $\ell_2$ loss, to handle corruption in the responses. On the other hand, the proof attempt in the (incorrect) submission \cite{Dong2024outlierrobust} uses truncation together with concentration inequalities, which is conceptually similar to the truncation-plus-filtering framework, where the validity of their argument appears to rely critically on the last inequality on page 14 of their manuscript, which claims that $\operatorname{Pr}_{x \sim \mathcal{N}(0,1)}(x^4 \geq \ln z/c) \leq z^{-3}$ holds for all $z > 1$ for some constant c; however, this is false, and more generally the integral $\int^\infty_{\tilde{\tau}} \operatorname{Pr}_{x \sim \mathcal{N}(0,1)}(x^4 \geq \ln z/c)dz$ diverges for all $\tilde{\tau}$ and constants $c>0$; from this viewpoint, the issues in their argument can be interpreted as a failed attempt to apply a truncation-based method, whereas in a concurrent work \cite{dong2025outlier}, the authors of \cite{Dong2024outlierrobust} claim a result that resolves  the incorrect part of their earlier results for phase retrieval in the zero-noise setting using a sophisticated truncation+filtering techniques.
 These examples highlight that the correct use of truncation and filtering is subtle and depends critically on the specifics of the problem: over-truncation can destroy essential structure, while under-truncation can render filtering ineffective. We truncate the vectors $ya$ using a norm-based cutoff determined by the input parameters, chosen large enough to keep the top eigenvector of the covariance of the truncated $ay$ close to the top eigenvector of the true covariance, yet small enough to apply the robust PCA method of \cite{kong2020robust}. Furthermore, unlike earlier truncation-based phase retrieval methods—which typically truncate only on the responses—our method truncates both the responses and the sensing vectors. We also show that our method remains valid even when the noise has a nonzero mean, which requires robustly solving a special case of blind deconvolution. Our work provides a principled and effective initialization scheme using truncation + filtering that may naturally extend to other robust non-convex estimation settings.

\textbf{Towards near-linear time algorithms for robust phase retrieval:} In fact, we can show that the appropriately truncated empirical distribution of $ya$ over $m$ samples satisfies a certain stability condition (\Cref{definition: Stability for nearly linear time robust PCA algorithm}, \cite{robust_PCA_Algorithm_in_nearly_linear_time_by_Diakonikolas}) that can be used to employ the nearly-linear time PCA algorithm of \cite{robust_PCA_Algorithm_in_nearly_linear_time_by_Diakonikolas} for robust spectral initialization, at the cost of losing the $\tilde{O}(n)$ sample complexity; however, the sample complexity $m$ still remains polynomial in $n$ (even for heavy-tailed noise). This also provides a different perspective on \cite{Dong2024outlierrobust}, whose reduction to convex Ky-fan-2 norm optimization relies on the same stability condition (see \Cref{definition: Stability for nearly linear time robust PCA algorithm}), for the zero-noise case. We pose the open problem: does the empirical distribution over $\tilde{O}(n)$ appropriately truncated samples of $ya$ satisfy stability (\Cref{definition: Stability for nearly linear time robust PCA algorithm})? This question, if answered affirmatively, would imply a near-linear time algorithm for robust phase retrieval in the \Cref{def:error_model}.

\textbf{Robust phase retrieval for non-zero mean noise:} Next, we study \Cref{def:error_model_nonzero_mean}, which considers the case of non-zero unknown mean noise. \cite{bunaR24_robust_phase} introduced a \emph{symmetrization-based} trick, also used in \cite{pensia2024robust}, that converts phase retrieval with non-zero mean noise into a special case of the robust blind deconvolution problem. Similar to \Cref{def:error_model}, \cite{bunaR24_robust_phase} provided an exponential-time algorithm for \Cref{def:error_model_nonzero_mean} with the same guarantees as those given for \Cref{def:error_model}.
We show that our techniques can also be extended to \Cref{def:error_model_nonzero_mean}. However, estimating $\|x^*\|$ in this setting is not straightforward. Refer \Cref{gap in blind convolution by Buna an dRebashini} for more details, in fact, the argument in \cite{bunaR24_robust_phase} contains a gap in the robust estimation of $\|x^*\|$.
%, even for their exponential-time algorithm tolerating a constant fraction of corruption with nearly linear sample complexity. 
In this work, we provide a fix for this issue by first obtaining a crude robust estimate of $\|x^*\|$, using it to carefully choose the truncation parameter, and then using robust PCA to fine-tune the previous crude estimate.
Specifically, we show the following theorem (informally stated, whose formal version appears in \Cref{theorem: formal version of blind deconvolution theorem}):

\begin{theorem}
\label{theorem: informal for non-zero mean noise using robust pca from kong paper}
 \textbf{(Informal.)}  Under \Cref{def:error_model_nonzero_mean} and the same settings as in \Cref{theorem: informal using robust pca from kong paper}, 
there exists a polynomial-time algorithm (\Cref{algorithm: Spectral Initialisation with Robust PCA for additive non-zero mean noise 1} (this work) + Algorithm 4 from \cite{bunaR24_robust_phase}) 
that achieves the same guarantees as in \Cref{theorem: informal using robust pca from kong paper}, 
except that the requirement on $m$ changes to
$
m \;\geq\; C_2 \, n \, (r_{up}^2+1)^2 \, \frac{\log\!\bigl(2n/(\delta \epsilon)\bigr)}{\epsilon}.
$
\end{theorem}

\textbf{Roadmap of the paper.} In \Cref{sec:prelim}, we give the necessary preliminaries for algorithmic robust statistics. In \Cref{sec:alg,sec: non mean zero noise}, we describe our polynomial-time algorithms with $\tilde{O}(n)$ sample complexity for both phase retrieval models  \ref{def:error_model} and \ref{def:error_model_nonzero_mean} respectively, and prove that they output $x$ achieving the error guarantee $dist(x,x^*)=O(\sigma \sqrt{\epsilon}/\|x^*\|)$. The crux of the argument for both models \ref{def:error_model} and \ref{def:error_model_nonzero_mean}, presented in \Cref{subsec:spectralinit,subsec: non mena zero spectral initialzation} respectively, shows that we can find a choice of truncation in either model with which the robust PCA estimator of \cite{kong2020robust} suffices as a spectral initialization for phase retrieval.  In \Cref{sec: dong}, we describe the connection of our work with \cite{Dong2024outlierrobust} and pose an open question that, if answered affirmatively, gives a nearly-linear time algorithm for robust phase retrieval in the \Cref{def:error_model}.
%with sample complexity $\tilde{O}(n^5)$ with the same error guarantee of $O(\sigma \sqrt{\epsilon}/\|x^*\|)$ in the same \Cref{def:error_model}. 
Finally, in \Cref{sec:conc}, we conclude with some open problems raised by our work.

\section{Preliminaries}
\label{sec:prelim}
\vspace{-0.5em}
\textbf{Notations used in paper.}
 We use $\|x\|$ to denote the $\ell_{2}$ norm of a vector $x$, $\|A\|_{\mathrm{op}}$ the operator norm of a matrix $A$, and $\|A\|_{F}$ its Frobenius norm. We write $\mathcal{B}(x,R)$ for the $\ell_{2}$ ball in $\mathbb{R}^n$ centered at $x$ with radius $R$. As is common in the statistics literature, we slightly abuse notation by using the same symbol for a random variable and its realization, with the meaning clear from context. We denote large unspecified constants by $C_{1},C_{2},\ldots$, where the actual constant denoted by a symbol $C_i$ can change from result to result.

We now state the main results in algorithmic robust statistics that our paper relies on.
%\vspace{- 1em}
\subsection{Robust mean estimation}

For robust mean estimation, we will employ the following results of \cite{DiaKP20outlier,hopkins_2020_robust_and_heavy_tailed_mean_estimation}.

\begin{theorem}
\label{thm:robust_mean}
    [Theorem 1.4, Proposition 1.5, Proposition 1.6 in (\cite{DiaKP20outlier})] Suppose a dataset $S$ of size $m$ is sampled i.i.d. from a distribution $\mathcal{D}$ with mean $\mu \in \mathbb{R}^n$ and covariance matrix $\Sigma \in \mathbb{R}^{n \times n}$. Let T denote an $\epsilon-$corrupted version of $S$, where the corruption is performed by a strong adversary (i.e. up to $\epsilon m$ points in $S$ are arbitrarily modified to obtain $T$). For $\epsilon$ less than a sufficiently small universal constant, there exists a polynomial time algorithm on input $T,\epsilon$ that efficiently computes $\widehat{\mu}$ such that, with probability at least $1-\delta$,
 \begin{equation}
 \label{eq:mean_est}
    \|\widehat{\mu}-\mu\|=\sqrt{\|\Sigma\|_{\mathrm{op}}}O\Biggl(\sqrt{\operatorname{r}_{\mathrm{eff}}(\Sigma) / m}+\sqrt{ \epsilon}+\sqrt{\log (1 / \delta) / m}\Biggr),
 \end{equation}
where $\operatorname{r}_{\mathrm{eff}}(\Sigma)=\operatorname{tr}(\Sigma)/\|\Sigma\|_{\mathrm{op}}$ is the effective rank of matrix $\Sigma$. The polynomial time algorithm is completely oblivious to $\Sigma$ (Theorem A.3 in \cite{DiaKP20outlier}). Let's denote the time complexity of this polynomial-time algorithm by $T_{\mathrm{rob\text{-}mean}}(m,n)$. Further, by \cite{hopkins_2020_robust_and_heavy_tailed_mean_estimation}, if an upper bound $\sigma_{up}$ on $\|\Sigma\|_{op}$ is known, the above guarantee in \eqref{eq:mean_est} can be obtained in time $\tilde{O}(m n)$ (with $\|\Sigma\|_{op}$ replaced with $\sigma_{up}$ in \eqref{eq:mean_est}). 
 \end{theorem}
%Discussing robust estimators with full generality is beyond the scope of this paper; for a detailed discussion, see the book \cite{DiaKP20outlier}.
\subsection{Robust principal component analysis}
For robust principal component analysis (PCA), we will employ the following results of \cite{kong2020robust}.
\begin{theorem}
\label{theorem: Guarantees of Algorithm 2 from kong robust 2020 paper 1}
(Proposition 2.6 from \cite{kong2020robust}.)  Let \( S = \{X_i\}_{i=1}^m \) be drawn i.i.d from a distribution \( \mathcal{D} \) supported on \( \mathbb{R}^n \). Let's define the \emph{second-order raw moment} of \( X \sim \mathcal{D} \) as \(
\Sigma_{r} := \mathbb{E}_{X \sim \mathcal{D}}[XX^\top].
\)  Given \( \delta \in (0, 0.5) \), and a corrupted dataset \( T \) such that an \( \epsilon \in (0, 1/36] \) fraction of the points are corrupted arbitrarily, suppose:

\begin{itemize}
    \item The distribution \( \mathcal{D} \) has bounded support such that \( \|XX^\top-\Sigma_{r}\| \leq B \) for all $X\sim \mathcal{D}$ with probability one.
    \item The distribution \( \mathcal{D} \) has bounded fourth moments, i.e.,
    \[
    \max_{x \in \mathbb{R}^n:\|x\|\leq 1} \mathbb{E}_{X\sim\mathcal{D}}\left[ \langle xx^{\top} ,XX^\top-\Sigma_{r}\rangle^2  \right] \leq v'^{2} .
    \]
\end{itemize}

Then, with probability at least \( 1 - \delta \), there exists a polynomial-time algorithm (Algorithm 2 from \cite{kong2020robust}) that efficiently computes $\hat{u}\in \mathbb{R}^{n\times 1}$ and $\lambda_{\hat{u}}\in \mathbb{R}$ such that it satisfies:
\begin{equation}
\label{eq: knog robust eigenvector}
 \operatorname{Tr}[P_1(\Sigma_{r})] - \operatorname{Tr}[\hat{u}^\top (\Sigma_{r}) \hat{u}] = O\left( \epsilon \cdot \operatorname{Tr}[P_1(\Sigma_{r})] + \sqrt{\epsilon} \cdot v'  \right),   
\end{equation}
\begin{equation}
\label{eq: kong robust eigenvalue}
    |\lambda_{\hat{u}}-\operatorname{Tr}[P_1(\Sigma_{r})]|\leq O\left( \epsilon \cdot \operatorname{Tr}[P_1(\Sigma_{r})] + \sqrt{\epsilon} \cdot v'  \right)
\end{equation}
provided that
\[
m = \Omega\left( (n+(B/v')\sqrt{\epsilon})\log(n/(\delta \epsilon))/\epsilon\right).
\]

Here, \( P_k(\cdot) \) is the best rank-\( k \) approximation of a matrix in $\ell_{2}$. The time complexity of Algorithm 2 from \cite{kong2020robust} is $\widetilde{O}(m^2n).$
  \end{theorem}
\begin{remark}
\label{kong: eigenvalue}
     Note that Algorithm~2 from \cite{kong2020robust} is actually designed to return the robust top eigenvector in the case $k=1$ but we can estimate the corresponding eigenvalue also. Specifically, Algorithm~2 takes as input the corrupted set $T = \{X_{j}X_{j}^{\top}\}_{i=1}^{m}$, together with additional parameters such as the corruption level and failure probability. It then applies a filtering procedure (Algorithm~3, \emph{Double Filtering}, from \cite{kong2020robust}) to remove points in $T$ that are likely to be adversarially corrupted. The average of the remaining points in $T$ is computed, which we denote by $\hat{M}$. Finally, the top eigenvector of $\hat{M}$ is returned as an estimate of the robust PCA solution, and the guarantee is given by \Cref{eq: knog robust eigenvector}. In the last step, one may also compute the corresponding eigenvalue of this top eigenvector, which serves as an estimate of the leading eigenvalue of $\Sigma_r$. While Proposition~2.6 from \cite{kong2020robust} provides guarantees only for the eigenvector estimate (\Cref{eq: knog robust eigenvector}) returned by Algorithm~2, the corresponding eigenvalue estimate can also be analyzed using the results in \cite{kong2020robust}. Its guarantee is stated in \Cref{eq: kong robust eigenvalue}. We provide a rigorous proof for \Cref{eq: kong robust eigenvalue} in \Cref{proof: rigorous proof of eigenvector bound of kong robust pca}.
\end{remark}
 \section{Robust Phase Retrieval with zero mean noise}
 \label{sec:alg}
 In this section, we first present the formal version of \Cref{theorem: informal using robust pca from kong paper}, and then progressively build up the necessary tools and results required to prove it. %We conclude the section with the formal version of the theorem, stated below.
\begin{theorem}
\label{theorem: formal version of informal theorem using kong robust}
 Assume that an upper bound $r_{up}$ on $K_4 / \|x^*\|^2$ is known. There exists a polynomial-time algorithm (namely, \Cref{algorithm: Spectral Initialisation with Robust PCA for additive zero mean noise 1} combined with Algorithm 2 from \cite{bunaR24_robust_phase}) such that, if the total number of samples is \( m = m_1+m_2 + P\tilde{m} \), and the following conditions hold:
$m_1 \geq C_2 \log(2/\delta)(1+r_{up}^2)  ,  m_2 \geq C_3 n(r_{up}^2+1)\frac{\log(2n/(\delta \epsilon))}{\epsilon},$
$
\epsilon \leq \frac{C_1}{(r_{up}^2 + 1)^2},  \quad\tilde{m} \geq C_4  \cdot \max\{n \log n,\ \log(1/\delta)\} \cdot r_{up}^2,
$
then, with probability at least \(1 - (P+1)\delta\), the algorithm outputs \(x_P\) satisfying
\begin{align}
\frac{dist(x_P, x^*)}{\|x^*\|} 
&\leq C_5 \exp\!\left( -C_6  P  (1 - \sqrt{\epsilon}) \right) + C_5 \frac{\sigma}{\|x^*\|^2} \left( \sqrt{\tfrac{n \log n}{\tilde{m}}} 
   + \sqrt{\tfrac{\log(1/\delta)}{\tilde{m}}} + \sqrt{\epsilon} \right) 
   \label{eq:guarantees_stat_error}
\end{align}
Moreover, the time complexity of the algorithm is \( O(T_{\mathrm{rob\text{-}mean}}(m_{1},1)+m_2^2 n + P T_{\mathrm{rob\text{-}mean}}(\tilde{m},n)) \).
\end{theorem}

We will use the following result of \cite{bunaR24_robust_phase} who showed that there exists a polynomial time algorithm (robust gradient descent) with sample complexity $\tilde{O}(n)$ for robust phase retrieval if given a point $x_0$ inside a ball of radius $\|x^*\|/9$ centered at $\pm x^*$.

\begin{theorem}
\label{thm:grad}
\textbf{(Gradient descent for phase retrieval, see Theorem 3.3 from \cite{bunaR24_robust_phase}.)} There exists a polynomial time algorithm (see Algorithm 2 from \cite{bunaR24_robust_phase}) that, given a point $x_0 \in B(\pm x^*,\|x^*\|/9)$ and if $\epsilon \leq \frac{C_1 \|x^*\|^4}{\sigma^2}$ and $P$ draws of $\tilde{m}$ samples where $P\tilde{m} \geq C_2 P \cdot \max\{n \log n, \log(1/\delta)\} \cdot \frac{\sigma^2}{\|x^*\|^4}$, with probability at least $1-P\delta$, outputs $x_P$ satisfying 
\begin{align*}
  &\frac{dist(x_P,x^*)}{\left\|x^*\right\|}\leq C_3 \exp\left( -C_4 P  (1 - \sqrt{\epsilon}) \right)+ C_3 \frac{\sigma}{\|x^*\|^2} \left( \sqrt{\frac{n \log n}{\tilde{m}}} + \sqrt{\frac{\log(1/\delta)}{\tilde{m}}} \right)
+ C_3 \frac{\sigma}{\|x^*\|^2} \sqrt{\epsilon}.
\end{align*}
\end{theorem}

\begin{remark}
\label{remark: use of fresh samples}
  Note that \cite{bunaR24_robust_phase} employed fresh samples at each step of robust gradient descent, which is standard in prior work \cite{liu2020high, liu2019high, merad2023robust}. For instance, see Step~4 of Algorithm~1 in \cite{liu2020high}, Step~4 of Algorithm~1 in \cite{liu2019high}, and the paragraph following Lemma~1 in Section~5 of \cite{merad2023robust}. The reason that this is standard is that a constant number of batches of fresh samples is not morally very different from a single batch of fresh samples. For a formal treatment, see \Cref{sec:fresh samples}. Nevertheless, we leave as open problem the question of whether robust gradient descent can be performed using only a single batch of samples. 
\end{remark}
In the following \Cref{subsec:spectralinit}, we show how to obtain such a point $x_0$ which gives the desired algorithm together with \Cref{thm:grad}.

\subsection{Robust Spectral initialization}
\label{subsec:spectralinit}

Define the random variable $X := y a$ and denote its covariance matrix as $\Sigma := \mathbb{E}[XX^T]$ (note that $\mathbb{E}[X]=0$). We begin with the idea of \cite{bunaR24_robust_phase} which interprets the spectral initialization of a long line of work \cite{NJS15,chen_candes_2015_solving} as follows: the top eigenvector of the covariance matrix $\Sigma = \mathbb{E}[XX^T]$ (note that $\mathbb{E}[X]=0$) is parallel to $x^*$:
\begin{align*}
 &\mathbb{E}[y^2 a a^T] = \mathbb{E}\left(\langle a, x^* \rangle^4 a a^T\right)+\mathbb{E}\left( a a^T \, \mathbb{E}[\zeta^2 \mid a]\right)\\
 &\overset{(a)}{=} (3||x^{*}||^{4}+\sigma^2)\mathbb{I}_{n}+12||x^{*}||^{2}x^{*}x^{*T},   
\end{align*}
with eigenvalue $15 \|x^*\|^4 + \sigma^2$ (where $(a)$ follows from the \Cref{lemma:spectral initialization expectation}). 

This implies that it is sufficient to estimate only the top eigenvector of the matrix $\Sigma$, as we can then scale this direction to obtain an initial estimate of $x^*$. Importantly, there is no need to estimate the full covariance matrix $\Sigma$. 

Our tool for this purpose is the robust PCA algorithm of \Cref{theorem: Guarantees of Algorithm 2 from kong robust 2020 paper 1} which estimates the top eigenvector with sample complexity $\tilde{O}(n)$ for \emph{bounded} random variables satisfying a certain fourth-moment condition. $X = y a$ is unbounded; however, one can construct a bounded random variable by \emph{truncating} $X$ at a norm cutoff $\tau$ to be chosen later, without affecting the covariance matrix by much (due to Chebyshev-like concentration because $X$ also has a bounded fourth moment). Formally, we show:

\begin{lemma}
\label{lem:trunc}
    \textbf{(Truncation.)} Define $\widetilde{X} := X\mathbb{I}_{\|X\|\leq \tau}$, 
    $\widetilde{\Sigma}_r:=\mathbb{E}[\widetilde{X}\widetilde{X}^{\top}]$,
     where $\tau$ is the truncation parameter. Then, the following holds:
    \begin{equation}
    \label{eq:trunc_bound}
        \|\Sigma-\widetilde{\Sigma}_r\|_{\mathrm{op}}=O( n(K_{4}^2+ \|x^*\|^4)^2/\tau^2).
    \end{equation}
    
\end{lemma}
The proof of the \Cref{lem:trunc} can be found in \Cref{proof: truncation lemma}.

\textbf{Estimating the error due to truncation.} It turns out that for estimating the top eigenvector of $\Sigma$, we can afford a truncation of $\tau = O(\sqrt{n}((K^2_4/\|x^*\|^4)+1)(\|x^*\|^2))$ because the top eigenvector of $\widetilde{\Sigma}_r$ only needs to lie in a sufficiently small ball around $x^*$. Specifically, the above choice of $\tau$ gives $\|\Sigma-\widetilde{\Sigma}_r\|_{\mathrm{op}}\leq \frac{\sqrt{2}\|x^*\|^4}{9}.$ Let $u$ be the normalized top eigenvector of $\widetilde{\Sigma}_r$, then together with \Cref{throrem:conclusion of davis-kahan} (observing that $\lambda_1(\Sigma)-\lambda_2(\Sigma)=12\|x^*\|^4$), we get that 
 \[dist(u,x^*/\|x^*\|)\leq \frac{2\sqrt{2}\|\Sigma-\widetilde{\Sigma}_{r}\|_{\mathrm{op}}}{12\|x^*\|^4}\leq \frac{1}{27}.\]
If the norm $\|x^*\|$ was known and $u$ could be estimated exactly, then the above calculation shows that the estimate $\|x^*\| \cdot u$ satisfies $dist(\|x^*\| \cdot u,x^*) \leq \|x^*\|/27,$ which suffices as the input $x_0$ to the algorithm of \Cref{thm:grad}. This suggests that we first estimate $\|x^*\|$ and use this estimate to define the truncation parameter. Use this parameter to truncate the r.v. $ya$, and then apply \Cref{theorem: Guarantees of Algorithm 2 from kong robust 2020 paper 1} to obtain the direction for $x^*$, and finally, combine the direction and scaling to get the final estimate of $x^*$. See \Cref{algorithm: Spectral Initialisation with Robust PCA for additive zero mean noise 1} for the full algorithm. We now present the main theorem analyzing \Cref{algorithm: Spectral Initialisation with Robust PCA for additive zero mean noise 1} whose proof can be found in \Cref{proof: proof of the theorem Robust PCA Spectral initialisation}.
\begin{algorithm}[t]
\caption{Spectral Initialization.}
\label{algorithm: Spectral Initialisation with Robust PCA for additive zero mean noise 1}
\textbf{Inputs:} failure probability $\delta \in (0, 1)$, corruption fraction $\epsilon > 0$, an upper bound $r_{up}$ on $K_4/\|x^*\|^2$, access to 2 batches of samples of sizes $m_{1}=\Omega(\log(2/\delta)(1+r_{up}^2))$ and $m_{2}=\Omega\left( n(r_{up}^2+1)\log(2n/(\delta \epsilon))/\epsilon\right)$ respectively from \Cref{def:error_model}.\\
\textbf{Output:} $x_0 \in \mathbb{R}^n$
\begin{algorithmic}[1]
\State \textbf{$\|x^*\|$ calculation}:
\begin{itemize}
    \item Receive a set of samples $T_y = \{\left(y_j\right)\}_{j=1}^{m_1}$ of size $m_1$ from the robust phase retrieval \Cref{def:error_model}.
    \item Let $\hat{\|x^*\|}^2$ be the maximum of zero and the robust mean estimate of the dataset $T_y$ by the Algorithm in \Cref{thm:robust_mean}.
    \item Calculate $\hat{\|x^*\|}=\sqrt{\hat{\|x^*\|}^2}$ as robust estimate of $\|x^*\|$.
\end{itemize}
    
    \State \textbf{Robust PCA}:
\begin{itemize}
    \item Calculate the truncation parameter $\hat{\tau}=O(\sqrt{n }(r_{up}^2+1)\hat{\|x^*\|}^2)$.
    \item Receive a set of samples $T = \left\{\left(a_{j},y_j\right)\right\}_{j=1}^{m_{2}}$ of size $m_{2}$ from the robust phase retrieval \Cref{def:error_model}.
    \item  Form the truncated dataset $\widetilde{D}=\{\widetilde{X}_{j}=X_{j}\mathbb{I}_{\|X_{j}\|\leq \hat{\tau}}:X_{j}=y_{j}a_{j}\}_{j=1}^{m_{2}}$. 
    \item Let $\hat{u}$ be the output of the robust PCA  algorithm (Algorithm 2 from \cite{kong2020robust}) applied to the truncated dataset $\widetilde{D}$.
\end{itemize}
\State \textbf{Scaling}: Return $x_0=\hat{\|x^*\|}\hat{u}$.
\end{algorithmic}
\end{algorithm}
\begin{theorem}
\label{theorem:Robust PCA Spectral initialisation}
    (\textbf{Robust spectral initialization for phase retrieval}.)  Under the setting and notations given in \Cref{algorithm: Spectral Initialisation with Robust PCA for additive zero mean noise 1}, if $\epsilon \leq C_1 (r_{up}^2+1)^{-2}$, $m_{1}\geq C_{2}\log(2/\delta)(1+r_{up}^2)$ and \\$ m_{2} \geq C_3 \left( n(r_{up}^2+1)\log(2n/\delta \epsilon)/\epsilon\right)$. Then, with probability at least $1-\delta$, $dist\left(x_0, x^*\right) \leq\left\|x^*\right\| / 9 .$  The time complexity of \Cref{algorithm: Spectral Initialisation with Robust PCA for additive zero mean noise 1} is $O(T_{\mathrm{rob\text{-}mean}}(m_{1},1)+m_2^2 n). $
\end{theorem}
%The proof of the \Cref{theorem:Robust PCA Spectral initialisation} can be found in \Cref{proof: proof of the theorem Robust PCA Spectral initialisation}. 
%Now the proof of 
\Cref{theorem: formal version of informal theorem using kong robust} now follows by combining \Cref{thm:grad} and \Cref{theorem:Robust PCA Spectral initialisation}.
\subsection{Connection between stability-based robust PCA and robust phase retrieval}
\label{sec: dong}
In this section, we show a connection (\Cref{prop:stab_to_pr}) between near-linear time stability-based robust PCA algorithms and robust spectral initialization. \cite{robust_PCA_Algorithm_in_nearly_linear_time_by_Diakonikolas} studied the robust PCA problem and gave a nearly linear time robust PCA algorithm (Algorithm 2 (SampleTopEigenvector)) given that the empirical uniform distribution over good samples satisfies the stability condition. The stability condition from \cite{robust_PCA_Algorithm_in_nearly_linear_time_by_Diakonikolas} is as follows:
\begin{definition}
\label{definition: Stability for nearly linear time robust PCA algorithm}
    (See Definition 2.12 from \cite{robust_PCA_Algorithm_in_nearly_linear_time_by_Diakonikolas}.) \textbf{Stability Condition.} Let $0<\epsilon<1 / 2$ and $\epsilon \leq \gamma<1$. A distribution $G$ on $\mathbb{R}^d$ is called ( $\epsilon, \gamma$ )-stable with respect to a PSD matrix $\boldsymbol{\Sigma} \in \mathbb{R}^{d \times d}$, if for every weight function $w: \mathbb{R}^d \rightarrow[0,1]$ with $\mathbf{E}_{X \sim G}[w(X)] \geq 1-\epsilon$, the weighted second moment matrix, $\mathbf{\Sigma}_{G_w}:=$ $\mathbf{E}_{X \sim G}\left[w(X) X X^{\top}\right] / \mathbf{E}_{X \sim G}[w(X)]$, satisfies that $(1-\gamma) \boldsymbol{\Sigma} \preceq \boldsymbol{\Sigma}_{G_w} \preceq(1+\gamma) \boldsymbol{\Sigma}$.
\end{definition}
The Algorithm 2 (SampleTopEigenvector) from \cite{robust_PCA_Algorithm_in_nearly_linear_time_by_Diakonikolas} (whose detailed gurantees can be found in \Cref{proof: fromal version of Algorithm 2 from diakonikolas paper.}) returns $u$ with $u^{\top} \boldsymbol{\Sigma} u \geq(1-O(\gamma))\|\boldsymbol{\Sigma}\|_{\mathrm{op}}$ in time $O\left(\frac{m' n}{\gamma^2} \log ^4(n / \epsilon)\right)$, where $m' $ is the no of input samples given that the uncoruupted samples satisfies \Cref{definition: Stability for nearly linear time robust PCA algorithm}. 
%See the \Cref{proof: fromal version of Algorithm 2 from diakonikolas paper.} for the details of guarantees given by Algorithm 2 from \cite{robust_PCA_Algorithm_in_nearly_linear_time_by_Diakonikolas}.

\begin{proposition}
\label{prop:stab_to_pr}
 \textbf{(Informal.)}   If the truncated empirical distribution of size $m$ satisfies the stability condition (\Cref{definition: Stability for nearly linear time robust PCA algorithm}), then we obtain an algorithm for robust phase retrieval with runtime $\tilde{O}(mn)$, sample complexity $\tilde{O}(m)$, and contamination tolerance $\epsilon \leq C_{2}/(1+r_{up}^2)$. The algorithm is $A_{alt}+$ Algorithm 2 from \cite{bunaR24_robust_phase}, where $A_{alt}$ is identical to \Cref{algorithm: Spectral Initialisation with Robust PCA for additive zero mean noise 1} except that Step 2 replaces Algorithm 2 of \cite{kong2020robust} with SampleTopEigenvector (Algorithm 2) from \cite{robust_PCA_Algorithm_in_nearly_linear_time_by_Diakonikolas}. 
\end{proposition}
 See \Cref{proof: formal version of proposition and its proof } for the formal statement and proof of \Cref{prop:stab_to_pr}.
 Such a proposition is also implicit in the work of \cite{Dong2024outlierrobust}. \cite{Dong2024outlierrobust} claimed a \emph{near-linear} time for the phase retrieval problem in the zero noise settings with adversarial corruption. They reduced the problem of robust spectral initialization (their Lemma 4.1)  to an optimization problem involving the Ky Fan 2 norm problem, which can be solved in \emph{linear time} by packing SDP solvers. However, their reduction crucially relies on their stability-based lemma (Lemma 4.2). 
 %from \cite{Dong2024outlierrobust}.

By \Cref{prop:stab_to_pr}, we conclude that to obtain nearly-linear time algorithms with sample complexity $\tilde{O}(n)$ for robust phase retrieval, it suffices to prove the stability condition (\Cref{definition: Stability for nearly linear time robust PCA algorithm}) for the empirical distribution with $\widetilde{O}(n)$ truncated samples, specifically, to answer in the affirmative the following question: \footnote{\cite{Dong2024outlierrobust} claims a proof of the open question for the noiseless case ($\zeta = 0$) in their manuscript. However, their argument is invalid, as discussed on page 6.}
\begin{question}
\label{conj}
Let $\tau$ be $\Theta(\sqrt{n})$. Can we prove or disprove that the truncated empirical distribution of $Y = (\langle a, x^* \rangle^2+\zeta) a\mathbb{I}_{\|(\langle a, x^* \rangle^2+\zeta) a\|\leq \tau}$ over $\tilde{O}(n)$ samples satisfies \Cref{definition: Stability for nearly linear time robust PCA algorithm} with respect to \(\Sigma := \|x^*\|^2\mathbb{I}_{n}+2 x^* x^{*\top}\)? 
\end{question}
\vspace{-0.5 em}
To answer the question partly, we show that the empirical distribution over $\tilde{O}(n^5)$ samples of the random variable $Y$
%= (\langle a,x^* \rangle^2+\zeta) a \mathbb{I}_{\|(\langle a, x^* \rangle^2+\zeta) a\|\leq \hat{\tau}}$
satisfies the stability definition \Cref{definition: Stability for nearly linear time robust PCA algorithm}. We  prove this formally in Lemma \ref{lemma:stability lemma for uncorrupted truncated samples}.\footnote{
In a concurrent work \cite{dong2025outlier}, the authors of \cite{Dong2024outlierrobust} claim to show that the truncated empirical distribution satisfies the stability condition (\cref{definition: Stability for nearly linear time robust PCA algorithm}) with $\tilde{O}(n)$ samples, providing an affirmative answer to \cref{conj} in the noise-free ($\zeta = 0$) setting.
} \Cref{lemma:stability lemma for uncorrupted truncated samples} leads to the \Cref{theorem:Robust PCA Spectral initialisation using stability lemma proof} (which can be found with proof in \Cref{proof: proof of spectral initialization using linear time robust pca algorithm}), which shows that the algorithm $A_{alt}$ suffices for robust spectral initialization (which, along with \Cref{thm:grad} gives an algorithm for robust phase retrieval) 
 and runs in $\tilde{O}(T_{\mathrm{rob\text{-}mean}}(m_{1},1)+m_{2}n)$ time but at the cost of higher sample complexity $m_0 = m_1 + m_2 =  \Omega(n^5).$   
 %See  
%\Cref{proof: proof of spectral initialization using linear time robust pca algorithm} for the formal version of \Cref{theorem:Robust PCA Spectral initialisation using stability lemma proof} and its proof. 
%\vspace{-1 em}

\section{Robust Phase Retrieval with non-zero mean noise}
\label{sec: non mean zero noise}

In this section, we study \Cref{def:error_model_nonzero_mean}. Using the symmetrization trick from \cite{bunaR24_robust_phase}, we reduce phase retrieval to a special case of blind deconvolution and discuss the associated statistical challenges. We then present the formal version of \Cref{theorem: informal for non-zero mean noise using robust pca from kong paper} and develop the necessary tools to establish it.  

When the noise has an \emph{unknown} (possibly non-zero) mean, the approach in \Cref{algorithm: Spectral Initialisation with Robust PCA for additive zero mean noise 1} becomes invalid, since $\mathbb{E}[y] = \|x^*\|^2 + \mu$, where $\mu$ is unknown.
%and its steps rely critically on the zero-mean assumption.

To address this issue, we follow the idea of \emph{symmetrization-based pre-processing} from \cite{bunaR24_robust_phase}. This idea has been used before, notably by \cite{pensia2024robust} for robust linear regression.

Suppose we are given two independent observations:
$
y = (\langle a^*, x^* \rangle)^2 + z, \quad y' = (\langle a^{*'}, x^* \rangle)^2 + z',
$
where $z, z'$ are independent noise terms with unknown mean. Define:
$
\upsilon := \frac{y - y'}{2}, \quad 
b := \frac{a^* + a^{*'}}{\sqrt{2}}, \quad 
c := \frac{a^* - a^{*'}}{\sqrt{2}}, \quad 
\zeta := \frac{z - z'}{2}.
$ Using the identity
$
(\langle a^*, x^* \rangle)^2 - (\langle a^{*'}, x^* \rangle)^2 = \langle a^* + a^{*'}, x^* \rangle \cdot \langle a^* - a^{*'}, x^* \rangle,
$
we can rewrite the above as and obtain a new model of the form:
$
\upsilon = \langle b, x^* \rangle \cdot \langle c, x^* \rangle + \zeta.
$
where: $b, c \sim \mathcal{N}(0, I_d)$ are independent, $\mathbb{E}[\zeta\mid b,c]=0,\mathbb{E}[\zeta^2\mid b,c]=\sigma^2/2, \mathbb{E}[\zeta^4\mid b,c]\leq 2K_{4}^4$.
This is called the blind deconvolution problem \cite{bunaR24_robust_phase}. We construct a new sample $S' = \{(b_j, c_j, \upsilon_j)\}_{j=1}^m$ from the original sample $S = \{(a^*_j, y_j)\}_{j=1}^{2m}$ by defining:
$
b_j := a^*_j + a^*_{m+j}/\sqrt{2}, \quad
c_j := a^*_j - a^*_{m+j}/\sqrt{2}, \quad
\upsilon_j := y_j - y_{m+j}/2.
$

Importantly, if the original dataset is $\varepsilon$-corrupted, the new dataset becomes $2\varepsilon$-corrupted. Now, we will state some statistical properties of the above-defined random variables, which will be needed later.
\begin{proposition}
\label{prop: some properties of random avriable}
\begin{equation}
\label{eq: property 0}
   \mathbb{E}[v]=0,  \mathbb{E}[\upsilon b]= 0, \mathbb{E}[v^2]=\|x^*\|^4+\sigma^2/2,
\end{equation}
\begin{equation}
\label{eq: property 1}
\operatorname{Var}(v^2)\leq 72(\|x^*\|^8+K_{4}^4),   \mathbb{E}[\upsilon b^\top c]=\|x^*\|^2,
\end{equation}
\begin{equation}
\label{eq: variance}
   \operatorname{Var}(vb^{\top}c)=n\|x^*\|^4+n(\sigma^2/2)+7\|x^*\|^4,
\end{equation}
\begin{equation}
\label{eq: property 3}
    \operatorname{Cov}[\upsilon b]= (\|x^*\|^4+\sigma^2/2)\mathbb{I}_{n}+2\|x^*\|^2 x^*x^{*\top}.
\end{equation}
\end{proposition}

Note that the population loss $r(x):=\mathbb{E}[(v-\langle b,x\rangle\langle c,x\rangle)^2]/2$ is strongly convex and smooth inside a small ball around $x^*$ (see \Cref{prop:poprisk for no corruption for non zero mean noise case} and \Cref{prop: local structure of blind deconvolution} for details), implying the success of robust spectral initialization + robust gradient descent.

 We will use the result (Theorem 6.3) of \cite{bunaR24_robust_phase} (stated in \Cref{subsection: fromal version of robust gradient descent for blind deconvolution}) who showed that there exists a polynomial time algorithm (Algorithm 4: robust gradient descent by \cite{bunaR24_robust_phase}) with sample complexity $\tilde{O}(n)$ for blind deconvolution problem if given a point $x_0$ inside a ball of radius $\|x^*\|/9$ centered at $\pm x^*$. 
 %We defer the formal version of Theorem 6.3 of \cite{bunaR24_robust_phase} in \Cref{subsection: fromal version of robust gradient descent for blind deconvolution} because it has the same essence as the \Cref{thm:grad}.
 
We are now ready to state the formal  
version of \Cref{theorem: informal for non-zero mean noise using robust pca from kong paper}.
\begin{theorem}
\label{theorem: formal version of blind deconvolution theorem}
 Consider the same settings as \Cref{theorem: formal version of informal theorem using kong robust}. There exists a polynomial-time algorithm (namely, \Cref{algorithm: Spectral Initialisation with Robust PCA for additive non-zero mean noise 1} (ours work) combined with Algorithm 4 (robust gradient descent) from \cite{bunaR24_robust_phase}) such that, if the total number of samples is \( m = 2m_1+2m_2 + 2P\tilde{m} \), and the following conditions hold:
 $m_1 \geq C_2 \log(2/\delta)(1+r_{up}^2)  ,  m_2 \geq C_3 \frac{n(r_{up}^2+1)^{2}\log(2n/(\delta \epsilon))}{\epsilon}, 2\epsilon \leq \frac{C_1}{(r_{up}^2 + 1)^2},  \quad \tilde{m} \geq C_4  \cdot \max\{n \log n,\ \log(1/\delta)\} \cdot r_{up}^2,
$
then, with probability at least \(1 - (P+1)\delta\), the algorithm outputs \(x_P\) satisfying \Cref{eq:guarantees_stat_error}.
Moreover, the time complexity of the algorithm is \( O(T_{\mathrm{rob\text{-}mean}}(m_{1},1)+m_2^2 n + P T_{\mathrm{rob\text{-}mean}}(\tilde{m},n)) \).
\end{theorem}
\subsection{Robust spectral initialization for non-mean zero case}
\label{subsec: non mena zero spectral initialzation}
Proposition \ref{prop: some properties of random avriable} tells us that $x^*$ is the top eigenvector of $\Sigma := \operatorname{Cov}(vb)$. Motivated by \Cref{sec:alg}, define the random vector $X := vb$ and let $\tau$ be the truncation parameter, which we will fix later.

\begin{remark}
\label{gap in blind convolution by Buna an dRebashini}
    In \Cref{algorithm: Spectral Initialisation with Robust PCA for additive zero mean noise 1}, the truncation parameter explicitly depends on an estimate of $\|x^*\|$. The natural idea of estimating $\|x^*\|$ stems from the fact that $\mathbb{E}[\upsilon b^\top c]=\|x^*\|^2$ (\Cref{eq: property 1}), which has also appeared in the work of \cite{bunaR24_robust_phase}.\footnote{ In \cite{bunaR24_robust_phase}, step 3 of Procedure 3 , they used \Cref{thm:robust_mean} to robustly estimate $\mathbb{E}[\upsilon b^\top c]$. As noted in \Cref{gap in blind convolution by Buna an dRebashini}, the guarantees of \Cref{thm:robust_mean} is not strong enough to get a sufficiently accurate robust estimate of $\|x^*\|^2$.
} 
%For a random variable $X$, $\|\Sigma\|_{\mathrm{op}}=\operatorname{Var}(X)$.
Since $\operatorname{Var}(vb^{\top}c)=\Theta(n)$ (\Cref{eq: variance}) and the upper bound on the robust estimation error given by \Cref{thm:robust_mean} is $O(\sqrt{\|\Sigma_{vb^Tc}\|_{\mathrm{op}}\epsilon})=O(\sqrt{\operatorname{Var}(vb^Tc)\epsilon})=O(\sqrt{n\epsilon})$, it implies that obtaining a sufficiently accurate approximation of $\|x^*\|^2$ would require a contamination tolerance below $1/n$\footnote{Moreover, other natural first order statistics like $\mathbb{E}[vb], \mathbb{E}[vc]$ and other zero order statistics like $\mathbb{E}[v]$ contains no information about $\|x^*\|$.}. Hence, $\|x^*\|$ is not easy to directly estimate accurately, unlike in phase retrieval, and a modification to \Cref{algorithm: Spectral Initialisation with Robust PCA for additive zero mean noise 1} is needed for robust blind deconvolution.
%for blind deconvolution.
%cannot be directly applied to the blind deconvolution case, and a modification is necessary.
\end{remark}
\textbf{Our approach:} We first get a crude estimate of $\|x^*\|^4$ by using the fact that $\mathbb{E}[v^2] = \|x^*\|^{4}+\sigma^2/2$ and a known upper bound on the noise to signal ratio, and use it to get the truncation parameter (\Cref{lemma: trunc for non mean zero noise} whose proof can be found in \Cref{proof: proof of the truncation lemma fro non-zero mean noise case}). We then apply robust PCA (\Cref{theorem: Guarantees of Algorithm 2 from kong robust 2020 paper 1}) to the truncated random variable $vb$ to obtain accurate estimates of both the top eigendirection (parallel to $x^*$) and the top eigenvalue $3\|x^*\|^4 + \sigma^2/2$ of $\Sigma = \operatorname{Cov}(vb)$. Together with an accurate estimate of $\mathbb{E}[v^2] = \|x^*\|^{4}+\sigma^2/2$, this yields an accurate estimate of $\|x^*\|^4$. At last, we combine direction and scaling to get the desired output.
\begin{lemma}
\label{lemma: trunc for non mean zero noise}
\textbf{(Truncation.)} Define $\widetilde{X} := X\mathbb{I}_{\|X\|\leq \tau}$, 
    $\widetilde{\Sigma}_r:=\mathbb{E}[\widetilde{X}\widetilde{X}^{\top}]$,
    and $p :=\mathbb{P}(\|X\|\geq \tau)$, where $\tau$ is the truncation parameter. Then, the following holds:   
    \begin{equation}
    \label{eq:trunc_bound for the noise with non-zero mean}
        \|\Sigma-\widetilde{\Sigma}_r\|_{\mathrm{op}}=O( \sqrt{n}(\sigma^2/2+ \|x^*\|^4)^{1/2}(K_{4}^2+\|x^*\|^4)/\tau).
    \end{equation}
     \end{lemma}
%The proof of the \Cref{lemma: trunc for non mean zero noise} can be found in .

We now give the details. Assuming a known upper bound on the noise to signal ratio $K_{4}/\|x^*\|^2\leq r_{up}$, the choice of  $\tau=O( \sqrt{n}(\sigma^2/2+\|x^*\|^4)^{1/2}(r_{up}^2+1))$ (which can be robustly estimated %by robustly estimating $\sigma^2/2+\|x^*\|^4$ 
by applying \Cref{thm:robust_mean} to the random variable $v^2$, since $\mathbb{E}[v^2] = \sigma^2/2+\|x^*\|^{4}$ and this estimate is sufficiently accurate because $\operatorname{Var}(v^2)=O(\|x^*\|^{8})$), and using \Cref{throrem:conclusion of davis-kahan} and \Cref{lemma: trunc for non mean zero noise}, we can say that the top eigenvector $u$ of the truncated covariance matrix $\widetilde{\Sigma}_r$ satisfies $dist(\|x^*\|u,x^*)\leq \frac{2\sqrt{2}\|x^*\|\|\Sigma-\widetilde{\Sigma}_r\|_{\mathrm{op}}}{2\|x^*\|^4}\leq \frac{\|x^*\|}{27}.$ Thus, $u$ is close to the direction of $x^*$, and applying robust PCA (\Cref{theorem: Guarantees of Algorithm 2 from kong robust 2020 paper 1}) on $\widetilde{X}$ yields a robust estimate $\hat{u}$ of the direction of $x^*$.

 %So, we need to only estimate $\sigma^2/2+\|x^*\|^4$ from data to get the truncation parameter. 
 
 %Our approach is to first estimate $\sigma^2/2+\|x^*\|^4$, which is easy to estimate because ,  and moreover \Cref{thm:robust_mean} implies that we can get a good estimate, and after that use the estimate to set the truncation parameter, and then approximate the direction $\hat{u}$ of $x^*$ via \Cref{theorem: Guarantees of Algorithm 2 from kong robust 2020 paper 1}. 
 
% Since $\|\Sigma\|_{\mathrm{op}} = 3\|x^*\|^4+\sigma^2/2$, a good estimate of both $\|\Sigma\|_{\mathrm{op}}$ and $v^2$ allows us to recover $\|x^*\|^4$ by taking their difference of estimates divided by 2. So the only thing remains to get a good estimate of $\|\Sigma\|_{\mathrm{op}}$. 
 
 We now accurately estimate $\|x^*\|$. Note that $\|\Sigma\|_{\mathrm{op}} = 3\|x^*\|^4+\sigma^2/2$. 
By \Cref{theorem: Guarantees of Algorithm 2 from kong robust 2020 paper 1}, 
the robust estimate  $\lambda_{\hat{u}}$ of the top eigenvalue of $\widetilde{\Sigma}_r$ satisfies 
$|\lambda_{\hat{u}}-\|\widetilde{\Sigma}_r\|_{\mathrm{op}}|
=O(\sqrt{\epsilon}\|x^*\|^4)$. Combining this with \Cref{eq:trunc_bound for the noise with non-zero mean} and 
the triangle inequality yields 
$
|\lambda_{\hat{u}}-\|\Sigma\|_{\mathrm{op}}|
\leq |\lambda_{\hat{u}}-\|\widetilde{\Sigma}_r\|_{\mathrm{op}}|
+|\|\Sigma\|_{\mathrm{op}}-\|\widetilde{\Sigma}_r\|_{\mathrm{op}}|
\leq c\|x^*\|^4,
$
for as small as needed constant $c$,
by the choice of $\tau$ and sufficiently small $\epsilon$.
%for a small constant $c$ depending on $\tau,\epsilon$. 
%Thus, the eigenvalue estimate $\lambda_{\hat{u}}$ is close to $\|\Sigma\|_{\mathrm{op}}$, 
This suggests the robust estimator
$
\hat{\|x^*\|}=((\lambda_{\hat{u}}-\hat{v}^2)/2)^{1/4}
$
 for $\|x^*\|$, where $\hat{v}^2$ estimates $\mathbb{E}[v^2] = \|x^*\|^4 + \sigma^2/2$. 
See \Cref{algorithm: Spectral Initialisation with Robust PCA for additive non-zero mean noise 1} 
for the full procedure. We now present the final theorem for spectral initialization whose proof can be found in \Cref{subsection: formal algorithm and its proof}.
  % Using  ( given by \Cref{theorem: Guarantees of Algorithm 2 from kong robust 2020 paper 1}), 
\begin{theorem}
\label{theorem:Robust PCA Spectral initialisation for non-zero mean noise}(\textbf{Robust spectral initialization for blind deconvolution}.)  Under the setting and notation of \Cref{algorithm: Spectral Initialisation with Robust PCA for additive non-zero mean noise 1}, if $2\epsilon \leq C_1 (r_{up}^2+1)^{-2}$, $2m_{1}\geq C_{2} \log(3/\delta)(1+r_{up}^2)$ and $ 2m_{2} \geq C_3 \left( n(r_{up}^2+1)^2\log(3n/(\delta \epsilon))/\epsilon\right)$. Then, with probability at least $1-\delta$, $dist\left(x_0, x^*\right) \leq\left\|x^*\right\| / 9 .$ The algorithm use $2m_{1}$ samples to estimate $\mathbb{E}[v^2]$ and $2m_{2}$ samples to estimate $\hat{u}$ and $\lambda_{\hat{u}}$.
. The time complexity of \Cref{algorithm: Spectral Initialisation with Robust PCA for additive non-zero mean noise 1} is $O(T_{\mathrm{rob\text{-}mean}}(m_{1},1)+m_2^2 n+m_{1}n+m_{2}n). $
\end{theorem}
%The proof of the \Cref{theorem:Robust PCA Spectral initialisation for non-zero mean noise} can be found in \Cref{subsection: formal algorithm and its proof}.
\Cref{theorem: formal version of blind deconvolution theorem} now follows by combining \Cref{thm:grad blind deconvolution } and \Cref{theorem:Robust PCA Spectral initialisation for non-zero mean noise}.

\section{Conclusion}
\label{sec:conc}
%\vspace{-1em}
We established a connection between robust phase retrieval (under heavy-tailed noise) and the recovery of the top eigenvector of a covariance matrix from corrupted samples, yielding the first polynomial-time algorithms that tolerate a constant corruption fraction depending on the noise-to-signal ratio (independent of dimension). We show the versatility of this approach by handling non-zero mean noise, which requires robustly solving  blind deconvolution.
  %The truncation parameter in \Cref{def:error_model_nonzero_mean} is smaller than in \Cref{def:error_model}, making its spectral initialization sample complexity larger by a factor of $(1+r_{up}^2)$.
 An open question is whether the corruption level can be made independent of the noise-to-signal ratio. The known information-theoretic lower bound is $O(\sigma \epsilon)$ for robust Gaussian linear regression \cite{bakshi2021robust}, which extends to phase retrieval for $\|x^*\|=1$ case. Whether a sharper bound exists, or if there exists an algorithm that can achieve an error of $O(\sigma \epsilon/\|x^*\|)$ —remains open.

\section{Acknowledgment}
We thank Pravesh Kothari for useful discussions during the authors’ visit to the International Centre for Theoretical Sciences, Tata Institute of Fundamental Research (ICTS-TIFR), Bangalore, for the conference GEOMETRY, PROBABILITY, AND ALGORITHMS. We also acknowledge the support of the Department of Atomic Energy, Government of India, under Project No. RTI4014.
\bibliographystyle{plain}  % or alpha, ieeetr, etc.
\bibliography{references}

%%%%%%%%%%%%%%%%%%%%%%%%%%%%%%%%%%%%%%%%%%%%%%%%%%%%%%%%%%%%

\clearpage
\appendix
\thispagestyle{empty}

% Supplementary material: To improve readability, you must use a single-column format for the supplementary material.
\onecolumn
\section{Other related work}
\label{sec:rel_work}
\textbf{Phase retrieval.} The literature on phase retrieval is both vast and diverse. Articles \cite{balan_radu_casazza_2006_signal_construction_without_phase} and \cite{conca_aldo_edidn_2015_algebraic_characterization_of_injectivity_in_phase_retrieval} have shown that \( m \geq 2n - 1 \) or \( m \geq 4n - 4 \) measurements \( \left\{\left(a_i, y_i\right)\right\}_{i=1}^m \) are sufficient for uniquely determining an \( n \)-dimensional real-valued or complex-valued vector \( x^* \), respectively, in the noise less settings. In the noiseless and no corruption setting, a wide array of algorithms have been proposed, including semidefinite programming (SDP) methods \cite{candesphaselifttight2014} and nonconvex methods such as Wirtinger Flow \cite{CLS15}, trust-region methods \cite{sun2018geometric}, alternating projections \cite{Wa18}, and variants like reshaped Wirtinger Flow \cite{zhang2017nonconvex}, truncated amplitude flow (TAF) \cite{wang2017solvingsystemofquadraticequationsviaatruncatedmplitdeflow} and AltMinPhase with resampling \cite{NJS15}. All of the above algorithm has sample complexities $\widetilde{O}(n)$. Beyond the noiseless setting, different noise models have been explored. \cite{wuR23_noisy} addressed signal-independent light-tailed noise and proposed a mirror descent algorithm for sparse phase retrieval with sample complexity $\tilde{O}(k^2)$, where $\|x^*\|_{0}=k$. Several works also studied bounded noise: \cite{zhang2017nonconvex} and \cite{wang2017solvingsystemofquadraticequationsviaatruncatedmplitdeflow} also considered $|\zeta|_\infty = O(\|x^*\|)$, whereas \cite{GFAZ24} considered the case $|\zeta|_\infty = O(\|x^*\|^2)$ and proposed an efficient algorithm (spectral initialization and mirror descent) with the sample complexity $\tilde{O}(n)$. \cite{chen_candes_2015_solving} presented a general framework that includes bounded and sub-exponential signal-dependent noise, highlighting two notable cases: $|\zeta|_\infty = O(\|x^*\|^2)$ and Poisson noise where $y_i \sim \text{Poisson}(|\langle a_i, x^* \rangle|^2)$. Their truncated Wirtinger Flow (TWF) algorithm uses gradient descent on a truncated empirical log-likelihood, discarding outliers at each step. In a related direction, \cite{DNNB23} studied a heteroscedastic regression model with multiplicative noise of the form $y_i^2 = \chi^2(1) (\langle a_i, x^*\rangle)^2$, where $\chi^2(1)$ is a mean-one sub-exponential variable. 
%They proposed a new loss function with a customized gradient descent method, and their analysis relies crucially on concentration properties of sample means.

\textbf{Robust statistics.} Robust parameter estimation under heavy-tailed noise and adversarial corruption is a central theme in robust statistics, aiming for reliable learning under real-world data challenges. Suppose a dataset \( \{x_i \in \mathbb{R}^d\}_{i=1}^n \) is drawn i.i.d.\ from a distribution \( \mathcal{D}_\theta \), and an \( \epsilon \)-fraction of the samples is arbitrarily corrupted. The goal is to robustly estimate \( \theta \) from these corrupted samples. Robust statistics aims to estimate the mean (first-order) and covariance (second-order) statistical quantities under heavy-tailed noise and adversarial corruption. While the Tukey median achieves optimal mean estimation, it is computationally intractable. \cite{diakonikolas2019robust} introduced the first polynomial-time algorithm with near-optimal error, which was later followed by nearly linear-time methods proposed in \cite{hopkins_2020_robust_and_heavy_tailed_mean_estimation}. For covariance estimation, \cite{oliveira2024improved} proposed a statistically optimal but computationally intractable estimator, with efficient algorithms known only for Gaussian distributions \cite{kothari2018robust,nearly_linear_time_covarience_estiamtion_for_gaussain}, and for general distributions under low contamination \cite{duchi2025fast}. Designing efficient algorithms for covariance estimation with optimal guarantees for arbitrary heavy-tailed distributions and constant corruption remains a key open problem. For principal component analysis (PCA), \cite{kong2020robust} presented a polynomial-time algorithm for bounded random vectors, achieving almost linear sample complexity and tolerance to a constant fraction of corruptions. On the other hand, \cite{robust_PCA_Algorithm_in_nearly_linear_time_by_Diakonikolas} developed a nearly linear-time algorithm under the assumption that the data satisfies a certain stability condition (see \Cref{definition: Stability for nearly linear time robust PCA algorithm} for details).

\section{Correct Error Scaling For Phase Retrieval.}
\label{sec: correct scaling}
In this section, we formally prove that any algorithm solving phase retrieval must have a 
$\tfrac{1}{\|x^{\ast}\|}$ dependence on the estimation error. Before presenting the proof, we first 
provide some intuition for this result using dimensional analysis. Note that the error scale 
$\sigma$ in phase retrieval is $\mathcal{O}(\|x^{\ast}\|^2)$, while 
$dist(x, x^{\ast}) = \mathcal{O}(\|x^{\ast}\|)$. This implies a 
$\tfrac{1}{\|x^{\ast}\|}$ dependence by dimensional analysis. We now proceed to prove this statement formally.

\begin{claim}
  Any algorithm that solves the phase retrieval problem (\Cref{def:error_model}) with an estimation error of order 
\( O(\sigma \sqrt{\epsilon}\|x^*\|^{\alpha}) \) must necessarily have \(\alpha = -1\); in other words, the estimation 
error must scale inversely with the signal norm, i.e., as \(1 / \|x^*\|\).

\end{claim}
\begin{proof}
To address the above claim, we consider the following four exhaustive cases and show that each of them leads to a contradiction:

\begin{enumerate}
    \item \textbf{Case 1.} $\mathrm{dist}(x, x^{\ast}) = \mathcal{O}(\sigma \sqrt{\epsilon})$.
    
    \item \textbf{Case 2.} $\mathrm{dist}(x, x^{\ast}) = \mathcal{O}(\sigma \sqrt{\epsilon}\|x^{\ast}\|^{\alpha})$, where $\alpha > 0$.
    
    \item \textbf{Case 3.} $\mathrm{dist}(x, x^{\ast}) = \mathcal{O}\!\left(\frac{\sigma \sqrt{\epsilon}}{\|x^{\ast}\|^{\alpha}}\right)$, where $0 < \alpha < 1$.
    
    \item \textbf{Case 4.} $\mathrm{dist}(x, x^{\ast}) = \mathcal{O}\!\left(\frac{\sigma \sqrt{\epsilon}}{\|x^{\ast}\|^{\alpha}}\right)$, where $\alpha > 1$.
\end{enumerate}

  \textbf{Case 1.} Suppose there exists a phase retrieval algorithm whose estimation error does not depend on the signal norm $\|x^{\ast}\|$. Specifically, the algorithm outputs an estimate $x_{\text{out}}$ such that
\begin{equation}
\label{eq: correct scaling}
  dist(x_{\text{out}},x^{\ast})=O\left(\sigma \sqrt{\epsilon}\right),  
\end{equation}

where $\sigma$ is the noise level, $\epsilon$ is the fraction of corrupted samples. Assume $\|x^{\ast}\|=1$, and there exists a dataset $S={(a_{i}, y_{i})}$ for this signal such that when the dataset $S$ is given to the algorithm, it outputs $x_{\text{out}}$ satisfying \Cref{eq: correct scaling}.

Now fix some $n \in \mathbb{N}$, and consider the rescaled signal $x^{\ast}/n$. The dataset $S$ can be converted into a dataset for this new signal simply by dividing each $y_{i}$ by $n^{2}$. In this case, the measurement model becomes
\[
\frac{y}{n^{2}}= \langle a, x^{\ast}/n \rangle^{2} + \frac{\zeta}{n^{2}}.
\]

Applying the algorithm to this rescaled dataset yields an output $\hat{x}^{*}_{n}$ such that
\[
dist(\hat{x}^{*}_{n}, x^{\ast}/n)= O\left(\frac{\sigma}{n^{2}} \sqrt{\epsilon}\right),
\]

which further implies
\[
dist(n\hat{x}^{*}_{n}, x^{\ast}) =O\left(\frac{\sigma}{n} \sqrt{\epsilon}\right).
\]

This shows that if there exists a algorithm for robust phase retrieval with absolute error independent of $\|x^{\ast}\|$ then it could be rescaled to solve the problem in the normalized setting $\|x^{\ast}\|=1$ to arbitrary precision. This contradicts the known information-theoretic lower bound of $\mathcal{O}(\sigma \epsilon)$ for Gaussian linear regression \cite{bakshi2021robust}, which is a strictly easier problem than phase retrieval.

\textbf{Case 2.} Suppose there exists a phase retrieval algorithm whose estimation error follows case 2. Specifically, the algorithm outputs an estimate $x_{\text{out}}$ such that
\begin{equation}
\label{eq: correct scaling case 2}
  dist(x_{\text{out}},x^{\ast})=O\left(\sigma \sqrt{\epsilon}\|x^*\|^{\alpha}\right),  
\end{equation}
where $\alpha > 0$. Assume $\|x^{\ast}\|=1$, and there exists a dataset $S={(a_{i}, y_{i})}$ for this signal such that when the dataset $S$ is given to the algorithm, it outputs $x_{\text{out}}$ satisfying \Cref{eq: correct scaling case 2}.

Now fix some $n \in \mathbb{N}$, and consider the rescaled signal $x^{\ast}/n$. The dataset $S$ can be converted into a dataset for this new signal simply by dividing each $y_{i}$ by $n^{2}$. In this case, the measurement model becomes
\[
\frac{y}{n^{2}}= \langle a, x^{\ast}/n \rangle^{2} + \frac{\zeta}{n^{2}}.
\]

Applying the algorithm to this rescaled dataset yields an output $\hat{x}^{*}_{n}$ such that
\[
dist(\hat{x}^{*}_{n}, x^{\ast}/n)= O\left(\frac{\sigma}{n^{2}} \sqrt{\epsilon}\frac{\|x^*\|^{\alpha}}{n^{\alpha}}\right)=O(\frac{\sigma}{n^{2+\alpha}} \sqrt{\epsilon}),
\]

which further implies
\[
dist(n\hat{x}^{*}_{n}, x^{\ast}) =O\left(\frac{\sigma}{n^{1+\alpha}} \sqrt{\epsilon}\right).
\]

This implies that if there exists an algorithm for robust phase retrieval whose estimation error satisfies \Cref{eq: correct scaling case 2}, then the algorithm can be appropriately rescaled to solve the problem in the normalized setting $\|x^{\ast}\| = 1$ with arbitrarily high precision. Thus, a contradiction arises using the same argument as in case 1.

\textbf{Case 3.} Suppose there exists a phase retrieval algorithm whose estimation error follows case 3. Specifically, the algorithm outputs an estimate $x_{\text{out}}$ such that
\begin{equation}
\label{eq: correct scaling case 3}
  dist(x_{\text{out}},x^{\ast})=O\left(\sigma \sqrt{\epsilon}/\|x^*\|^{\alpha}\right),  
\end{equation}
where $0<\alpha <1$. Assume $\|x^{\ast}\|=1$, and there exists a dataset $S={(a_{i}, y_{i})}$ for this signal such that when the dataset $S$ is given to the algorithm, it outputs $x_{\text{out}}$ satisfying \Cref{eq: correct scaling case 3}.

Now fix some $n \in \mathbb{N}$, and consider the rescaled signal $x^{\ast}/n$. The dataset $S$ can be converted into a dataset for this new signal simply by dividing each $y_{i}$ by $n^{2}$. In this case, the measurement model becomes
\[
\frac{y}{n^{2}}= \langle a, x^{\ast}/n \rangle^{2} + \frac{\zeta}{n^{2}}.
\]

Applying the algorithm to this rescaled dataset yields an output $\hat{x}^{*}_{n}$ such that
\[
dist(\hat{x}^{*}_{n}, x^{\ast}/n)= O\left(\frac{\frac{\sigma}{n^{2}} \sqrt{\epsilon}}{\frac{\|x^*\|^{\alpha}}{n^{\alpha}}}\right)=O(\frac{\sigma}{n^{2-\alpha}} \sqrt{\epsilon}),
\]

which further implies
\[
dist(n\hat{x}^{*}_{n}, x^{\ast}) =O\left(\frac{\sigma}{n^{1-\alpha}} \sqrt{\epsilon}\right).
\]

This implies that if there exists an algorithm for robust phase retrieval whose estimation error satisfies \Cref{eq: correct scaling case 3}, then the algorithm can be appropriately rescaled to solve the problem in the normalized setting $\|x^{\ast}\| = 1$ with arbitrarily high precision.
 Thus, a contradiction arises using the same argument as in case 1.

 \textbf{Case 4.} Suppose there exists a phase retrieval algorithm whose estimation error follows case 4. Specifically, the algorithm outputs an estimate $x_{\text{out}}$ such that
\begin{equation}
\label{eq: correct scaling case 4}
  dist(x_{\text{out}},x^{\ast})=O\left(\sigma \sqrt{\epsilon}/\|x^*\|^{\alpha}\right),  
\end{equation}
where $\alpha>1$. Assume $\|x^{\ast}\|=1$, and there exists a dataset $S={(a_{i}, y_{i})}$ for this signal such that when the dataset $S$ is given to the algorithm, it outputs $x_{\text{out}}$ satisfying \Cref{eq: correct scaling case 4}.

Now fix some $n \in \mathbb{N}$, and consider the rescaled signal $n x^{\ast}$. The dataset $S$ can be converted into a dataset for this new signal simply by multiplying each $y_{i}$ by $n^{2}$. In this case, the measurement model becomes
\[
y n^{2}= \langle a, x^{\ast}n \rangle^{2} + \zeta n^{2}.
\]

Applying the algorithm to this rescaled dataset yields an output $\hat{x}^{*}_{n}$ such that
\[
dist(\hat{x}^{*}_{n}, x^{\ast}n)= O\left(\frac{n^{2}\sigma \sqrt{\epsilon} }{\|x^*\|^{\alpha} n^{\alpha}}\right)=O(\sigma n^{2-\alpha}\sqrt{\epsilon}),
\]

which further implies
\[
dist(\hat{x}^{*}_{n}/n, x^{\ast}) =O\left(\sigma n^{1-\alpha}\sqrt{\epsilon}\right)=O\left(\frac{\sigma \sqrt{\epsilon}}{ n^{\alpha-1}}\right).
\]

This implies that if there exists an algorithm for robust phase retrieval whose estimation error satisfies \Cref{eq: correct scaling case 4}, then the algorithm can be appropriately rescaled to solve the problem in the normalized setting $\|x^{\ast}\| = 1$ with arbitrarily high precision. Thus, a contradiction arises using the same argument as in case 1.

\end{proof}

\section{Comments On Using Fresh Samples At Each Iteration Of Gradient Descent}
\label{sec:fresh samples}
\Cref{remark: use of fresh samples} informally states that using a constant number of batches of fresh samples is not fundamentally different from using a single batch. In particular, if we collect \(O(Tn)\) samples with \(T = O(1)\) and \(T \ll n\), and uniformly divide them (without replacement) into \(T\) batches, then each batch will contain at least a \((1+c)\epsilon\) fraction of corrupted samples, where $c$ is a very small positive constant.  We formally state the above claim below and provide a proof. 
\begin{claim}
    Let the total dataset consist of \(N = C T n \log(n)\) samples for some absolute constant \(C > 0\), out of which exactly \(K = \varepsilon N\) samples are corrupted. We construct \(T\) subsets (or batches) by partitioning the \(N\) samples uniformly at random, that is, by drawing \(C n \log(n)\) samples for each subset uniformly \emph{without replacement} from the dataset. Let \(X_j\) denote the number of corrupted samples in the \(j\)-th subset. If $n\log(n)\geq C_{1}\frac{\log(T)+\log(1/\delta)}{\epsilon^2}$, then, with probability at least \(1 - \delta\), every subset contains at most a \((1+c)\varepsilon\) fraction of corrupted samples, where \(c > 0\) is a constant depending on \(C\) and \(C_{1}\).
\end{claim}

\begin{proof}
   \(X_j\) be the number of corrupted samples in subset \(j\). So, \(\mathbb{P}(X_j=k)=\frac{\Mycomb[K]{k}\Mycomb[N-K]{Cn\log(n)-k}}{\Mycomb[N]{Cn\log(n)}}\). This implies $X_j$ follows Hypergeometric distribution with parameter \((N, K, Cn\log(n))\). So, \(\mathbb{E}[X_j]=\frac{C n\log(n)K}{N}=\epsilon C n\log(n)\).

\paragraph{Tail bound (Hypergeometric).}
Let $X \sim \operatorname{Hypergeometric}(N,K,n')$ and set $p = K/N$.  
For any $t$ with $0 < t < p$, the following tail bounds hold:
\begin{align*}
\mathbb{P}\!\big[X \le (p-t)n'\big] &\le \exp\!\left(-n'\,D(p-t \,\|\, p)\right), \\
\mathbb{P}\!\big[X \ge (p+t)n'\big] &\le \exp\!\left(-n'\,D(p+t \,\|\, p)\right),
\end{align*}where the binary Kullback–Leibler divergence is
$
D(q \,\|\, p) \;=\; q \log\frac{q}{p} + (1-q)\log\frac{1-q}{1-p}.
$ Moreover, since $D(p \pm t \,\|\, p) \ge 2t^2$, one obtains the simpler Hoeffding-type bounds:
\begin{align*}
\mathbb{P}\!\big[X \le (p-t)n'\big] &\le e^{-2t^{2}n'}, \\
\mathbb{P}\!\big[X \ge (p+t)n'\big] &\le e^{-2t^{2}n'}.
\end{align*}
So, applying the above tail bound for the hypergeometric distribution in our case implies for $0\leq t\leq \epsilon$,
\[\mathbb{P}[X_j\geq (\epsilon+t)C n\log(n)]\leq e^{-2Ct^{2}n\log(n)}.\]

\paragraph{Union bound over the \(T\) subsets.}
Applying the tail bound to each subset and union bounding,
\[
\mathbb{P}\big[\exists j\in[T]\ :\ X_j \geq (\epsilon+t)C n\log(n)\big]
\leq T\cdot e^{-2Ct^{2}n\log(n)} .
\]
Thus with probability at least \(1 - T\cdot e^{-2Ct^{2}n\log(n)}\), every subset satisfies
\[
\frac{X_j}{Cn\log(n)} \le (\epsilon+t).
\]
\paragraph{Choose \(\delta\) for a desired failure probability.}
Fix a target failure probability \(\delta\in(0,1)\). To ensure the RHS \(\le \delta\), pick \(\delta\) solving
\[
T \cdot e^{-2Ct^{2}n\log(n)}\le \delta
\quad\Longrightarrow\quad
t \ge \sqrt{\frac{(\log T + \log(1/\delta))}{2C n\log(n)}}.
\]
Hence, with probability at least \(1-\delta\), every subset has corrupted fraction
\[
\frac{X_j}{C n \log(n)} \le \Bigg(\varepsilon + \sqrt{\frac{(\log T + \log(1/\delta))}{ 2C n\log(n)}}\ \Bigg)=O((1+c)\epsilon).
\]
\end{proof}

\section{Extra Guarantees Of Robust PCA Algorithm}

\subsection{Proof of \texorpdfstring{\Cref{eq: kong robust eigenvalue}}.}
\label{proof: rigorous proof of eigenvector bound of kong robust pca}
As the \Cref{kong: eigenvalue} suggests that Algorithm~2 \cite{kong2020robust} has been designed to retrieve the top eigenvector for the case $k=1$. Here , we rigorously proof \Cref{eq: kong robust eigenvalue}. Before proving \Cref{eq: kong robust eigenvalue}, we state a lemma from \cite{kong2020robust} that will be used in the proof. The lemma is as follows:
\begin{lemma}[Main Lemma (Lemma C.4) for Algorithm 2 from \cite{kong2020robust}]\label{lem:main_alg2}
Consider the settings given in \Cref{theorem: Guarantees of Algorithm 2 from kong robust 2020 paper 1}.
 The Algorithm 2 from \cite{kong2020robust} outputs a dataset $S' \subseteq T$ satisfying the following for $\hat{M} = \frac{1}{|S'|}\sum_{X_i\in S'} X_i X_i^{\top}$:
\begin{enumerate}
    \item For the top normalized eigenvector $\hat{u} \in \mathbb{R}^{n\times 1}$ of $\hat{M}$,
    $$
    \Tr\left[ \hat{u}^\top \left(\hat{M} -\Sigma_{r} \right) \hat{u} \right] \le 48\epsilon \Tr\left[ \hat{u}^\top \Sigma_{r} \hat{u} \right] + 102 v'\sqrt{\epsilon}.
    $$
    \item For all unit-norm vectors $V \in \mathbb{R}^{n\times 1}$, we have
    $$
    \Tr\left[ V^\top \left(\hat{M} - \Sigma_{r}\right) V \right] \ge -10\epsilon \Tr[ V^\top \Sigma_{r} V ] - 8 v'\sqrt{\epsilon}.
    $$
\end{enumerate}
\end{lemma}
Now, we give a proof for \Cref{eq: kong robust eigenvalue}.
\begin{theorem}
\label{theorem: eigenvalue bound from algorithm of theorem} 
    In the settings of \Cref{theorem: Guarantees of Algorithm 2 from kong robust 2020 paper 1}, if we also output the eigenvalue $\lambda_{\hat{u}}$ of the top eigenvector $\hat{u}$, then 
    \[|\lambda_{\hat{u}}-\|\Sigma_r\|_{\mathrm{op}}|=O(\epsilon \|\Sigma_r\|_{\mathrm{op}}+\sqrt{\epsilon} v').\]
\end{theorem}
\begin{proof}
  Since $\Sigma_r$ is positive definite, the Eckart--Young--Mirsky theorem implies that
\[
P_1(\Sigma_r) = \argmin_{B:\,\mathrm{rank}(B)=1} \|\Sigma_r - B\|_{\mathrm{op}}
= \lambda_1 v_1 v_1^{\top},
\]
where $\lambda_1$ and $v_1$ denote the largest eigenvalue and the corresponding normalized eigenvector of $\Sigma_r$, respectively.  
Consequently,
$
\operatorname{Tr}[P_1(\Sigma_r)] = \lambda_1 = \|\Sigma_r\|_{\mathrm{op}}.
$
. By using \Cref{lem:main_alg2}, we can conclude that
    \begin{equation}\label{eq: coming from extra prop}
        |\lambda_{\hat{u}}-\hat{u}^{\top}\Sigma_r \hat{u}|\leq 48 \epsilon \hat{u}^{\top}\Sigma_r \hat{u}+102 \sqrt{\epsilon} v'\leq 48 \epsilon \|\Sigma_r\|_{\mathrm{op}}+102 \sqrt{\epsilon} v'.
    \end{equation}
  And using the \Cref{theorem: Guarantees of Algorithm 2 from kong robust 2020 paper 1}, we know that 
  \begin{equation}\label{eq: coming from exact prop}
  \|\Sigma_r\|_{\mathrm{op}}-\hat{u}^{\top}\Sigma_r \hat{u} =\operatorname{Tr}[P_1(\Sigma_r)]-\hat{u}^{\top}\Sigma_r \hat{u}\leq 58 \epsilon \operatorname{Tr}[P_1(\Sigma_r)]+110 \sqrt{\epsilon} v'=58 \|\Sigma_r\|_{\mathrm{op}}\epsilon +110 \sqrt{\epsilon} v'.
  \end{equation}
 Therefore, applying the triangle inequality together with \Cref{eq: coming from extra prop} and \Cref{eq: coming from exact prop}, we obtain that
   \[|\lambda_{\hat{u}}-\|\Sigma_r\|_{\mathrm{op}}|\leq 106 \epsilon \|\Sigma_r\|_{\mathrm{op}}+212\sqrt{\epsilon} v'=O( \epsilon \|\Sigma_r\|_{\mathrm{op}}+\sqrt{\epsilon} v').\]
    \end{proof}

\section{Robust Phase retrieval under Zero-Mean Noise}
\subsection{Proof of \texorpdfstring{\Cref{lem:trunc}}.}
\label{proof: truncation lemma}
Before proceeding with the proof, we state a special case of \Cref{lemma: calcualtions needed for spectral initialization with covest estimator}, which will be used repeatedly. Specifically, for $i = 8$, \Cref{lemma: calcualtions needed for spectral initialization with covest estimator} gives that:
\begin{equation}\label{eq: special version of lemma}
    \mathbb{E}\left[\left(a^{\top} x^*\right)^8\left(v^{\top} a\right)^4\right] 
    \leq 10395\, \|x^*\|^4 (v^{\top}x^*)^4 + 5985\|x^*\|^8 \, \|v\|^4 
    \leq 16380\, \|x^*\|^8 \|v\|^4=O(\|x^*\|^8 \|v\|^4),
\end{equation}
for any fixed vector \( v \in \mathbb{R}^{n \times 1} \). We now present the main proof.

\begin{proof}
It will be useful to bound the tail probability $p = \mathbb{P}(\|X\|\geq \tau)$:
    \begin{align}
   p &=\mathbb{P}(\|X\|\geq \tau)\leq \mathbb{E}[\|X\|^4]/ \tau^4 = \mathbb{E}[(\langle a, x^*\rangle^2+\zeta)^4\|a\|^4]/ \tau^4 \leq\mathbb{E}\left[8n(\langle a, x^*\rangle^8+\zeta^4)\left(\sum_{j=1}^n a^{4}_{j}\right)/\tau^4\right]\nonumber\\
   &\overset{(a)}{\leq}\left(8n^2(16380\|x^*\|^8)+8n K_{4}^4\mathbb{E}\left[\left(\sum_{j=1}^n a^{4}_{j}\right)\right]\right)/ \tau^4=O\left(\frac{n^2(\|x^*\|^8+K_{4}^4)}{\tau^4} \right)\label{eq:boundp},
\end{align}
where $(a)$ follows from \Cref{eq: special version of lemma}. Decompose $\Sigma$ as
\begin{equation}
  \Sigma = \mathbb{E}[XX^{\top}]= \mathbb{E}[XX^{\top}\mathbb{I}_{\|X\|\leq \tau}]+\mathbb{E}[XX^{\top}\mathbb{I}_{\|X\|\geq \tau}]= \mathbb{E}[\widetilde{X}\widetilde{X}^{\top}]+\mathbb{E}[XX^{\top}\mathbb{I}_{\|X\|\geq \tau}]\\
  = \widetilde{\Sigma}_r+\mathbb{E}[XX^{\top}\mathbb{I}_{\|X\|\geq \tau}].
\end{equation}
 Upper bounding
\begin{align*}
 \|\mathbb{E}[XX^{\top}\mathbb{I}_{\|X\|\geq \tau}]\|_{\mathrm{op}}&=\sup_{\|v\|=1}\mathbb{E}[\langle v,X\rangle^{2}\mathbb{I}_{\|X\|\geq \tau}]\overset{(a)}{\leq}\sup_{\|v\|=1}\sqrt{p}\sqrt{\mathbb{E}[\langle v,X\rangle^{4}]}\\
 &\leq\sqrt{p} \sup_{\|v\|=1}\sqrt{\mathbb{E}[8(\zeta^4+\langle a,x^*\rangle^{8})\langle a,v\rangle^{4}]}\overset{(b)}{\leq}\sqrt{p}\sqrt{8(3 K_{4}^4+16380 \|x^*\|^8)}\\
 &=O\left( \sqrt{p}\sqrt{( K_{4}^4+ \|x^*\|^8)}\right)= O\left(\sqrt{p}(K_{4}^2+ \|x^*\|^4)\right), 
    \end{align*}
where $(a)$ follows from Cauchy-Schwarz inequality and $(b)$ follows from \Cref{eq: special version of lemma} and using that $\sqrt{p}=O( n (K_{4}^2+ \|x^*\|^4)/\tau^2)$ (by \Cref{eq:boundp}), we have 
\[\|\Sigma-\widetilde{\Sigma}_r\|_{\mathrm{op}} = \mathbb{E}[XX^{\top}\mathbb{I}_{\|X\|\geq \tau}]\|_{\mathrm{op}} =  O( n(K_{4}^2+ \|x^*\|^4)^2/\tau^2).\]
\end{proof}

\subsection{Proof of the \texorpdfstring{\Cref{theorem:Robust PCA Spectral initialisation}}.}
\label{proof: proof of the theorem Robust PCA Spectral initialisation}
Before presenting the proof of \Cref{theorem:Robust PCA Spectral initialisation}, we introduce two helper theorems. These theorems provide the guarantees for the first two steps of \Cref{algorithm: Spectral Initialisation with Robust PCA for additive zero mean noise 1}. We first analyze the step 1 of \Cref{algorithm: Spectral Initialisation with Robust PCA for additive zero mean noise 1}
\begin{theorem}
\label{theorem: Guarantees of estimating norm of actual signal for additive mean-zero noise}
\textbf{(Step 1 of \Cref{algorithm: Spectral Initialisation with Robust PCA for additive zero mean noise 1}: Estimate $\|x^*\|$ using a robust mean estimator).} (See also the part "scaling" from the proof of Theorem 3.1 in \cite{bunaR24_robust_phase}.) If $m_{1}=\Omega\left(\log(2/\delta)\left(1+r_{up}^2\right)\right)$ and $\epsilon =O\left(\left(1+r_{up}^2\right)^{-1}\right)$, then with probability at least $1-\delta/2$ we can conclude that 
\begin{equation}
\label{equation: initial scale}
    |\hat{\|x^*\|}-\left\|x^*\right\| | \leq \frac{1}{27}\left\|x^*\right\|.
\end{equation}
  \end{theorem}

\begin{proof}
      Note that  $\mathbb{E}[y]=\mathbb{E}\left[\left(a^{\top} x^*\right)^2+\zeta\right]=$ $\left\|x^*\right\|^2$ and 
$\operatorname{Var}(y)=\mathbb{E}\left[\left(a^{\top} x^*\right)^4\right]+\mathbb{E}\left[(\zeta)^2\mid a\right]-\left\|x^*\right\|^4=2\left\|x^*\right\|^4+\sigma^2$. Then, using \Cref{thm:robust_mean} we can say that, with probability at least $1-\delta/2$ :

$$
\begin{aligned}
\left|\hat{\|x^*\|}^2-\left\|x^*\right\|^2\right| & =O\left(\sqrt{\operatorname{Var}(y)}\left(\sqrt{\epsilon}+\sqrt{\frac{\log (2 / \delta)}{m_1}}\right)\right) = O\left(\sqrt{\left\|x^*\right\|^4+\sigma^2}\left(\sqrt{\epsilon}+\sqrt{\frac{\log (2 / \delta)}{m_1}}\right)\right)
\end{aligned}
$$
By using Cauchy-Schwarz inequality, we can say that $\sigma^2 \leq K_{4}^2$, which with
 the choices of $m_{1}$ and $\epsilon$ concludes that $\sqrt{\left(1+\frac{\sigma^2}{\|x^*\|^4}\right)}\left(\sqrt{\epsilon}+\sqrt{\frac{\log (2 / \delta)}{m_1}}\right)\leq \sqrt{\left(1+r_{up}^2\right)}\left(\sqrt{\epsilon}+\sqrt{\frac{\log (2 / \delta)}{m_1}}\right) \leq 1$. Now, if $\widetilde{y}^2>0$, then

$$
\left|\hat{\|x^*\|}^2-\left\|x^*\right\|^2\right|=O\left(\left\|x^*\right\|^2 \sqrt{\left(1+\frac{\sigma^2}{\|x^*\|^4}\right)}\left(\sqrt{\epsilon}+\sqrt{\frac{\log (2 / \delta)}{m_1}}\right)\right).
$$

Now, we know that for any $a, b, c>0$ with $b^2>c,\left|a^2-b^2\right| \leq c$ implies $|a-b| \leq b-\sqrt{b^2-c}$. By using the above property, we get to

$$
|\hat{\|x^*\|}-\left\|x^*\right\| |=O\left(\left\|x^*\right\|\left(1-\sqrt{1-\sqrt{\left(1+\frac{\sigma^2}{\|x^*\|^4}\right)}\left(\sqrt{\epsilon}+\sqrt{\frac{\log (2 / \delta)}{m_1}}\right)}\right)\right).
$$
Now, using the bound \(\left(1 + \frac{\sigma^2}{\|x^*\|^4}\right) \leq \left(1 + r_{up}^2\right)\), we can conclude that
\[
\left| \hat{\|x^*\|} - \|x^*\| \right| = O\left( \|x^*\| \left(1 - \sqrt{1 - \sqrt{1 + r_{up}^2} \left( \sqrt{\epsilon} + \sqrt{\frac{\log(2/\delta)}{m_1}} \right)} \right) \right).
\]

Now, we can choose the hidden in such a way that
\begin{equation}
    |\hat{\|x^*\|}-\left\|x^*\right\| | \leq \frac{1}{27}\left\|x^*\right\|.
\end{equation}
  \end{proof}

Next, we analyze the step 2 of \Cref{algorithm: Spectral Initialisation with Robust PCA for additive zero mean noise 1}. 
\begin{theorem}
\label{theorem: estimating robust pca for addtive noise with zero mean}
 (\textbf{Step 2 of \Cref{algorithm: Spectral Initialisation with Robust PCA for additive zero mean noise 1}: Estimating $\hat{u}$.}) Let's assume that $|\hat{\|x^*\|}-\left\|x^*\right\| | \leq \frac{\left\|x^*\right\|}{27}$ and \( \delta \in (0, 0.5) \) is a positive small constant. If $m_{2}=\Omega\left( (n+ n(r_{up}^2+1)\sqrt{\epsilon})\log(2n/(\delta \epsilon))/\epsilon\right)$ and  \( \epsilon \in (0, 1/36] \).  Then, under the above notation, with probability at least \( 1 - \delta/2 \), the following holds:
  \[\operatorname{dist}(\hat{u},u)=\min\{\|\hat{u}-u\|_{2},\|\hat{u}+u\|_{2}\} \leq \sqrt{2-2\sqrt{1-O\left( (r_{up}^2+1) \sqrt{\epsilon}\right)}},\]
  where $u$ is the top normalized eigenvector of $\widetilde{\Sigma}_{r}$.
\end{theorem}

  \begin{proof}
    We use \Cref{theorem: Guarantees of Algorithm 2 from kong robust 2020 paper 1} to calculate the distance between $\hat{u}$ and the top eigenvector $u$ of $\widetilde{\Sigma}_{r}$. To apply \Cref{theorem: Guarantees of Algorithm 2 from kong robust 2020 paper 1}, we must verify that all its underlying assumptions are satisfied. We now proceed to check these assumptions one by one. We know that $\|\widetilde{X}\|\leq \hat{\tau}$. \Cref{lem:trunc} implies
    \[
\|\Sigma - \widetilde{\Sigma}_r\|_{\mathrm{op}} = O(\, n (K_4^2 + \|x^*\|^4)^2/\hat{\tau}^2) \leq \|x^*\|^4 \quad \text{and} \quad -\|x^*\|^4 I \preceq \widetilde{\Sigma}_r - \Sigma \preceq \|x^*\|^4 I.
\]

The first assumption is to show that $\|\widetilde{X}\widetilde{X}^\top-\widetilde{\Sigma}_r\|_{\mathrm{op}}\leq B$ for all $\widetilde{X}$ with probability one.
    \begin{align*}
      \|\widetilde{X}\widetilde{X}^\top-(\widetilde{\Sigma}_r)\|_{\mathrm{op}}&\leq \|\widetilde{X}\|^2+\lambda_{max}(\widetilde{\Sigma}_r)\leq\hat{\tau}^2+\lambda_{max}(\Sigma)+\|x^*\|^4=O( n (r_{up}^2+1)^2 \hat{x}^{*4})+\left(16\|x^*\|^4+\sigma^2\right)\\
      &= O\left( n (r_{up}^2+1)^2 +16\right)\|x^*\|^4+\sigma^2.   
    \end{align*}
The second assumption is to show that $\mathbb{E}\left[ \langle xx^{\top} ,\widetilde{X}\widetilde{X}^\top-\widetilde{\Sigma}_r\rangle^2  \right]\leq v^{'2}.$
    
    Consider $\|x\|\leq 1$, then
     \begin{align*}
     \mathbb{E}\left[ \langle xx^{\top} ,\widetilde{X}\widetilde{X}^\top-\widetilde{\Sigma}_r\rangle^2  \right]&= \mathbb{E}\left[( \langle x ,\widetilde{X}\rangle^2-x^{\top}\widetilde{\Sigma}_{r}x)^2 \right]=\mathbb{E}\left[( \langle x ,\widetilde{X}\rangle)^4\right]-(x^{\top}\widetilde{\Sigma}_{r}x)^4
     \overset{(a)}{\leq} \mathbb{E}\left[( \langle x ,X\rangle)^4\right]\\
     &\leq \mathbb{E}[8(\zeta^4+\langle a,x^*\rangle^{8})\langle a,x\rangle^{4}]\overset{(b)}{\leq}8(3 K_{4}^4+16380 \|x^*\|^8)=O((r_{up}^2+1)^2  \|x^*\|^8),
    \end{align*}
where \((a)\) follows by applying the trivial upper bound \(\hat{\tau} \leq \infty\) and $(b)$ follows from \Cref{eq: special version of lemma}.

This implies \[\max_{x:\|x\|\leq 1} \mathbb{E}\left[ \langle xx^{\top} ,\widetilde{X}\widetilde{X}^\top-\widetilde{\Sigma}\rangle^2  \right]=O((r_{up}^2+1)^2  \|x^*\|^8).\] 

In the notations of \Cref{theorem: Guarantees of Algorithm 2 from kong robust 2020 paper 1}, we have  
\[
B = O\left(\, n\, (r_{\mathrm{up}}^2+1)^2 + 16\right)\|x^*\|^4 + \sigma^2, \quad v' = O(\, (r_{\mathrm{up}}^2+1)\, \|x^*\|^4).
\]  
To apply \Cref{theorem: Guarantees of Algorithm 2 from kong robust 2020 paper 1}, we must ensure that its sample complexity condition is satisfied. This requires
\[
m_2 = \Omega\left( \frac{n + \frac{B\sqrt{\epsilon}}{v'}}{\epsilon} \cdot \log\left( \frac{2n}{\delta \epsilon} \right) \right) = \Omega\left( \frac{n +\, n\, (r_{\mathrm{up}}^2+1)\sqrt{\epsilon}}{\epsilon} \cdot \log\left( \frac{2n}{\delta \epsilon} \right) \right).
\]

So, we have verified all assumptions of \Cref{theorem: Guarantees of Algorithm 2 from kong robust 2020 paper 1}. So, by using \Cref{theorem: Guarantees of Algorithm 2 from kong robust 2020 paper 1}, we can conclude that, with probability at least \( 1 - \delta/2 \), the output \( \hat{u} \in \mathbb{R}^{n \times 1} \) of Algorithm 2 satisfies:
\[
\operatorname{Tr}[P_1(\widetilde{\Sigma}_r)] - \operatorname{Tr}[\hat{u}^\top (\widetilde{\Sigma}_r) \hat{u}] = O\left( \epsilon \cdot \operatorname{Tr}[P_1(\widetilde{\Sigma}_r)] + \sqrt{\epsilon} \cdot v'  \right).
\]

We know that the top eigenvalue of \(\Sigma\) is \(15\|x^*\|^4 + \sigma^2\), while all remaining eigenvalues equal \(3\|x^*\|^4 + \sigma^2\). Let \(\widetilde{\Sigma}_r = \sum_{i=1}^n \lambda_i f_i f_i^\top\) be the eigenvalue decomposition of \(\widetilde{\Sigma}_r\), with \(\lambda_1 = \lambda_{\max}(\widetilde{\Sigma}_r)\) and \(f_1 = u\) such that eigenvalues of \(\widetilde{\Sigma}_r\) are arranged in non-increasing order, where $\lambda_{\max}(A)$ denotes the largest eigenvalue of $A$. By the min–max theorem, we have \(\lambda_1 \in [14\|x^*\|^4 + \sigma^2,\ 16\|x^*\|^4 + \sigma^2]\) and \(\lambda_2 \in [2\|x^*\|^4 + \sigma^2,\ 4\|x^*\|^4 + \sigma^2]\). Since \(4\|x^*\|^4 + \sigma^2 \neq 14\|x^*\|^4 + \sigma^2\), it follows that \(\lambda_1 \neq \lambda_2\).
\begin{align*}
  \operatorname{Tr}[P_1(\widetilde{\Sigma}_r)] - \operatorname{Tr}[\hat{u}^\top (\widetilde{\Sigma}_r) \hat{u}]&\geq\lambda_1-(\lambda_1\langle \hat{u}, u \rangle^2+(1-\langle \hat{u}, u \rangle^2)\lambda_2)= (\lambda_1-\lambda_2)(1-\langle \hat{u}, u \rangle^2)\\
  &\geq 10 \|x^*\|^4 (1-\langle \hat{u}, u \rangle^2).
\end{align*}
So now, 
 \[(1-\langle \hat{u}, u \rangle^2)=  O\left( \epsilon \left(1+\frac{\sigma^2}{\|x^*\|^4}\right) + 73 (r_{up}^2+1) \sqrt{\epsilon}  \right)=O\left( \epsilon (r_{up}^2+1) +  (r_{up}^2+1) \sqrt{\epsilon}  \right)=O\left( (r_{up}^2+1) \sqrt{\epsilon}  \right).\]
 This implies \[|\langle \hat{u}, u \rangle|\geq \sqrt{1-O\left( (r_{up}^2+1) \sqrt{\epsilon}\right)}.\]
 Now,
 \begin{align*}
   \operatorname{dist}(\hat{u},u)&=\min\{\|\hat{u}-u\|_{2},\|\hat{u}+u\|_{2}\}\\
   &=\sqrt{2-2|\langle \hat{u}, u \rangle|}\\
   &\leq \sqrt{2-2\sqrt{1-O\left( (r_{up}^2+1) \sqrt{\epsilon}\right)}}.
 \end{align*}
\end{proof}

We are now in a position to prove \Cref{theorem:Robust PCA Spectral initialisation}.

\textbf{Proof of \Cref{theorem:Robust PCA Spectral initialisation}.}

\begin{proof}

Under the assumptions on the sample complexity \( m_1, m_2 \) and the corruption level \( \epsilon \), and using \Cref{theorem: Guarantees of estimating norm of actual signal for additive mean-zero noise}, \Cref{theorem: estimating robust pca for addtive noise with zero mean}, and \Cref{lem:trunc}, we conclude that, with probability at least \(1 - \delta\),
 the following holds:
\[\operatorname{dist}(\hat{u},u)=\min\{\|\hat{u}-u\|_{2},\|\hat{u}+u\|_{2}\} \leq \sqrt{2-2\sqrt{1-O\left( (r_{up}^2+1) \sqrt{\epsilon}\right)}}\leq 1/27 \],
\[\|\Sigma-\widetilde{\Sigma}_r\|_{\mathrm{op}}=O(n(K_{4}^2+ \|x^*\|^4)^2/\hat{\tau}^2)\leq \frac{\sqrt{2}\|x^*\|^4}{9} ,\]
\[|\hat{\|x^*\|}-\|x^*\||\leq \|x^*\|/27,\]
where $u$ is top normalized eigenvector of $\widetilde{\Sigma}_r$. Now, by applying \Cref{throrem:conclusion of davis-kahan} (Davis-Kahan), we can conclude that
\[dist(\|x^*\|u,x^*)\leq \frac{2\sqrt{2}\|x^*\|\|\Sigma-(\widetilde{\Sigma}_r)\|_{\mathrm{op}}}{12\|x^*\|^4}\leq \frac{\|x^*\|}{27},\quad \text{and}\]
\[dist(-\|x^*\|u,x^*)\leq \frac{2\sqrt{2}\|x^*\|\|\Sigma-(\widetilde{\Sigma}_r)\|_{\mathrm{op}}}{12\|x^*\|^4}\leq \frac{\|x^*\|}{27}.\]
 So,
\begin{align*}
\operatorname{dist}\left(x_0, x^*\right) &= \operatorname{dist}\left(\hat{\|x^*\|}\hat{u}, x^*\right)\leq \operatorname{dist}(\hat{\|x^*\|}\hat{u},\|x^*\|\hat{u})+\operatorname{dist}(\|x^*\|\hat{u},\|x^*\|u)+dist\left(\|x^*\|u, x^*\right)\\
&= \|\hat{\|x^*\|}\hat{u}-\|x^*\|\hat{u}\|+\operatorname{dist}(\|x^*\|\hat{u},\|x^*\|u)+dist\left(\|x^*\|u, x^*\right)\\
&= |\hat{\|x^*\|}-\|x^*\||+\|x^*\|\operatorname{dist}(\hat{u},u)+ dist\left(\|x^*\|u, x^*\right)\\
 &\leq \frac{\|x^*\|}{27}+\frac{\|x^*\|}{27}+\frac{\|x^*\|}{27}=\frac{\|x^*\|}{9} .
 \end{align*}
\end{proof}

\section{Near Linear Time Algorithms For Robust Phase Retrieval under Zero-Mean Noise}
\subsection{ Guarantees of Algorithm 2 (SampleTopEigenvector) from \texorpdfstring{\cite{robust_PCA_Algorithm_in_nearly_linear_time_by_Diakonikolas}} .}
\label{proof: fromal version of Algorithm 2 from diakonikolas paper.}

\begin{theorem}
\label{theorem: Gurantees of linear time robust pca algorithm}
     (See Theorem 3.1 from \cite{robust_PCA_Algorithm_in_nearly_linear_time_by_Diakonikolas}) Let $\gamma_0$ be a sufficiently small positive constant. Let $n>2$ be an integer, and $\epsilon, \gamma \in(0,1)$ such that $0<20 \epsilon<\gamma<\gamma_0$. Let $D$ be the uniform distribution over a set of $m'$ points in $\mathbb{R}^n$, that can be decomposed as $D=(1-\epsilon) G+\epsilon B$, where $G$ is a $(20 \epsilon, \gamma)$-stable distribution (\Cref{definition: Stability for nearly linear time robust PCA algorithm}) with respect to a PSD matrix $\boldsymbol{\Sigma} \in \mathbb{R}^{n \times n}$. Here, $G$ denotes the set of uncorrupted samples and $B$ denotes the set of corrupted samples. There exists an algorithm (Algorithm 2 from \cite{robust_PCA_Algorithm_in_nearly_linear_time_by_Diakonikolas}) that takes as input $D, \epsilon, \gamma$, runs for $O\left(\frac{m' n}{\gamma^2} \log ^4(n / \epsilon)\right)$ time, and with probability at least 0.99 , outputs a unit vector $\hat{u}$ such that $\hat{u}^{\top} \boldsymbol{\Sigma} \hat{u} \geq(1-O(\gamma))\|\boldsymbol{\Sigma}\|_{\mathrm{op}}$.
 \end{theorem}

 \subsection{Formal version of \texorpdfstring{\Cref{prop:stab_to_pr}}%
{Theorem: Robust PCA Spectral initialisation using stability lemma} and its proof}
\label{proof: formal version of proposition and its proof }
Here, we first present the formal version of \Cref{prop:stab_to_pr} and then we prove it.
 \begin{proposition}
Let's assume $\delta\in [0,0.5]$ be a constant, $m_{1}=\Omega(\log(2/\delta)(1+r_{up}^2))$, $\tilde{m}=\Omega( n\log(n)\log(2/\delta)(1+r_{up}^2))$ and $\epsilon\leq C_2/ (1+r_{up}^2)$. Let $G$ be the empirical distribution over $f(n)$ uncorrupted samples of $\widetilde{X}$. If $G$ is $(\epsilon, \gamma=5/C_{3}27^2(16+r_{up}^2))$-stable with respect to $\widetilde{\Sigma}_{r}$, then there exists an algorithm for robust phase retrieval with the sample complexity $O(m_{1}+f(n)+P\tilde{m})$ and time complexity $\widetilde{O}(T_{\mathrm{rob\text{-}mean}}(m_{1},1)+f(n)n+PT_{\mathrm{rob\text{-}mean}}(\tilde{m},n))$ such that with probability at least $0.99-(P+1)\delta$, the output $x_{out}$ of the algorithm satisfies $dist(x_{out},x^*)\leq \sigma\sqrt{\epsilon}/\|x^*\|$, where $P=O(1)$ is constant.   
\end{proposition}
\begin{proof}
The algorithm considered is $A_{\text{alt}} + \text{Algorithm 2 from } \cite{bunaR24_robust_phase}$, where $A_{\text{alt}}$ is the same as \cref{algorithm: Spectral Initialisation with Robust PCA for additive zero mean noise 1}, except that in Step~2, Algorithm~2 of \cite{kong2020robust} is replaced with the \texttt{SampleTopEigenvector} (Algorithm~2) from \cite{robust_PCA_Algorithm_in_nearly_linear_time_by_Diakonikolas}. From the \Cref{thm:grad} it is evident that if the output $x_0$ of $A_{alt}$ lies in $\mathcal{B}(\pm x^*, \|x^*\|/9)$, then the final output $x_{out}$ of above algorithm satisfies $dist(x_{out},x^*)\leq \sigma\sqrt{\epsilon}/\|x^*\|$ (\emph{the algorithm successfully solves the phase retrieval problem}). So, we need to show that $x_0\in \mathcal{B}(\pm x^*, \|x^*\|/9)$.
Now, as the empirical distribution over $f(n)$ uncorrupted samples satisfies the stability lemma, using  \Cref{theorem: Gurantees of linear time robust pca algorithm}, \Cref{theorem: Guarantees of estimating norm of actual signal for additive mean-zero noise}, and \Cref{lem:trunc}, we conclude that, with probability at least \(0.99 - \delta\),
 the following holds:
\[\|\widetilde{\Sigma}_{r}\|_{\mathrm{op}}-\hat{u}^{\top} \widetilde{\Sigma}_{r}\hat{u} \leq O(\gamma)\|\widetilde{\Sigma}_{r}\|_{\mathrm{op}}, \quad |\hat{\|x^*\|}-\|x^*\||\leq \|x^*\|/27,\]
\[ \|\Sigma-\widetilde{\Sigma}_r\|_{\mathrm{op}}=O(n(K_{4}^2+ \|x^*\|^4)^2/\hat{\tau}^2)\leq \frac{\sqrt{2}\|x^*\|^4}{9},\]
where $u$ is top normalized eigenvector of $\widetilde{\Sigma}_r$ and $\hat{u}$ is the output of \texttt{SampleTopEigenvector} (Algorithm~2) from \cite{robust_PCA_Algorithm_in_nearly_linear_time_by_Diakonikolas} upon providing $f(n)/(1-\epsilon)$  ($f(n)$ uncorrupted samples and $\epsilon f(n)/1 - \epsilon$ corrupted samples) samples.  Now, by applying \Cref{throrem:conclusion of davis-kahan} (Davis-Kahan), we can conclude that
\[\operatorname{dist}(\|x^*\|u,x^*)\leq \frac{2\sqrt{2}\|x^*\|\|\Sigma-(\widetilde{\Sigma}_r)\|_{\mathrm{op}}}{12\|x^*\|^4}\leq \frac{\|x^*\|}{27} \quad\text{and}\]\[
 \operatorname{dist}(-\|x^*\|u,x^*)\leq \frac{2\sqrt{2}\|x^*\|\|\Sigma-(\widetilde{\Sigma}_r)\|_{\mathrm{op}}}{12\|x^*\|^4}\leq \frac{\|x^*\|}{27}.\]
Note that $\|\Sigma-\widetilde{\Sigma}_r\|_{\mathrm{op}}\leq \|x^*\|^4$ implies $-\|x^*\|^4 \cdot I \preceq (\widetilde{\Sigma}_r) - \Sigma \preceq \|x^*\|^4 \cdot I$. Now, we know that the top eigenvalue of $\Sigma$ is $15\|x^*\|^2+\sigma^2$, and the rest of the eigenvalues are the same, with a value of $3\|x^*\|^4+\sigma^2$. Assume that \( \widetilde{\Sigma}_r \) has the eigenvalue decomposition \( \widetilde{\Sigma}_r = \sum_{i=1}^n \lambda_i f_i f_i^\top \), with \( \lambda_1 = \lambda_{\max}(\widetilde{\Sigma}_r), f_1=u \) such that eigenvalues of \(\widetilde{\Sigma}_r\) are arranged in non-increasing order, where $\lambda_{\max}(A)$ denotes the largest eigenvalue of $A$. Min-max theorem implies that $\lambda_1\in [14\|x^*\|^4+\sigma^2, 16\|x^*\|^4+\sigma^2)]$ and $\lambda_2\in [2\|x^*\|^4+\sigma^2,4\|x^*\|^4+\sigma^2]$. Now, $4\|x^*\|^4+\sigma^2\neq 14\|x^*\|^4+\sigma^2$ implies that $\lambda_1\neq \lambda_2$.
\begin{align*}
  \|\widetilde{\Sigma}_r\|_{\mathrm{op}} - (\hat{u}^\top \widetilde{\Sigma}_r \hat{u})&\geq\lambda_1-(\lambda_1\langle \hat{u}, u \rangle^2+(1-\langle \hat{u}, u \rangle^2)\lambda_2)= (\lambda_1-\lambda_2)(1-\langle \hat{u}, u \rangle^2)\\
  &\geq 10 \|x^*\|^4 (1-\langle \hat{u}, u \rangle^2).
\end{align*}
So, this implies \[10 \|x^*\|^4 (1-\langle \hat{u}, u \rangle^2) \leq O(\gamma)(16\|x^*\|^4+\sigma^2)\leq O(\gamma)(16+r_{up}^2) \|x^*\|^4\leq 5\|x^*\|^4/27^2,\]
 which further concludes that
\[|\langle \hat{u}, u \rangle|\geq \sqrt{1-\frac{1}{2 \times 27^2}}\geq \left(1-\frac{1}{2 \times 27^2}\right).\]
 Now,
 \begin{align*}
   \operatorname{dist}(\hat{u},u)&=\min\{\|\hat{u}-u\|_{2},\|\hat{u}+u\|_{2}\}\\
   &=\sqrt{2-2|\langle \hat{u}, u \rangle|}=\sqrt{2}\sqrt{1-|\langle \hat{u}, u \rangle|}\leq \sqrt{2}\sqrt{1-\left(1-\frac{1}{2 \times 27^2}\right)}=\frac{1}{27}.
 \end{align*}
So, from above we can conclude that $
\operatorname{dist}(\hat{u},u) \leq 1/27.
$ So the output $x_{0}$ of $A_{alt}$  satisfies,
\begin{align*}
\operatorname{dist}\left(x_0, x^*\right) &= \operatorname{dist}\left(\hat{\|x^*\|}\hat{u}, x^*\right)\leq \operatorname{dist}(\hat{\|x^*\|}\hat{u},\|x^*\|\hat{u})+\operatorname{dist}(\|x^*\|\hat{u},\|x^*\|u)+dist\left(\|x^*\|u, x^*\right)\\
&= \|\hat{\|x^*\|}\hat{u}-\|x^*\|\hat{u}\|+\operatorname{dist}(\|x^*\|\hat{u},\|x^*\|u)+dist\left(\|x^*\|u, x^*\right)\\
&= |\hat{\|x^*\|}-\|x^*\||+\|x^*\|\operatorname{dist}(\hat{u},u)+ dist\left(\|x^*\|u, x^*\right)\\
 &\leq \frac{\|x^*\|}{27}+\frac{\|x^*\|}{27}+\frac{\|x^*\|}{27}=\frac{\|x^*\|}{9} .
 \end{align*}

 The running time of the above algorithm is $\widetilde{O}(T_{\mathrm{rob\text{-}mean}}(m_{1},1)+f(n)n+PT_{\mathrm{rob\text{-}mean}}(\tilde{m},n))$ and it follows from \Cref{theorem: Gurantees of linear time robust pca algorithm}.
\end{proof}

\subsection{ \texorpdfstring{\Cref{lemma:stability lemma for uncorrupted truncated samples}}%
{Theorem: Robust PCA Spectral initialisation using stability lemma} and its proof}
\label{proof: proof of the stability lemma}

\begin{lemma}
\label{lemma:stability lemma for uncorrupted truncated samples}
(Stability lemma). Assume that $|\hat{x}^*-\|x^*\||\leq \frac{1}{27}\|x^*\|$.  Let $\delta \in[0,0.5]$, $p \geq 2$. Under the assumptions and notations from above if  $m_{2}\geq C_{1}\max\{n\log(4n/\delta)((16+r_{up}^2)^4),n^5 (r_{up}^2+1)^4\log(4/\delta)\}$, and $\epsilon \leq C_{2}(16+r_{up}^2)^{-4}$ where $C_{1},C_{2}$ are universal constants.  Then the following holds:
\begin{equation}
\label{eq:stability condition needed for alternative approach of robust pca}
    \operatorname{Pr}\Bigg[ \exists w \in \mathfrak{S}_\epsilon^{m_2} \,\Bigg|\,
    \left\|\sum_{i} w_i \widetilde{X}_i \widetilde{X}_i^{\top}
      - \widetilde{\Sigma}_r \right\|_{\mathrm{op}} 
    > \gamma
      \|\widetilde{\Sigma}_r\|_{\mathrm{op}} \Bigg]
    \leq \frac{\delta}{2}.
\end{equation}

where $\mathfrak{S}_\epsilon^{m_2}=\{w\in \mathbb{R}^{m_{2}}:\sum_{i=1}^{m_2} w_i=1, w_{i}\leq 1/(1-\epsilon)m_2\}$, and $\gamma<1$.
\end{lemma}

Before presenting the the proof of \Cref{lemma:stability lemma for uncorrupted truncated samples}, we first state a concentration lemma for the sample covariance matrix of bounded random vectors, which will be needed to establish the stability lemma.
 \begin{theorem}
 \label{theorem: Matrix Brenstein inequality 1}
 \textbf{( Matrix Bernstein inequality (See \cite{vershynin2018high}).}  Let $\mathcal{D}$ be a distribution on $\mathbb{R}^n$ supported within an $\ell_2$-ball of radius $\tau$ centered at the origin. Let $\widetilde{\Sigma}_r$ denotes the second moment matrix of $\mathcal{D}$, and let $\widetilde{\Sigma}_{m'} = \frac{1}{m'} \sum_{i=1}^{m'} \widetilde{X}_i \widetilde{X}_i^\top$ be the empirical second moment matrix based on $n$ i.i.d. samples $\widetilde{X}_1, \dots, \widetilde{X}_{m'} \sim \mathcal{D}$. Then there exists a universal constant $C > 0$ such that for any $0 < t < 1$ and $0 < \delta < 1$, if
 \[m' > C \cdot \frac{\tau^2}{t^2 \|\widetilde{\Sigma}_r\|_{\mathrm{op}}} \log\left(\frac{n}{\delta}\right),\]
    then with probability at least $1 - \delta$, we have $\quad \| \widetilde{\Sigma}_{m'}-\widetilde{\Sigma}_r\|_{\mathrm{op}} \leq t \|\widetilde{\Sigma}_r\|_{\mathrm{op}}.$
\end{theorem}

Now we prove the \Cref{lemma:stability lemma for uncorrupted truncated samples}. 
\begin{proof}
We prove the \Cref{lemma:stability lemma for uncorrupted truncated samples} for $\gamma=\frac{5}{27^2 (16+r_{up}^2)}$. Consider $G$ to be a uniform distribution over a set of $m_2$ uncorrupted samples $\{\widetilde{X}_{1},\widetilde{X}_{2},\ldots,\widetilde{X}_{m_{2}}\}$. Then, for any weight function $w$, $\boldsymbol{\Sigma}_{G_w}= \sum_{i} w_i \widetilde{X}_i \widetilde{X}_i^{\top} $  for some $w \in \mathfrak{S}_\epsilon^{m_{2}}.$ Let's denote the left-hand side of \Cref{eq:stability condition needed for alternative approach of robust pca} by $\widetilde{p}$, then

\begin{align*}
\widetilde{p}&\leq  \operatorname{Pr}\left[\max_{ w \in \mathfrak{S}_\epsilon^{m_{2}}}\left\|\sum_{i} w_i \widetilde{X}_i \widetilde{X}_i^{\top}-\widetilde{\Sigma}_r\right\|_{\mathrm{op}}>5/(27^2 (16+r_{up}^2))\|\widetilde{\Sigma}_r\|_{\mathrm{op}}\right].
\end{align*}
 Now for any $w \in \mathfrak{S}_\epsilon^{m_{2}}$, 
 \begin{align*}
    \left\|\sum_{i} w_i \widetilde{X}_i \widetilde{X}_i^{\top}-\widetilde{\Sigma}_r\right\|_{\mathrm{op}}&=\left\|\sum_{i} w_i \widetilde{X}_i \widetilde{X}_i^{\top}-\frac{1}{m_{2}}\sum_{i=1}^{m_{2}} \widetilde{X}_i \widetilde{X}_i^{\top}+\frac{1}{m_{2}}\sum_{i=1}^{m_{2}} \widetilde{X}_i \widetilde{X}_i^{\top}-\widetilde{\Sigma}_r\right\|_{\mathrm{op}}\\
    &\leq \left\|\sum_{i} w_i \widetilde{X}_i \widetilde{X}_i^{\top}-\frac{1}{m_{2}}\sum_{i=1}^{m_{2}} \widetilde{X}_i \widetilde{X}_i^{\top}\right\|_{\mathrm{op}}+\left\|\frac{1}{m_{2}}\sum_{i=1}^{m_{2}} \widetilde{X}_i \widetilde{X}_i^{\top}-\widetilde{\Sigma}_r\right\|_{\mathrm{op}}.\\
 \end{align*}
Now,
\begin{align*}
   \max_{ w \in \mathfrak{S}_\epsilon^{m_{2}}}\left\|\sum_{i} w_i \widetilde{X}_i \widetilde{X}_i^{\top}-\frac{1}{m_{2}}\sum_{i=1}^{m_{2}} \widetilde{X}_i \widetilde{X}_i^{\top}\right\|_{\mathrm{op}}&= \max_{ w \in \mathfrak{S}_\epsilon^{m_{2}}}\max_{v:\|v\|=1}\left|\sum_{i} \left(w_i-\frac{1}{m_{2}}\right)  (\widetilde{X}_i^{\top}v)^2\right|\\
   &\leq \max_{ w \in \mathfrak{S}_\epsilon^{m_{2}}}\max_{v:\|v\|=1}\sqrt{\sum_{i}\left(w_i-\frac{1}{m_{2}}\right)^2}\sqrt{\sum_{i}\left(\widetilde{X}_i^{\top}v\right)^4}\\
   & \leq \sqrt{\frac{\epsilon}{(1-\epsilon)}}\max_{v:\|v\|=1}\sqrt{\frac{1}{m_{2}}\sum_{i}\left(\widetilde{X}_i^{\top}v\right)^4},
\end{align*}
because $\sum_{i}\left(w_i-1/m_{2}\right)^2\leq \sum_{i}w_{i}^2-1/m_{2}\leq \epsilon/((1-\epsilon)m_{2}))$. Now, note that for any $v$, $\left(\widetilde{X}_i^{\top}v\right)^4 =O( n^2(r_{up}^2+1)^4\hat{x^*}^8)= O( n^2(r_{up}^2+1)^4\|x^*\|^8)$ is a bounded random variable.  Then, by using Hoeffding's concentration inequality (see \cite{vershynin2018high}), we can say that
\begin{align*}
    \operatorname{Pr}\left[\left|\frac{1}{m_{2}}\sum_{i}\left(\widetilde{X}_i^{\top}v\right)^4-\mathbb{E}\left[\left(\widetilde{X}^{\top}v\right)^4\right]\right|\geq t\right]\leq 2 \exp{\frac{-2m_{2}t^2}{C n^4 (r_{up}^2+1)^8 \|x^*\|^{16}}}.
\end{align*}
Now $\mathbb{E}\left[\left(\widetilde{X}^{\top}v\right)^4\right]\leq \mathbb{E}\left[8(\left(a^{\top} x^*\right)^8+\zeta^4)\left(a^{\top}v\right)^4\right] \overset{(a)}{\leq} 131040( K_{4}^4+ \|x^*\|^8)\leq 131040(r_{up}^2+1)^2 \|x^*\|^{8}$, where $(a)$ follows from \Cref{eq: special version of lemma}. Now, if we choose $t= 131040 (r_{up}^2+1)^2 \|x^*\|^{8},$ by a standard epsilon-net argument, we can say that  
\begin{align*}
    \operatorname{Pr}\left[\max_{v:\|v\|=1}\sqrt{\frac{1}{m_{2}}\sum_{i}\left(\widetilde{X}_i^{\top}v\right)^4}\geq 724 (r_{up}^2+1)\|x^*\|^{4}\right]\leq 2\cdot 12 ^{n}\exp{\frac{-2m_{2}}{C n^4 (r_{up}^2+1)^4 }},
\end{align*}
where $12^n$ is the size of the net over $n$ dimensions unit sphere. If $m_{2}=\Omega(n^5 (r_{up}^2+1)^4\log(4/\delta))$ then, with probability at least $1-\delta/4$, we can say that 
\begin{align*}
     \max_{ w \in \mathfrak{S}_\epsilon^{m_{2}}}\left\|\sum_{i} w_i \widetilde{X}_i \widetilde{X}_i^{\top}-\frac{1}{m_{2}}\sum_{i=1}^{m_{2}} \widetilde{X}_i \widetilde{X}_i^{\top}\right\|_{\mathrm{op}}&\leq \sqrt{\epsilon}\sqrt{2}\cdot 724 (r_{up}^2+1)\|x^*\|^{4}. 
\end{align*}

\Cref{lem:trunc} implies $14\|x^*\|^4+\sigma^2\leq \|\widetilde{\Sigma}_r\|_{\mathrm{op}}\leq 16\|x^*\|^4+\sigma^2$. Let's define $\alpha:= (r_{up}^2+1)\|x^*\|^{4}/(14\|x^*\|^4+\sigma^2)$. Then, $1/14\leq \alpha \leq (r_{up}^2+1)/14.$   So, using this, we can say that with probability at least $1-\delta/4$, the following holds:
\begin{align*}
     \max_{ w \in \mathfrak{S}_\epsilon^{m_{2}}}\left\|\sum_{i} w_i \widetilde{X}_i \widetilde{X}_i^{\top}-\frac{1}{m_{2}}\sum_{i=1}^{m_{2}} \widetilde{X}_i \widetilde{X}_i^{\top}\right\|_{\mathrm{op}}&\leq \sqrt{\epsilon}\sqrt{2}\cdot 724 \alpha\|\widetilde{\Sigma}_r\|_{\mathrm{op}}\leq 74 \sqrt{\epsilon}(r_{up}^2+1)\|\widetilde{\Sigma}_r\|_{\mathrm{op}}. 
\end{align*}
Now, for the above choice of $\epsilon$, $74 \sqrt{\epsilon}(r_{up}^2+1)\leq 5/(2\cdot(27^2 (16+r_{up}^2)))<1$. Now, by using \Cref{theorem: Matrix Brenstein inequality 1}, we can conclude that with probability at least $1-\delta/4$ the following holds:  \[\left\|\frac{1}{m_{2}}\sum_{i=1}^{m_{2}} \widetilde{X}_i \widetilde{X}_i^{\top}-\widetilde{\Sigma}_r\right\|_{\mathrm{op}}\leq \frac{5}{2\cdot(27^2 (16+r_{up}^2))}\|\widetilde{\Sigma}_r\|_{\mathrm{op}},\]
given that $m_{2}\geq \Omega(n\log(4n/\delta)((16+r_{up}^2)^4)$. So, this implies that with probability at least $1-\delta/2$, we can conclude that 
\begin{align*}
   \max_{ w \in \mathfrak{S}_\epsilon^{m_{2}}}\left\|\sum_{i} w_i \widetilde{X}_i \widetilde{X}_i^{\top}-\widetilde{\Sigma}_r\right\|_{\mathrm{op}}& \leq \frac{5 }{27^2 (16+r_{up}^2)}\|\widetilde{\Sigma}_r\|_{\mathrm{op}} .\\
\end{align*} 
\end{proof}

\subsection{\texorpdfstring{\Cref{theorem:Robust PCA Spectral initialisation using stability lemma proof}}%
{Theorem: Robust PCA Spectral initialisation using stability lemma} and its proof}
\label{proof: proof of spectral initialization using linear time robust pca algorithm}

\begin{theorem}
\label{theorem:Robust PCA Spectral initialisation using stability lemma proof}
    ( Spectral initialization with stability-based Robust PCA) Let $\delta \in[0,0.5]$. Let's assume that an upper bound on signal to noise ratio is known $K_{4}/\|x^*\|^2\leq r_{up}$. Under the setting and notation of Algorithm $A_{alt}$,  if $m_1\geq C_{1}\log(2/\delta)(1+r_{up}^2)$, $m_2\geq C_{2}n^5 (16+r_{up}^2)^4\log(4n/\delta)$,  and $\epsilon \leq C_{3}(16+r_{up}^2)^{-4}$ where $C_{1},C_{2}$ and $C_{3}$ are universal constants. Then with probability at least $0.99-\delta$, the output $x_0$ of the algorithm $A_{alt}$ satisfies the following: 
    \begin{equation}
    dist\left(x_0, x^*\right) \leq\left\|x^*\right\| / 9 .
\end{equation}
The time complexity of the algorithm is $\widetilde{O}(T_{\mathrm{rob\text{-}mean}}(m_{1},1)+m_2n)$.
\end{theorem}

Before presenting the proof of \Cref{theorem:Robust PCA Spectral initialisation using stability lemma proof}, we first establish a helper theorem that analyzes the second step of the $A_{alt}$ algorithm. In particular, this helper theorem focuses on the eigenvector direction computation step of $A_{alt}$. This step is the only difference from \Cref{algorithm: Spectral Initialisation with Robust PCA for additive zero mean noise 1}; all the remaining steps are identical, as discussed earlier.

\begin{theorem}
\label{theorem: estimating direction using the idea of dong paper}
 \textbf{(Estimating $\hat{u}$).}  Assume that $|\hat{\|{x}^*\|}-\|x^*\||\leq \frac{1}{27}\|x^*\|$.  Let $\delta \in[0,1]$. Under the assumptions and notations from above if  $m_{2}=\max\{n\log(4n/\delta)((16+r_{up}^2)^4),n^5 (r_{up}^2+1)^4\log(4/\delta)\}$, and $\epsilon \leq C_{2}(16+r_{up}^2)^{-4}$ where $C_{1},C_{2}$ are universal constants. Then, we can say that at least with probability $0.99-\delta/2$, 
    \[\operatorname{dist}(\hat{u},u)\leq \frac{1}{27},\]
    where $u$ is the top normalized eigenvector of $\widetilde{\Sigma}_r$ matrix.
\end{theorem}
\begin{proof}  
Under the assumption of sample size and corruption level, \Cref{lemma:stability lemma for uncorrupted truncated samples} implies that the uniform distribution over the uncorrupted truncated samples $\{\widetilde{X}_{1},\widetilde{X}_{1}\ldots, \widetilde{X}_{m_2}\}$ satisfies the stability condition (\Cref{definition: Stability for nearly linear time robust PCA algorithm}) with probability at least $1-\delta/2$ with the $\gamma=5/(27^2 (16+r_{up}^2))$. We know that $\|\widetilde{X}\|\leq \hat{\tau}$. \Cref{lem:trunc} implies
    \[\|\Sigma-\widetilde{\Sigma}_r\|_{\mathrm{op}}=O( n (K_{4}^2+\|x^*\|^4)^2/\hat{\tau}^2)\leq \|x^*\|^4,\] and
    \[
-\|x^*\|^4 \cdot I \preceq (\widetilde{\Sigma}_r) - \Sigma \preceq \|x^*\|^4 \cdot I.
\]

Now, we know that the top eigenvalue of $\Sigma$ is $15\|x^*\|^2+\sigma^2$, and the rest of the eigenvalues are the same, with a value of $3\|x^*\|^4+\sigma^2$. Assume that \( \widetilde{\Sigma}_r \) has the eigenvalue decomposition \( \widetilde{\Sigma}_r = \sum_{i=1}^n \lambda_i f_i f_i^\top \), with \( \lambda_1 = \lambda_{\max}(\widetilde{\Sigma}_r), f_1=u \) such that the eigenvalues of \(\widetilde{\Sigma}_r\) are arranged in non-increasing order, where $\lambda_{\max}(A)$ denotes the largest eigenvalue of $A$. The Min-max theorem implies that $\lambda_1\in [14\|x^*\|^4+\sigma^2, 16\|x^*\|^4+\sigma^2)]$ and $\lambda_2\in [2\|x^*\|^4+\sigma^2,4\|x^*\|^4+\sigma^2]$. Now, $4\|x^*\|^4+\sigma^2\neq 14\|x^*\|^4+\sigma^2$ implies that $\lambda_1\neq \lambda_2$.
\begin{align*}
  \|\widetilde{\Sigma}_r\|_{\mathrm{op}} - (\hat{u}^\top (\widetilde{\Sigma}_r) \hat{u})&\geq\lambda_1-(\lambda_1\langle \hat{u}, u \rangle^2+(1-\langle \hat{u}, u \rangle^2)\lambda_2)= (\lambda_1-\lambda_2)(1-\langle \hat{u}, u \rangle^2)\\
  &= 10 \|x^*\|^4 (1-\langle \hat{u}, u \rangle^2).
\end{align*}
 So, now by using theorem \Cref{theorem: Gurantees of linear time robust pca algorithm}, we can conclude that with probability at least $0.99-\delta/2$, the following holds:
 \[\|\widetilde{\Sigma}_r\|_{\mathrm{op}} - (\hat{u}^\top (\widetilde{\Sigma}_r) \hat{u})\leq O(\gamma)\|\widetilde{\Sigma}_r\|_{\mathrm{op}},\]
which further implies \[10 \|x^*\|^4 (1-\langle \hat{u}, u \rangle^2) \leq O(\gamma)(16\|x^*\|^4+\sigma^2)\leq O(\gamma)(16+r_{up}^2) \|x^*\|^4\leq 5\|x^*\|^4/27^2,\]
 which again concludes that
\[|\langle \hat{u}, u \rangle|\geq \sqrt{1-\frac{1}{2 \times 27^2}}\geq \left(1-\frac{1}{2 \times 27^2}\right).\]
 Now,
 \begin{align*}
   \operatorname{dist}(\hat{u},u)&=\min\{\|\hat{u}-u\|_{2},\|\hat{u}+u\|_{2}\}\\
   &=\sqrt{2-2|\langle \hat{u}, u \rangle|}=\sqrt{2}\sqrt{1-|\langle \hat{u}, u \rangle|}\leq \sqrt{2}\sqrt{1-\left(1-\frac{1}{2 \times 27^2}\right)}=\frac{1}{27}.
 \end{align*}
So, from above we can conclude that $
\operatorname{dist}(\hat{u},u) \leq 1/27$. So, this implies that 
$
\|\hat{u} - u\| = O(1/27) \leq \frac{C}{27}.
$
If the constant $C$ hidden in the big-$O(\cdot)$ notation is less than $1$, then we are done. Otherwise, we set
$
\gamma = \frac{5}{C^{2} \, 27^{2} \, (r_{\mathrm{up}}^{2} + 16)},
$
and proceed to prove \cref{lemma:stability lemma for uncorrupted truncated samples} using this choice of $\gamma$. 
\end{proof}

\textbf{Proof of the \Cref{theorem:Robust PCA Spectral initialisation using stability lemma proof}}
\begin{proof}
    Now, the proof of the \Cref{theorem:Robust PCA Spectral initialisation using stability lemma proof} is analogous to \Cref{theorem:Robust PCA Spectral initialisation} once we have shown \Cref{theorem: estimating direction using the idea of dong paper}. So, we are omitting the proof here. Note that the time complexity of algorithm $A_{alt}$ follows from the \Cref{theorem: Gurantees of linear time robust pca algorithm} and \Cref{thm:robust_mean}.
\end{proof}
\section{Robust Phase Retrieval under Non-Zero-Mean Noise}
\label{sec: robust phase retrieval under non-zero mean noise}

\subsection{Local properties of population risk for non-zero mean noise robust phase retrieval}

We first verify the necessary geometric properties of the population risk $r(x)$ in the neighborhood of $\pm x^{\ast}$. These results are given in \cite{bunaR24_robust_phase}. We are reproving them for the sake of completeness. 
\begin{proposition}
\label{prop:poprisk for no corruption for non zero mean noise case}
\begin{equation}
\label{eq:poprisk for no corruption}
 r(x) = \mathbb{E}[(v-\langle b,x\rangle\langle c,x\rangle)^2]/2= \left(\|x^*\|^4+\|x\|^4-2\langle x, x^*\rangle^2+\sigma^2/2)\right)/2
\end{equation}
\begin{equation}
\label{eq:popriskgrad for no corruption}
\nabla r(x)=2 x\|x\|^2-2\langle x, x^*\rangle x^*
\end{equation}
\begin{equation}
\label{eq:popriskhessian for no corruption}
    \nabla^2 r(x)=2\|x\|^2\mathbb{I}_{n}+4xx^{\top}-2x^*x^{*\top}.
\end{equation}
\end{proposition}
\begin{proof} 
\begin{align*}
 2 r(x) &= \mathbb{E}[(v-\langle b,x\rangle\langle c,x\rangle)^2] 
   = \mathbb{E}[(\langle b,x^{*}\rangle\langle c,x^{*}\rangle+\zeta-\langle b,x\rangle\langle c,x\rangle)^2]\\
   &= \mathbb{E}[(\langle b,x^{*}\rangle\langle c,x^{*}\rangle-\langle b,x\rangle\langle c,x\rangle)^2]+\mathbb{E}[\zeta^2]
   +2\mathbb{E}[(\langle b,x^{*}\rangle\langle c,x^{*}\rangle-\langle b,x\rangle\langle c,x\rangle)\zeta]\\
   &= \mathbb{E}[(\langle b,x^{*}\rangle\langle c,x^{*}\rangle-\langle b,x\rangle\langle c,x\rangle)^2]+\sigma^2/2
   +2\mathbb{E}[(\langle b,x^{*}\rangle\langle c,x^{*}\rangle-\langle b,x\rangle\langle c,x\rangle)^2\mathbb{E}[\zeta\mid b,c]]\\
   &\overset{(a)}{=} \mathbb{E}[(\langle b,x^{*}\rangle\langle c,x^{*}\rangle-\langle b,x\rangle\langle c,x\rangle)^2]+\sigma^2/2\\
   &=\mathbb{E}[(\langle b,x^{*}\rangle\langle c,x^{*}\rangle)^2]+\mathbb{E}[(\langle b,x\rangle\langle c,x\rangle)^2]-2\mathbb{E}[(\langle b,x^{*}\rangle\langle c,x^{*}\rangle\langle b,x\rangle\langle c,x\rangle)^2]+\sigma^2/2\\
   &= \mathbb{E}[\langle b,x^{*}\rangle^2]\mathbb{E}[\langle c,x^{*}\rangle^2]+\mathbb{E}[\langle b,x\rangle^2]\mathbb{E}[\langle c,x\rangle^2]-2\mathbb{E}[\langle b,x^{*}\rangle\langle b,x\rangle]\mathbb{E}[\langle c,x^{*}\rangle\langle c,x\rangle]+\sigma^2/2\\
   &= \|x^*\|^4+\|x\|^4-2\langle x, x^*\rangle^2+\sigma^2/2
   \end{align*}
   where  step $(a)$ follows from the fact that $\mathbb{E}[\zeta\mid b,c]=0$. $\nabla r(x)$ and $\nabla^2 r(x)$ are obtained by further calculation.
\end{proof}

As noted before, for spectral initialization + robust gradient descent to work the loss should be strongly convex and smooth inside a ball of a non-zero radius around $x^*$. The specific structure of our transformed phase retrieval model ensures that these properties hold in our setting. Our next proposition is about this.

\begin{proposition}
\label{prop: local structure of blind deconvolution}
    The population loss $r(x)$ is $8\|x^*\|^2/3-$ strongly convex and $6\|x^*\|^2-$ smooth inside a ball of radius $\|x^*\|/9$ around $x^*$.
\end{proposition}
\begin{proof}
Consider $h = x - x^*$, such that $\|h\| \leq R \|x^*\|$. Then, replacing $x$ in the expression of the Hessian in \Cref{eq:popriskhessian for no corruption} with $x^* + h$, the Hessian becomes
\[
\nabla^2 r(x) = 4hh^\top + 4h(x^*)^\top + 4x^*h^\top + 2x^*(x^*)^\top+ 2\|h\|^2 \mathbb{I}_n + 4(h^\top x^*)\mathbb{I}_n + 2\|x^*\|^2 \mathbb{I}_n.
\]
Using the triangle and Cauchy--Schwarz inequalities along with the identity $\|uv^\top\|_{\mathrm{op}} = \|u\|\|v\|$, we obtain
\[
\lambda_{\max}(\nabla^2 r(x)) = \|\nabla^2 r(x)\|_{\mathrm{op}} \leq 6\|h\|^2 + 12\|x^*\|\|h\| + 4\|x^*\|^2 \leq (6R^2 + 12R + 4)\|x^*\|^2.
\]
 Thus, the smoothness parameter of $r(x)$ is $(6R^2 + 12R + 2)\|x^*\|^2$ for some choice of $R$. Now, to calculate the strong convexity parameter, we need to put a lower bound on  $\lambda_{\min}(\nabla^2 r(x))$.
\begin{align*}
    (x^*)^\top \nabla^2 r(x) x^* &= 4(h^\top x^*)^2 + 12(h^\top x^*)\|x^*\|^2 +2 \|x^*\|^2 \|h\|^2 + 4\|x^*\|^4\geq 12(h^\top x^*)\|x^*\|^2 + 4 \|x^*\|^4\\
    &\overset{(a)}{\geq} 4\|x^*\|^2 \big(\|x^*\|^2 - 3\|h\|\|x^*\|\big)
\end{align*}
where $(a)$ follows from the Cauchy-Schwarz inequality. Thus,
$
\lambda_{\min}(\nabla^2 r(x)) \geq 4(1 - 3R)\|x^*\|^2.
$
The conclusion of the theorem follows by choosing $R = \frac{1}{9}$, noting that $6R^2 + 12R + 4 \leq 6$ in this case.
\end{proof}

\subsection{ Theorem 6.3 of \texorpdfstring{\cite{bunaR24_robust_phase}}.}
\label{subsection: fromal version of robust gradient descent for blind deconvolution}

\begin{theorem}
\label{thm:grad blind deconvolution }
\textbf{(Gradient descent for blind deconvolution, see Theorem 6.3 from \cite{bunaR24_robust_phase}.)} There exists a polynomial time algorithm (see Algorithm 4 from \cite{bunaR24_robust_phase}) that, given a point $x_0 \in B(\pm x^*,\|x^*\|/9)$ and if $2\epsilon \leq \frac{C_1 \|x^*\|^4}{\sigma^2}$ and $P$ draws of $2\tilde{m}$ samples where $P\tilde{m} \geq C_2 P \cdot \max\{n \log n, \log(1/\delta)\} \cdot \frac{\sigma^2}{\|x^*\|^4}$, with probability at least $1-P\delta$, outputs $x_P$ satisfying 
\begin{align*}
  \frac{dist(x_P,x^*)}{\left\|x^*\right\|}\leq C_3 \exp\left( -C_4  P  (1 - \sqrt{\epsilon}) \right)+ C_3 \frac{\sigma}{\|x^*\|^2} \left( \sqrt{\frac{n \log n}{\tilde{m}}} + \sqrt{\frac{\log(1/\delta)}{\tilde{m}}} \right)
+ C_3 \frac{\sigma}{\|x^*\|^2} \sqrt{\epsilon}.
\end{align*}
\end{theorem}

\subsection{ Proof of \texorpdfstring{\Cref{lemma: trunc for non mean zero noise}}.}
\label{proof: proof of the truncation lemma fro non-zero mean noise case}

\begin{proof}
  Like before, we will first try to bound the tail probability:
\begin{align*}
   p &=\mathbb{P}(\|X\|\geq \tau)\leq \mathbb{E}[\|X\|^2]/ \tau^2 = \mathbb{E}[(\langle b, x^*\rangle \langle c, x^*\rangle+\zeta)^2\|b\|^2]/ \tau^2 =\frac{\mathbb{E}\left[(\langle b, x^*\rangle^2\langle c, x^*\rangle^2+\zeta^2)\left(\sum_{j=1}^n b^{2}_{j}\right)\right]}{\tau^2}\\
   &\overset{(a)}{\leq} \left(\|x^*\|^2\operatorname{tr}(\|x^*\|^2\mathbb{I}_{n}+2x^{*}x^{*\top}) +(n\sigma^2/ 2)\right)/\tau^2\leq\frac{(n+2)(\|x^*\|^4+\sigma^2/2)}{\tau^2},
\end{align*}
where $(a)$ follows from fact that $b,c$ are independent and $\mathbb{E}[\langle b, x^*\rangle^2 b b^{\top}]=\|x^*\|^2\mathbb{I}_{n}+2x^{*}x^{*\top}$. Now,
\begin{align*}
  \Sigma&= \mathbb{E}[XX^{\top}]= \mathbb{E}[XX^{\top}\mathbb{I}_{\|X\|\leq \tau}]+\mathbb{E}[XX^{\top}\mathbb{I}_{\|X\|\geq \tau}]=\widetilde{\Sigma}_r+\mathbb{E}[XX^{\top}\mathbb{I}_{\|X\|\geq \tau}].
\end{align*} 

Upper bounding 
\begin{align*}
 \|\mathbb{E}[XX^{\top}\mathbb{I}_{\|X\|\geq \tau}]\|_{\mathrm{op}}&=\sup_{\|v\|=1}\mathbb{E}[\langle v,X\rangle^{2}\mathbb{I}_{\|X\|\geq \tau}]\leq \sup_{\|v\|=1}\sqrt{p}\sqrt{\mathbb{E}[\langle v,X\rangle^{4}]}\\
 &\leq \sup_{\|v\|=1}\sqrt{p}\sqrt{\mathbb{E}[8(\zeta^4+\langle b,x^*\rangle^{4}\langle c,x^*\rangle^{4})\langle b,v\rangle^{4}]}\leq\sqrt{p}\sqrt{8(6 K_{4}^4+ 2529 \|x^*\|^8)}\\
 &=O( \sqrt{p}\sqrt{( K_{4}^4+ \|x^*\|^8)})=O(\sqrt{p}(K_{4}^2+ \|x^*\|^4)) . 
    \end{align*}
 Now, $\sqrt{p}= O( \sqrt{n (\sigma^2/2+ \|x^*\|^4}/\tau)$. So, this implies 
\[\|\Sigma-\widetilde{\Sigma}_r\|_{\mathrm{op}}=O(\sqrt{p}(K_{4}^2+ \|x^*\|^4))= O( \sqrt{n(\sigma^2/2+ \|x^*\|^4)}(K_{4}^2+ \|x^*\|^4)/\tau).\]   
\end{proof}

\subsection{Formal version of \texorpdfstring{\Cref{algorithm: Spectral Initialisation with Robust PCA for additive non-zero mean noise 1}}%
{Algorithm: Spectral Initialisation with Robust PCA for additive non-zero mean noise} 
and proof of \texorpdfstring{\Cref{theorem:Robust PCA Spectral initialisation for non-zero mean noise}}%
{Theorem: Robust PCA Spectral initialisation for non-zero mean noise}}
\label{subsection: formal algorithm and its proof}

\textbf{The algorithm.} Based on the discussions in \Cref{subsec: non mena zero spectral initialzation}, we have the following informal algorithm (which assumes a known upper bound $r_{up}$ for $K_4/\|x^*\|^2$):
\begin{enumerate}
   \item Compute an estimate $\hat{v}$ for $(\|x^*\|^4+\sigma^2/2)^{1/2}$ by using the fact that $\mathbb{E}[v^2] =\|x^*\|^4+\sigma^2/2 $.  
    \item Use this estimate (together with $r_{up}$) to compute the truncation threshold $\tau= O(\sqrt{n}\hat{v}(r^2_{up}+1)$.
    \item  Truncate the random variable $X := vb$ using the truncation parameter $\tau$.
   \item Apply a robust PCA algorithm (Algorithm 2 from \cite{kong2020robust}, \Cref{theorem: Guarantees of Algorithm 2 from kong robust 2020 paper 1}) on the truncated data to obtain an estimate $\hat{u}$ and $\hat{\lambda}$ of the top (normalized) eigenvector and top eigenvalue of the matrix $\widetilde{\Sigma}_{r}$.
    \item Calculate $(\hat{\lambda}-\hat{v}^2)/2$ as an robust estimate for $\|x^*\|^4$. Now calculate $\hat{\|x^*\|}=\left((\hat{\lambda}-\hat{v}^2)/2\right)^{1/4}$ as an robust estimate for $\|x^*\|$.
    \item Scale $\hat{u}$ using the estimated norm $\hat{\|x^*\|}$ to obtain the output $x_0$ of the robust spectral initialization.
\end{enumerate}

The formal algorithm for robust spectral initialization in blind deconvolution problem is as follows:
\begin{algorithm}[H]
\caption{Spectral Initialisation.}
\label{algorithm: Spectral Initialisation with Robust PCA for additive non-zero mean noise 1}
\textbf{Inputs:} failure probability $\delta \in (0, 1)$, corruption fraction $\epsilon > 0$, an upper bound $r_{up}$ on $K_4/\|x^*\|^2$, access to 4 batches of samples out of which 2 batches are of sizes $m_{1}=\Omega(\log(2/\delta)(1+r_{up}^2))$ and remaining batches are of sizes $m_{2}=\Omega\left( n(r_{up}^2+1)\log(2n/(\delta \epsilon))/\epsilon\right)$ respectively from \Cref{def:error_model}.\\
\textbf{Output:} $x_0 \in \mathbb{R}^n$
\begin{algorithmic}[1]
\State We take two batches of data of size $m_{1}$, each consisting of samples of the form $(a_{i}, y_{i})$, 
and apply the symmetrization trick to obtain a new dataset of size $m_{1}$ of the form $(v_{i}, b_{i}, c_{i})$. 
Analogously, we perform the same procedure on the other two batches of data of size $m_{2}$.

\State \textbf{$(\|x^*\|^4+\sigma^2/2)^{1/2}$ calculation}:
\begin{itemize}
    \item Receive a set of samples $T_y = \{\left(b_j\right)\}_{j=1}^{m_1}$ of size $m_1$ from the step 1. Form the dataset $T_y^{2} = \left\{\left(b_j^{2}\right)\right\}_{j=1}^{m_1}$.
    \item Let $\hat{v}^2$ be the robust mean estimate of the dataset $T_y^2$ by \Cref{thm:robust_mean}. 
\end{itemize}
    
    \State \textbf{Robust PCA and robust eigenvalue}:
\begin{itemize}
    \item Calculate the truncation parameter $\hat{\tau}=O(\sqrt{n }(r_{up}^2+1)\hat{v})$.
    \item Receive a set of samples $T = \left\{\left(v_{j}, b_{j},c_j\right)\right\}_{j=1}^{m_{2}}$ of size $m_{2}$ from step 1.
    \item  Form the truncated dataset $\widetilde{D}=\{\widetilde{X}_{j}=X_{j}\mathbb{I}_{\|X_{j}\|\leq \hat{\tau}}:X_{j}=v_{j}b_{j}\}_{j=1}^{m_{2}}$. 
    \item Let $\hat{u}$ and $\hat{\lambda}$ be the output of the robust PCA  algorithm (Algorithm 2 from \cite{kong2020robust}) applied to the truncated dataset $\widetilde{D}$.
\end{itemize}
\State \textbf{$\|x^*\|$ calculation}:
\begin{itemize}
    \item Calculate $\hat{\|x^*\|}=\left((\hat{\lambda}-\hat{v}^2)/2\right)^{1/4}$. 
\end{itemize}
\State \textbf{Scaling}:
\begin{itemize}
    \item Return $x_0=\hat{\|x^*\|}\hat{u}$.
\end{itemize}
 \end{algorithmic}
\end{algorithm}
We analyze all the steps of \Cref{algorithm: Spectral Initialisation with Robust PCA for additive non-zero mean noise 1} sequentially and, finally, combine the results of Steps 2, 3, and 4 of \Cref{algorithm: Spectral Initialisation with Robust PCA for additive non-zero mean noise 1} to provide the proof of \Cref{theorem:Robust PCA Spectral initialisation for non-zero mean noise}. We first analyze the step 2 of \cref{algorithm: Spectral Initialisation with Robust PCA for additive non-zero mean noise 1}.
\begin{theorem}
\label{theorem: Guarantees of estimating norm of square of v and v for additive non-zero mean noise}
\textbf{(Step 2 of \Cref{algorithm: Spectral Initialisation with Robust PCA for additive non-zero mean noise 1}: Estimate $(\|x^*\|^4+\sigma^2/2)^{1/2}$ using a robust mean estimator)}.  If $m_{1}=\Omega\left(\log(4/\delta)\left(1+r_{up}^2\right)^2\right)$ and $\epsilon =O\left(\left(1+r_{up}^2\right)^{-2}\right)$, then with probability at least $1-\delta/4$ we can conclude that 
\begin{equation}
\label{eq: calculation of v^2}
|\hat{v}^2-(\|x^*\|^4+\sigma^2/2)|=O(\left\|x^*\right\|^4)=\|x^*\|^4/54
\end{equation}
\begin{equation}
\label{equation: calculation of v}
    |\hat{v}-(\|x^*\|^4+\sigma^2/2)^{1/2}|= O(\left\|x^*\right\|^2)=\|x^*\|^2/54.
\end{equation}
  \end{theorem}
\begin{proof}
      Note that $\mathbb{E}[v^2]=\|x^*\|^4+\sigma^2/2$  and $\operatorname{Var}(v^2)\leq 72(r_{up}^4+1)\left\|x^*\right\|^8.$  Then, using \Cref{thm:robust_mean} we can say that, with probability at least $1-\delta/4$ :
\begin{align*}
    \left|\hat{v}^2-(\left\|x^*\right\|^4+\sigma^2/2)\right| & =O\left(\sqrt{\operatorname{Var}(v^2)}\left(\sqrt{2\epsilon}+\sqrt{\frac{\log (4 / \delta)}{m_1}}\right)\right)\\
&=O\left(\sqrt{(r_{up}^4+1)\left\|x^*\right\|^8}\left(\sqrt{2\epsilon}+\sqrt{\frac{\log (4 / \delta)}{m_1}}\right)\right)\\
&=O\left((r_{up}^2+1)\left\|x^*\right\|^4\left(\sqrt{2\epsilon}+\sqrt{\frac{\log (4/ \delta)}{m_1}}\right)\right).
\end{align*}

Now, this implies
\begin{align*}
   |\hat{v}-(\|x^*\|^4+\sigma^2/2)^{1/2}|&=\frac{\left|\hat{v}^2-(\left\|x^*\right\|^4+\sigma^2/2)\right|}{(\hat{v}+(\|x^*\|^4+\sigma^2/2)^{1/2})}\leq \frac{\left|\hat{v}^2-(\left\|x^*\right\|^4+\sigma^2/2)\right|}{(\|x^*\|^4+\sigma^2/2)^{1/2}}\leq \frac{\left|\hat{v}^2-(\left\|x^*\right\|^4+\sigma^2/2)\right|}{\|x^*\|^2} \\
&=O\left((r_{up}^2+1)\left\|x^*\right\|^2\left(\sqrt{2\epsilon}+\sqrt{\frac{\log (4/ \delta)}{m_1}}\right)\right). 
\end{align*}
\end{proof}

Now we proceed to the analysis of Step~3 of \Cref{algorithm: Spectral Initialisation with Robust PCA for additive non-zero mean noise 1}.

\begin{theorem}
\label{theorem: estimating robust pca and robsut eigenvectr for addtive noise with non-zero mean}
 (\textbf{Step 3 of \Cref{algorithm: Spectral Initialisation with Robust PCA for additive zero mean noise 1}: Estimating $\hat{u}$.}) Let's assume that $|\hat{v}-(\|x^*\|^4+\sigma^2/2)^{1/2}|\leq \left\|x^*\right\|^2/54$, $|\hat{v}^2-(\|x^*\|^4+\sigma^2/2)|\leq \left\|x^*\right\|^4/54$ and \( \delta \in (0, 0.5) \) is a positive small constant. If $m_{2}=\Omega\left( (n+ n(r_{up}^2+1)^2\sqrt{\epsilon})\log(4n/(\delta \epsilon))/\epsilon\right)$ and  \( \epsilon \in (0, 1/36] \).  Then, under the above notation, with probability at least \( 1 - \delta/4 \), the following holds:
 \begin{equation}
 \label{eq: calculating the top eigenvector for non-zero mean noise}
      \operatorname{dist}(\hat{u},u) \leq \sqrt{2-2\sqrt{1-O\left( (r_{up}^2+1) \sqrt{\epsilon}\right)}},
 \end{equation}
  \begin{equation}
     |\lambda_{\hat{u}}-\|\widetilde{\Sigma}_r\|_{\mathrm{op}}|=O(\epsilon(r_{up}^2+1)+\sqrt{\epsilon}(r_{up}^2+1))\|x^*\|^4, 
  \end{equation}
  where $u$ is the top normalized eigenvector of $\widetilde{\Sigma}_{r}$ and $\lambda_{\hat{u}}$ denote the eigenvalue corresponding to the top eigenvector 
$\hat{u}$ of the estimated second-moment matrix.
.
\end{theorem}
\begin{proof}
 As before, we invoke \Cref{theorem: Guarantees of Algorithm 2 from kong robust 2020 paper 1} 
to bound the distance between $\hat{u}$ and the top eigenvector $u$ of $\widetilde{\Sigma}_{r}$. 
In order to apply \Cref{theorem: Guarantees of Algorithm 2 from kong robust 2020 paper 1}, 
it is necessary to verify that all the stated assumptions are satisfied. 
We now proceed to check these assumptions one by one.
We know that $\|\widetilde{X}\|\leq \hat{\tau}$. Note that $|\hat{v}-(\|x^*\|^4+\sigma^2/2)^{1/2}|\leq \left\|x^*\right\|^2/54$ implies that $|\hat{v}-(\|x^*\|^4+\sigma^2/2)^{1/2}|\leq (\|x^*\|^4+\sigma^2/2)^{1/2}/54$, which further implies that $\frac{\|x^*\|^4+\sigma^2/2)^{1/2}}{\hat{v}}\leq 54/53.$ \Cref{lemma: trunc for non mean zero noise} implies
    \[
\|\Sigma - \widetilde{\Sigma}_r\|_{\mathrm{op}} =O(\sqrt{n}(\|x^*\|^4+\sigma^2/2)^{1/2}(r_{up}^2+1)\|x^*\|^4/\hat{\tau}) \leq \|x^*\|^4/54\quad \text{and} \quad -\frac{\|x^*\|^4}{54} I \preceq \widetilde{\Sigma}_r - \Sigma \preceq  \frac{\|x^*\|^4}{54}I.
\]

The first assumption is to show that $\|\widetilde{X}\widetilde{X}^\top-(\widetilde{\Sigma}_r)\|_{\mathrm{op}}\leq B$ for all $\widetilde{X}$ with probability one.
    \begin{align*}
      \|\widetilde{X}\widetilde{X}^\top-(\widetilde{\Sigma}_r)\|_{\mathrm{op}}&\leq \|\widetilde{X}\|^2+\lambda_{max}(\widetilde{\Sigma}_r)\leq\hat{\tau}^2+\lambda_{max}(\Sigma)+\|x^*\|^4=O( n (r_{up}^2+1)^2 \hat{v}^{2}+\left(4\|x^*\|^4+\sigma^2\right))\\
      &=O \left( n (r_{up}^2+1)^2 (\|x^*\|^4+\sigma^2)\right)+(4\|x^*\|^4+\sigma^2)=O \left( n (r_{up}^2+1)^2 (\|x^*\|^4+\sigma^2)\right) .   
    \end{align*}
The second assumption is to show that $\mathbb{E}\left[ \langle xx^{\top} ,\widetilde{X}\widetilde{X}^\top-(\widetilde{\Sigma}_r)\rangle^2  \right]\leq v'^2.$
    
    Consider $\|x\|\leq 1$, then
     \begin{align*}
     \mathbb{E}\left[ \langle xx^{\top} ,\widetilde{X}\widetilde{X}^\top-(\widetilde{\Sigma}_r)\rangle^2  \right]&= \mathbb{E}\left[( \langle x ,\widetilde{X}\rangle^2-x^{\top}\widetilde{\Sigma}_{r}x)^2 \right]=\mathbb{E}\left[( \langle x ,\widetilde{X}\rangle)^4\right]-(x^{\top}\widetilde{\Sigma}_{r}x)^4
     \overset{(a)}{\leq} \mathbb{E}\left[( \langle x ,X\rangle)^4\right]\\
     &\leq \mathbb{E}[8(\zeta^4+\langle b,x^*\rangle^{4}\langle c,x^*\rangle^{4})\langle b,x\rangle^{4}]\overset{(b)}{=} O( K_{4}^4+\|x^*\|^8)= O( (r_{up}^2+1)^2 \|x^*\|^8),
    \end{align*}
where \((a)\) follows by applying the trivial upper bound \(\hat{\tau} \leq \infty\) and $(b)$ follows from \Cref{lemma: calcualtions needed for spectral initialization with covest estimator}.

This implies \[\max_{x:\|x\|\leq 1} \mathbb{E}\left[ \langle xx^{\top} ,\widetilde{X}\widetilde{X}^\top-\widetilde{\Sigma}\rangle^2  \right]= O( (r_{up}^2+1)^2  \|x^*\|^8).\] 

In the notations of \Cref{theorem: Guarantees of Algorithm 2 from kong robust 2020 paper 1}, we have  
\[
B = O\left(\, n\, (r_{\mathrm{up}}^2+1)^2(\|x^*\|^4 + \sigma^2)\right), \quad v' = O(\, (r_{\mathrm{up}}^2+1)\, \|x^*\|^4).
\]  
To apply \Cref{theorem: Guarantees of Algorithm 2 from kong robust 2020 paper 1}, we must ensure that its sample complexity condition is satisfied. This requires
\[
m_2 = \Omega\left( \frac{n + \frac{B\sqrt{\epsilon}}{v'}}{\epsilon} \cdot \log\left( \frac{4n}{\delta \epsilon} \right) \right) = \Omega\left( \frac{n + \, n\, (r_{\mathrm{up}}^2+1)^2\sqrt{\epsilon}}{\epsilon} \cdot \log\left( \frac{4n}{\delta \epsilon} \right) \right).
\]

So, we have verified all assumptions of \Cref{theorem: Guarantees of Algorithm 2 from kong robust 2020 paper 1}. So, by using \Cref{theorem: Guarantees of Algorithm 2 from kong robust 2020 paper 1} and \Cref{theorem: eigenvalue bound from algorithm of theorem}, we can conclude that, with probability at least \( 1 - \delta/4 \), the output \( \hat{u} \in \mathbb{R}^{n \times 1} \) and $\lambda_{\hat{u}}$ of Algorithm 2 satisfies:
\[
\operatorname{Tr}[P_1(\widetilde{\Sigma}_r)] - \operatorname{Tr}[\hat{u}^\top (\widetilde{\Sigma}_r) \hat{u}] = O\left( \epsilon \cdot \operatorname{Tr}[P_1(\widetilde{\Sigma}_r)] + \sqrt{\epsilon} \cdot v'  \right), \quad |\lambda_{\hat{u}}-\|\widetilde{\Sigma}_r\|_{\mathrm{op}}|\leq 106 \epsilon \|\widetilde{\Sigma}_r\|_{\mathrm{op}}+212\sqrt{\epsilon} v'.
\]

We know that the top eigenvalue of \(\Sigma\) is \(3\|x^*\|^4 + \sigma^2/2\), while all remaining eigenvalues equal \(\|x^*\|^4 + \sigma^2/2\). Let \(\widetilde{\Sigma}_r = \sum_{i=1}^n \lambda_i f_i f_i^\top\) be the eigenvalue decomposition of \(\widetilde{\Sigma}_r\), with \(\lambda_1 = \lambda_{\max}(\widetilde{\Sigma}_r)\) and \(f_1 = u\) such that eigenvalues of \(\widetilde{\Sigma}_r\) are arranged in non-increasing order, where $\lambda_{\max}(A)$ denotes the largest eigenvalue of $A$. By the min–max theorem, we have \(\lambda_1 \in [(161/54)\|x^*\|^4 + \sigma^2/2,\ (163/54)\|x^*\|^4 + \sigma^2/2]\) and \(\lambda_2 \in [(53/54)\|x^*\|^4 + \sigma^2/2,\ (55/54)\|x^*\|^4 + \sigma^2/2]\). Since \((55/54)\|x^*\|^4 + \sigma^2/2 \neq (161/4)\|x^*\|^4 + \sigma^2/2\), it follows that \(\lambda_1 \neq \lambda_2\).
\begin{align*}
  \operatorname{Tr}[P_1(\widetilde{\Sigma}_r)] - \operatorname{Tr}[\hat{u}^\top (\widetilde{\Sigma}_r) \hat{u}]&\geq\lambda_1-(\lambda_1\langle \hat{u}, u \rangle^2+(1-\langle \hat{u}, u \rangle^2)\lambda_2)= (\lambda_1-\lambda_2)(1-\langle \hat{u}, u \rangle^2)\\
  &=106/54(1-\langle \hat{u}, u \rangle^2).
\end{align*}
So now, 
 \[(1-\langle \hat{u}, u \rangle^2)=  O\left( \epsilon \left(1+\frac{\sigma^2}{\|x^*\|^4}\right) + 73 (r_{up}^2+1) \sqrt{\epsilon}  \right)=O\left( \epsilon (r_{up}^2+1) +  (r_{up}^2+1) \sqrt{\epsilon}  \right)=O\left( (r_{up}^2+1) \sqrt{\epsilon}  \right).\]
 This implies \[|\langle \hat{u}, u \rangle|\geq \sqrt{1-O\left( (r_{up}^2+1) \sqrt{\epsilon}\right)}.\]
 Now, we can conclude that,
 \begin{align*}
   \operatorname{dist}(\hat{u},u)&=\min\{\|\hat{u}-u\|_{2},\|\hat{u}+u\|_{2}\}\\
   &=\sqrt{2-2|\langle \hat{u}, u \rangle|}\\
   &\leq \sqrt{2-2\sqrt{1-O\left( (r_{up}^2+1) \sqrt{\epsilon}\right)}}.
 \end{align*}
and 
\[|\lambda_{\hat{u}}-\|\widetilde{\Sigma}_r\|_{\mathrm{op}}|\leq 106 \epsilon \|\widetilde{\Sigma}_r\|_{\mathrm{op}}+212\sqrt{\epsilon} v'=O(\epsilon(r_{up}^2+1)+\sqrt{\epsilon}(r_{up}^2+1))\|x^*\|^4.\]
\end{proof}
Now we proceed to the analysis of Step~4 of \Cref{algorithm: Spectral Initialisation with Robust PCA for additive non-zero mean noise 1}. 

\begin{theorem}
\label{theorem: scaling of norm of x^*}
\textbf{(Step 4 of \Cref{algorithm: Spectral Initialisation with Robust PCA for additive non-zero mean noise 1}: Estimate $\|x^*\|$)}.  If $m_{1}=\Omega\left(\log(4/\delta)\left(1+r_{up}^2\right)^2\right)$,\\ $m_{2}=\Omega\left( (n+ n(r_{up}^2+1)^2\sqrt{\epsilon})\log(4n/(\delta \epsilon))/\epsilon\right)$ and $\epsilon =O\left(\left(1+r_{up}^2\right)^{-2}\right)$, then with probability at least $1-\delta/2$ we can conclude that 
\begin{equation}
\label{eq: calculation of a parameter in truncation}
|\hat{\|x^{*}\|}-\|x^*\||\leq \frac{\|x^*\|}{27}
\end{equation}    
\end{theorem}
\begin{proof}
   Using \Cref{theorem: Guarantees of estimating norm of square of v and v for additive non-zero mean noise}, \Cref{lemma: trunc for non mean zero noise}  and \Cref{theorem: estimating robust pca and robsut eigenvectr for addtive noise with non-zero mean}, we can conclude that with probability at least $1-\delta/2$ the following holds:
  \[|\hat{v}^2-(\|x^*\|^4+\sigma^2/2)|\leq\left\|x^*\right\|^4/54,\quad     \|\Sigma - \widetilde{\Sigma}_r\|_{\mathrm{op}} \leq \|x^*\|^4/54,\] 
  \[|\lambda_{\hat{u}}-\|\widetilde{\Sigma}_r\|_{\mathrm{op}}|=O(\epsilon(r_{up}^2+1)+\sqrt{\epsilon}(r_{up}^2+1))\|x^*\|^4=\|x^*\|^4/27.\]
 This implies
\begin{align*}
    |(\lambda_{\hat{u}}-\hat{v}^2)/2-\|x^*\|^4|&=|(\lambda_{\hat{u}}-(3\|x^*\|^4+\sigma^2/2))/2-(\hat{v}^2-(\|x^*\|^4+\sigma^2/2))/2)|\\
    &\leq 1/2|(\lambda_{\hat{u}}-(3\|x^*\|^4+\sigma^2/2))|+1/2|(\hat{v}^2-(\|x^*\|^4+\sigma^2/2)))|\\
    &\leq 1/2 |(\lambda_{\hat{u}}-\|\widetilde{\Sigma}_r\|_{\mathrm{op}}|+1/2|\|\widetilde{\Sigma}_r\|_{\mathrm{op}}-(3\|x^*\|^4+\sigma^2/2))|+1/2|(\hat{v}^2-(\|x^*\|^4+\sigma^2/2)))|\\
    &\leq \frac{1}{2}\left(\frac{\|x^*\|^4}{27}+\frac{\|x^*\|^4}{54}+\frac{\|x^*\|^4}{54}\right)
    =\frac{\|x^*\|^4}{27}
\end{align*}
 Now, we know that for any $a, b, c>0$ with $b^2>c,\left|a^2-b^2\right| \leq c$ implies $|a-b| \leq b-\sqrt{b^2-c}$. By using the above property, we get to

$$
|\sqrt{(\lambda_{\hat{u}}-\hat{v}^2)/2}-\left\|x^*\right\|^2 |\leq \|x^*\|^2-\sqrt{26/27}\|x^*\|^2\leq \|x^*\|^2/27.
$$
Now, using the above result once more we conclude that
\[
\left|((\lambda_{\hat{u}}-\hat{v}^2)/2)^{1/4}  - \|x^*\| \right|\leq \|x^*\|-\sqrt{26/27}\|x^*\|\leq \|x^*\|/27 .
\]
Finally, we can conclude that 
\[
\left|\hat{\|x^{*}\|} - \|x^*\| \right|\leq \|x^*\|/27 .
\]
  \end{proof}
\textbf{Proof of \Cref{theorem:Robust PCA Spectral initialisation for non-zero mean noise}.}

Now we are ready to give the proof of \Cref{theorem:Robust PCA Spectral initialisation for non-zero mean noise}.
\begin{proof}
Under the assumptions on the sample complexity \( m_1, m_2 \) and the corruption level \( \epsilon \), and using \Cref{theorem: scaling of norm of x^*}, \Cref{theorem: estimating robust pca and robsut eigenvectr for addtive noise with non-zero mean}, \Cref{theorem: Guarantees of estimating norm of square of v and v for additive non-zero mean noise}, and \Cref{lemma: trunc for non mean zero noise}, we conclude that, with probability at least \(1 - \delta\),
 the following holds:
\[\operatorname{dist}(\hat{u},u) \leq \sqrt{2-2\sqrt{1-O\left( (r_{up}^2+1) \sqrt{\epsilon}\right)}} \leq (1/27) ,\]
\[\|\Sigma-\widetilde{\Sigma}_r\|_{\mathrm{op}}=O(\sqrt{n}(\|x^*\|^4+\sigma^2/2)^{1/2}(r_{up}^2+1) \|x^*\|^4/\hat{\tau}\leq \frac{\|x^*\|^4}{54}\leq \frac{\sqrt{2}\|x^*\|^4}{54} ,\]
\[|\hat{\|x^*\|}-\|x^*\||\leq \|x^*\|/27,\]
where $u$ is top normalized eigenvector of $\widetilde{\Sigma}_r$. Now, by applying \Cref{throrem:conclusion of davis-kahan} (Davis-Kahan), we can conclude that
\[\operatorname{dist}(\|x^*\|u,x^*)\leq \frac{2\sqrt{2}\|x^*\|\|\Sigma-(\widetilde{\Sigma}_r)\|_{\mathrm{op}}}{2\|x^*\|^4}\leq \frac{\|x^*\|}{27}\quad \text{and}\]
\[\operatorname{dist}(-\|x^*\|u,x^*)\leq \frac{2\sqrt{2}\|x^*\|\|\Sigma-(\widetilde{\Sigma}_r)\|_{\mathrm{op}}}{2\|x^*\|^4}\leq \frac{\|x^*\|}{27}.\]
So,
\begin{align*}
\operatorname{dist}\left(x_0, x^*\right) &= \operatorname{dist}\left(\hat{\|x^*\|}\hat{u}, x^*\right)\leq \operatorname{dist}(\hat{\|x^*\|}\hat{u},\|x^*\|\hat{u})+\operatorname{dist}(\|x^*\|\hat{u},\|x^*\|u)+dist\left(\|x^*\|u, x^*\right)\\
&= \|\hat{\|x^*\|}\hat{u}-\|x^*\|\hat{u}\|+\operatorname{dist}(\|x^*\|\hat{u},\|x^*\|u)+dist\left(\|x^*\|u, x^*\right)\\
&= |\hat{\|x^*\|}-\|x^*\||+\|x^*\|\operatorname{dist}(\hat{u},u)+ dist\left(\|x^*\|u, x^*\right)\\
 &\leq \frac{\|x^*\|}{27}+\frac{\|x^*\|}{27}+\frac{\|x^*\|}{27}=\frac{\|x^*\|}{9} .
 \end{align*}

\end{proof}
\subsection{ Proof of \texorpdfstring{\Cref{prop: some properties of random avriable}}.}
\label{proof: proof of the some properties of new defied random variable for non mean zero case}
\begin{proof}
    \begin{align*}
    \mathbb{E}[v^2]&=\mathbb{E}\left[\left(b^{\top} x^*\right)^2\left(c^{\top} x^*\right)^2+\zeta^2+2\zeta \left(b^{\top} x^*\right)\left(c^{\top} x^*\right)\right]=\mathbb{E}\left[\left(b^{\top} x^*\right)^2\left(c^{\top} x^*\right)^2+\zeta^2\right]\\
    &=\left\|x^*\right\|^4+\sigma^2/2\\\\
    \operatorname{Var}(v^2)&=\mathbb{E}\left[(\left(b^{\top} x^*\right)\left(c^{\top} x^*\right)+\zeta)^{4}\right]-(\left\|x^*\right\|^4+\sigma^2/2)^2
\leq \mathbb{E}\left[8\left(b^{\top} x^*\right)^4\left(c^{\top} x^*\right)^4\right]+\mathbb{E}\left[8(\zeta)^4\right]\\&-\left\|x^*\right\|^8\leq 72\left\|x^*\right\|^8+16 K_{4}^{4}-\left\|x^*\right\|^8\leq 72(r_{up}^4+1)\left\|x^*\right\|^8\\\\
\mathbb{E}[\upsilon b]&= \mathbb{E}[\langle b, x^* \rangle \cdot \langle c, x^* \rangle b + \zeta b]=\mathbb{E}[\langle b, x^* \rangle \cdot \langle c, x^* \rangle b]=0,\mathbb{E}[\upsilon c]= \mathbb{E}[\langle b, x^* \rangle \cdot \langle c, x^* \rangle c + \zeta c]\\
&=\mathbb{E}[\langle b, x^* \rangle \cdot \langle c, x^* \rangle c]=0\\\\
    \mathbb{E}[\upsilon b^\top c]&=\mathbb{E}[\langle b, x^* \rangle \cdot \langle c, x^* \rangle b^\top c + \zeta b^\top c]=\mathbb{E}[\langle b, x^* \rangle \cdot \langle c, x^* \rangle b^\top c]=\mathbb{E}\left[\sum_{i,j,k=1}^n b_{j}x^*_{j}c_{i} x^*_{i}b_{k}c_{k}\right]\\
    &=\sum_{i,j,k=1}^n \delta_{\{j=k\}}x^*_{j} \delta_{\{i=k\}}x^*_{i}=\sum_{k=1}^n x^{*2}_{k}=\|x^*\|^2,\\\\
    \operatorname{Var}(vb^{\top}c)&=\mathrm{Tr}(\mathbb{E}[\langle b, x^*\rangle^2bb^{\top}]\mathbb{E}[\langle c, x^*\rangle^2cc^{\top}])+(\sigma^2/2)\mathbb{E}[(b^{\top}c)^2]-\|x^*\|^4=
\mathrm{Tr}((\left\|x^*\right\|^2 \mathbb{I}_n+2 x^*x^{*\top})^2)\\&+(\sigma^2/2) n-\|x^*\|^4=(8+n)\|x^*\|^4+(\sigma^2/2)n-\|x^*\|^4=n\|x^*\|^4+n(\sigma^2/2)+7\|x^*\|^4.\\\\
    \operatorname{Cov}[\upsilon b]&=\mathbb{E}[\upsilon^2 b b^{\top}]-\mathbb{E}[\upsilon b]\mathbb{E}[\upsilon b]^{\top}=\mathbb{E}[\upsilon^2 b b^{\top}]= \mathbb{E}[\langle b, x^* \rangle^2 \cdot \langle c, x^* \rangle^2 bb^{\top}]+\mathbb{E}[\zeta^2 bb^{\top}]\\
    &=\|x^*\|^2\mathbb{E}[\langle b, x^* \rangle^2 bb^{\top}]+\sigma^2/2\mathbb{I}_{n}=(\|x^*\|^4+\sigma^2/2)\mathbb{I}_{n}+2\|x^*\|^2 x^*x^{*\top}
\end{align*}
.
\end{proof}

\section{Necessary lemmas}
\label{prof:Necessary lemmas}

\begin{lemma}
\label{lemma:spectral initialization expectation}
Let's assume that $ x^* \in \mathbb{R}^n$ is the fixed vector and $a \sim \mathcal{N}(0, \mathbb{I}_n)$. Then,
$$\mathbb{E}(\langle a,x^{*}\rangle^{2}aa^T)=2x^{*}x^{*T}+||x^{*}||^{2}\mathbb{I}_{n},\mathbb{E}(\langle a,x^{*}\rangle^{4}aa^T)=||x^{*}||^{4}\left(3\mathbb{I}_{n}+\frac{12}{||x^{*}||^{2}}x^{*}x^{*T}\right).$$
\end{lemma}
\begin{proof}
   Now, let's write $a=\langle a,x^{*}\rangle \frac{x^{*}}{||x^{*}||^{2}}+v_{1}$ such that $\langle v_{1},x^{*}\rangle=0$ . Now define $\alpha_{1}:=\frac{\langle a,x^{*}\rangle}{||x^{*}||^{2}}$, so we can write $a=\alpha_{1}x^{*}+v_{1}$. Note that because of the above dissociation,  $\alpha_{1}$ and $v_{1}$ are now independent random variables. So, $\langle a,x^{*}\rangle^{2}=\alpha_{1}^{2}||x^{*}||^{4}$, and $aa^{\top}=\alpha_{1}^{2}x^{*}x^{*\top}+v_{1}v_{1}^{\top}+\alpha_{1}x^{*\top}v_{1}+\alpha_{1}v_{1}x^{*\top}.$ 
\begin{align*}
   \mathbb{I}_{n}&=\mathbb{E}[aa^{\top}]= \mathbb{E}[\alpha_{1}^{2}x^{*}x^{*\top}]+\mathbb{E}[v_{1}v_{1}^{\top}]+\mathbb{E}[\alpha_{1}x^{*\top}v_{1}]+\mathbb{E}[\alpha_{1}v_{1}x^{*\top}]\\
   &= \mathbb{E}[\alpha_{1}^{2}]x^{*}x^{*\top}+\mathbb{E}[v_{1}v_{1}^{\top}]+x^{*}\mathbb{E}[\alpha_{1}v_{1}^{\top}]+\mathbb{E}[\alpha_{1}v_{1}]x^{*\top}\overset{(a)}{=} \frac{1}{||x^{*}||^{2}}x^{*}x^{*\top}+\mathbb{E}[v_{1}v_{1}^{\top}],
\end{align*}
where $(a)$ follows from the fact that $x^{*}\mathbb{E}[\alpha_{1}v_{1}^{\top}]=0,\mathbb{E}[\alpha_{1}v_{1}]x^{*\top}=0$ because $\mathbb{E}[\alpha_{1}v_{1}^{\top}]=\mathbb{E}[\alpha_{1}]\mathbb{E}[v_{1}^{\top}]=0$ and $ \mathbb{E}[\alpha_{1}v_{1}]=\mathbb{E}[\alpha_{1}v_{1}^{\top}]^{\top}$. Now,
\begin{align*}
  \mathbb{E}[\langle a,x^{*}\rangle^{2}aa^{\top}]
  &=\mathbb{E}[\alpha_{1}^{2}||x^{*}||^{4}(\alpha_{1}^{2}x^{*}x^{*\top}+v_{1}v_{1}^{\top})] 
  = ||x^{*}||^{4}\mathbb{E}[\alpha_{1}^{4}]x^{*}x^{*\top}+||x^{*}||^{4}\mathbb{E}[\alpha_{1}^{2}v_{1}v_{1}^{\top}] \\ 
  &= ||x^{*}||^{4}\frac{3||x^{*}||^{4}}{||x^{*}||^{8}}x^{*}x^{*\top}+\frac{||x^{*}||^{4}}{||x^{*}||^{2}}\left(\mathbb{I}_{n}-\frac{1}{||x^{*}||^{2}}x^{*}x^{*\top}\right)= 2x^{*}x^{*T}+||x^{*}||^{2}\mathbb{I}_{n}
  \end{align*}
\begin{align*}
  \mathbb{E}[\langle a,x^{*}\rangle^{4}aa^{\top}]
  &=\mathbb{E}[\alpha_{1}^{4}||x^{*}||^{8}(\alpha_{1}^{2}x^{*}x^{*\top}+v_{1}v_{1}^{\top})]= ||x^{*}||^{8}\mathbb{E}[\alpha_{1}^{6}]x^{*}x^{*\top}+||x^{*}||^{8}\mathbb{E}[\alpha_{1}^{4}v_{1}v_{1}^{\top}] \\ 
  &= ||x^{*}||^{8}\frac{15||x^{*}||^{6}}{||x^{*}||^{12}}x^{*}x^{*\top}+||x^{*}||^{8}\frac{3||x^{*}||^{4}}{||x^{*}||^{8}}\left(\mathbb{I}_{n}-\frac{1}{||x^{*}||^{2}}x^{*}x^{*\top}\right)\\
  &= ||x^{*}||^{4}\left(3\mathbb{I}_{n}+\frac{12}{||x^{*}||^{2}}x^{*}x^{*T}\right).
  \end{align*}

\end{proof}

\begin{lemma}
\label{lemma: calcualtions needed for spectral initialization with covest estimator}
Let's assume that $v, x^* \in \mathbb{R}^n$ are the fixed vectors and $a \sim \mathcal{N}(0, \mathbb{I}_n)$. If $i$ is even, then 
$$\mathbb{E}\left[\left(a^{\top} x^*\right)^i\left(v^{\top} a\right)^4\right]\leq (i+3)!! ||x^{*}||^{i-4}(v^{\top}x^{*})^4+||x^{*}||^{i}(3(i-1)!!+6(i+1)!!) ||v||^4.$$
If $i$ is odd, then 
$$\mathbb{E}\left[\left(a^{\top} x^*\right)^i\left(v^{\top} a\right)^4\right]=0.$$
\end{lemma}
\begin{proof}
In the notations of the proof of \Cref{lemma:spectral initialization expectation}, if $i$ is even, then 
\begin{align*}
 \mathbb{E}\left[\left(a^{\top} x^*\right)^i\left(v^{\top} a\right)^4\right]&=||x^{*}||^{2i}\mathbb{E}\left[\alpha_{1}^i\left(v^{\top} a\right)^4\right]=||x^{*}||^{2i}\mathbb{E}\left[\alpha_{1}^i\left(v^{\top}\alpha_{1}x^{*}+v^{\top}v_{1} \right)^4\right] \\
&=||x^{*}||^{2i}\mathbb{E}[\alpha_{1}^i(\alpha_{1}^4(v^{\top}x^{*})^4+(v^{\top}v_{1} )^4+6\alpha_{1}^2(v^{\top}v_{1} )^2(v^{\top}x^{*})^2+4\alpha_{1}^3(v^{\top}x^{*})^3 (v^{\top}v_{1} ) \\ 
&+4(v^{\top}v_{1} )^3\alpha_{1}(v^{\top}x^{*}))]=||x^{*}||^{2i}\mathbb{E}[\alpha_{1}^{(i+4)}(v^{\top}x^{*})^4+\alpha_{1}^{(i)}(v^{\top}v_{1} )^4+6\alpha_{1}^{i+2}(v^{\top}v_{1} )^2(v^{\top}x^{*})^2\\
&+4\alpha_{1}^{i+3}(v^{\top}x^{*})^3 (v^{\top}v_{1} )+4(v^{\top}v_{1} )^3\alpha_{1}^{i+1}(v^{\top}x^{*}))]\\
&= ||x^{*}||^{2i}\mathbb{E}[\alpha_{1}^{(i+4)}(v^{\top}x^{*})^4+\alpha_{1}^{(i)}(v^{\top}v_{1} )^4+6\alpha_{1}^{i+2}(v^{\top}v_{1} )^2(v^{\top}x^{*})^2]\\
&=||x^{*}||^{2i}\left(\mathbb{E}[\alpha_{1}^{(i+4)}](v^{\top}x^{*})^4+\mathbb{E}[\alpha_{1}^{(i)}]\mathbb{E}[(v^{\top}v_{1} )^4]+6\mathbb{E}[\alpha_{1}^{(i+2)}]\mathbb{E}[(v^{\top}v_{1} )^2](v^{\top}x^{*})^2\right)\\
&= ||x^{*}||^{i-4}(i+3)!!(v^{\top}x^{*})^4+||x^{*}||^{i}(i-1)!!\mathbb{E}[(v^{\top}v_{1} )^4]\\&+6||x^{*}||^{i-2}(i+1)!!\mathbb{E}[(v^{\top}v_{1} )^2](v^{\top}x^{*})^2.
\end{align*}
Since $\Sigma_{v_{1}}=\mathbb{I}_{n}-\frac{1}{||x^{*}||^{2}}x^{*}x^{*\top}$. So, $\mathbb{E}[(v^{\top}v_{1} )^2]=||v||^2-\frac{(v^{\top}x^{*})^2}{||x^*||^{2}}\leq ||v||^2.$
\begin{align*}
   3||v||^4 &=\mathbb{E}[(v^{\top}a)^4]= \mathbb{E}[\alpha_{1}^4(v^{\top}x^{*})^4+(v^{\top}v_{1} )^4+6\alpha_{1}^2(v^{\top}v_{1} )^2(v^{\top}x^{*})^2
+4\alpha_{1}^3(v^{\top}x^{*})^3 (v^{\top}v_{1} )+4(v^{\top}v_{1} )^3\alpha_{1}(v^{\top}x^{*})]\\
&= \mathbb{E}[\alpha_{1}^4(v^{\top}x^{*})^4+(v^{\top}v_{1} )^4+6\alpha_{1}^2(v^{\top}v_{1} )^2(v^{\top}x^{*})^2
]\\&= \frac{3(v^{\top}x^{*})^4}{\|x^{*}\|^4}+\mathbb{E}[(v^{\top}v_{1} )^4]+6(v^{\top}x^{*})^2\frac{1}{\|x^{*}\|^2}\left(||v||^2-\frac{v^{\top}x^{*})^2}{||x^*||^{2}}\right)\overset{(a)}{\geq} \mathbb{E}[(v^{\top}v_{1} )^4],
\end{align*}
where $(a)$ follows from the fact that $\frac{3(v^{\top}x^{*})^4}{\|x^{*}\|^4}$ and $6(v^{\top}x^{*})^2\frac{1}{\|x^{*}\|^2}\left(||v||^2-\frac{(v^{\top}x^{*})^2}{||x^*||^{2}}\right)$ are positive constants. 
So,
\begin{align*}
\mathbb{E}\left[\left(a^{\top} x^*\right)^i\left(v^{\top} a\right)^4\right]&= |x^{*}||^{i-4}(i+3)!!(v^{\top}x^{*})^4+||x^{*}||^{i}(i-1)!!\mathbb{E}[(v^{\top}v_{1} )^4]+6||x^{*}||^{i-2}(i+1)!!\mathbb{E}[(v^{\top}v_{1} )^2](v^{\top}x^{*})^2\\
&\leq |x^{*}||^{i-4}(i+3)!!(v^{\top}x^{*})^4+3||x^{*}||^{i}(i-1)!! ||v||^4+6||x^{*}||^{i-2}(i+1)!!||v||^2(v^{\top}x^{*})^2\\
&\leq (i+3)!! |x^{*}||^{i-4}(v^{\top}x^{*})^4+||x^{*}||^{i}(3(i-1)!!+6(i+1)!!) ||v||^4.
\end{align*}

In the notation of the proof of \Cref{lemma:spectral initialization expectation}, if $i$ is odd, then
\begin{align*}
 \mathbb{E}\left[\left(a^{\top} x^*\right)^i\left(v^{\top} a\right)^4\right]&=||x^{*}||^{2i}\mathbb{E}\left[\alpha_{1}^i\left(v^{\top} a\right)^4\right] =||x^{*}||^{2i}\mathbb{E}\left[\alpha_{1}^i\left(v^{\top}\alpha_{1}x^{*}+v^{\top}v_{1} \right)^4\right] \\ 
&=||x^{*}||^{2i}\mathbb{E}[\alpha_{1}^{(i+4)}(v^{\top}x^{*})^4+\alpha_{1}^{(i)}(v^{\top}v_{1} )^4+6\alpha_{1}^{i+2}(v^{\top}v_{1} )^2(v^{\top}x^{*})^2+4\alpha_{1}^{i+3}(v^{\top}x^{*})^3 (v^{\top}v_{1} )\\
&+4(v^{\top}v_{1} )^3\alpha_{1}^{i+1}(v^{\top}x^{*}))]= ||x^{*}||^{2i}\mathbb{E}[4\alpha_{1}^{i+3}(v^{\top}x^{*})^3 (v^{\top}v_{1} )+4(v^{\top}v_{1} )^3\alpha_{1}^{i+1}(v^{\top}x^{*}))]\\
&=4||x^{*}||^{2i}\left(\mathbb{E}[\alpha_{1}^{(i+3)}](v^{\top}x^{*})^3\mathbb{E}[(v^{\top}v_{1} )]+\mathbb{E}[\alpha_{1}^{(i+1)}]\mathbb{E}[(v^{\top}v_{1} )^3]\right)\\&= 4||x^{*}||^{2i}\left(\mathbb{E}[\alpha_{1}^{(i+1)}]\mathbb{E}[(v^{\top}v_{1} )^3]\right).
\end{align*}
Now,
\begin{align*}
 0 &=\mathbb{E}[(v^{\top}a)^3]= \mathbb{E}[\alpha_{1}^3(v^{\top}x^{*})^3+(v^{\top}v_{1} )^3
+3\alpha_{1}^2(v^{\top}x^{*})^2 (v^{\top}v_{1} )+3(v^{\top}v_{1} )^2\alpha_{1}(v^{\top}x^{*})]\\
&= \mathbb{E}[\alpha_{1}^3(v^{\top}x^{*})^3+(v^{\top}v_{1} )^3
]= \mathbb{E}[(v^{\top}v_{1} )^3]   
\end{align*}
So, this implies $\mathbb{E}\left[\left(a^{\top} x^*\right)^i\left(v^{\top} a\right)^4\right]=4||x^{*}||^{2i}\left(\mathbb{E}[\alpha_{1}^{(i+1)}]\mathbb{E}[(v^{\top}v_{1} )^3]\right)=0$
\end{proof}
%Using the fact that if $A\preceq B$ then $\operatorname{tr}(A)\geq \operatorname{tr}(B)$, we can conclude 
\begin{theorem}
\label{throrem:conclusion of davis-kahan}
 (See Theorem 4.5.5 and the following conclusion from \cite{vershynin2018high}, Edition 1.)  Let $S$ and $T$ be symmetric matrices with the same dimensions. Fix $i$ and assume that the $i$-th largest eigenvalue of $S$ is well separated from the rest of the spectrum:

$$
\min _{j: j \neq i}\left|\lambda_i(S)-\lambda_j(S)\right|=\delta>0
$$

Let's assume that $v_i(S)$ and $v_i(T)$ are the unit eigenvectors corresponding to the $i$-th largest eigenvalues of the matrices $S$ and $T$, respectively. Then the unit eigenvectors $v_i(S)$ and $v_i(T)$ are close to each other up to a sign, namely

$$
\exists \theta \in\{-1,1\}: \quad\left\|v_i(S)-\theta v_i(T)\right\|_2 \leq \frac{2^{3 / 2}\|S-T\|}{\delta}
$$
\end{theorem}
\begin{proof}
Let $u = v_i(S)$ and $v = v_i(T)$ be unit eigenvectors corresponding to the $i$-th largest eigenvalues of symmetric matrices $S$ and $T$, respectively. From the Davis-Kahan theorem (Theorem 4.5.5 from \cite{vershynin2018high}, Edition 1), we have
\[
\sin \angle(u, v) \leq \frac{2\|S - T\|}{\delta}.
\]
 Let $\theta := \operatorname{sign}(\langle u, v \rangle) \in \{-1, 1\}$ be chosen so that the angle between $u$ and $\theta v$ lies in $[0, \pi/2]$ because $\cos \angle (u,\theta v)=\theta (\langle u, v \rangle)=|(\langle u, v \rangle)|$. Then,
\[
\|u - \theta v\|_2^2 =\langle u-v, u-v \rangle=\|u\|^2+\|\theta v\|^2-2\langle u, \theta v \rangle=2-2\theta \langle u, v \rangle=  2(1 - |\langle u, v \rangle|).
\]

From the sine-angle bound,
\[
\sin \angle(u, v) = \sqrt{1 - \langle u, v \rangle^2} \leq \frac{2\|S - T\|}{\delta},
\]
which implies
\[
|\langle u, v \rangle| \geq \sqrt{1 - \left( \frac{2\|S - T\|}{\delta} \right)^2}.
\]

Therefore,
\[
\|u - \theta v\|_2^2 \leq 2\left(1 - \sqrt{1 - \left( \frac{2\|S - T\|}{\delta} \right)^2} \right).
\]

Using the inequality $1 - \sqrt{1 - x} \leq x$ for $x \in [0,1]$, we get
\[
\|u - \theta v\|_2^2 \leq 2  \left( \frac{2\|S - T\|}{\delta} \right)^2 = \left( \frac{2^{3/2} \|S - T\|}{\delta} \right)^2,
\]
and hence,
\[
\|u - \theta v\|_2 \leq \frac{2^{3/2} \|S - T\|}{\delta}.
\]
\end{proof}

\end{document}